\documentclass[twoside]{article}

\usepackage[accepted]{aistats2022}
\usepackage[round]{natbib}

\usepackage{setting}

\begin{document}

%

%

\twocolumn[

\aistatstitle{Metalearning Linear Bandits by Prior Update}

\aistatsauthor{ Amit Peleg \And Naama Pearl \And  Ron Meir }

\aistatsaddress{Viterbi Faculty of ECE\\ 
Technion, Israel\\ samit22@campus.technion.ac.il \And  University of Haifa, Israel \\ npearl@campus.haifa.ac.il \And Viterbi Faculty of ECE \\ Technion, Israel\\ rmeir@ee.technion.ac.il } ]

\begin{abstract}
Fully Bayesian approaches to sequential decision-making assume that problem parameters are generated from a \emph{known} prior.
In practice, such information is often lacking. 
This problem is exacerbated in setups with partial information, where a misspecified prior may lead to poor exploration and performance. 
In this work we prove, in the context of stochastic linear bandits and Gaussian priors, that as long as the prior is sufficiently close to the true prior, the performance of the applied algorithm is close to that of the algorithm that uses the true prior. 
Furthermore, we address the task of learning the prior through metalearning, where a learner updates her estimate of the prior across multiple task instances in order to improve performance on future tasks. 
We provide an algorithm and regret bounds, demonstrate its effectiveness in comparison to an algorithm that knows the correct prior, and support our theoretical results empirically. 
Our theoretical results hold for a broad class of algorithms, including Thompson Sampling and Information Directed Sampling. 
\end{abstract}

\section{INTRODUCTION}\label{sec:Intro}
Stochastic bandit problems involve sequential decision-making in the face of partial feedback, aiming to maximize cumulative reward or minimize regret gained over a series of interactions with the environment (for a comprehensive overview see \cite{lattimore2020bandit}). 
Bandit algorithms often differ in their prior knowledge about the nature of the rewards.
In a frequentist setting, one assumes a reward distribution with fixed, but unknown parameters, while in a Bayesian setting, these parameters are generated from a known prior.
While much effort has been devoted to devising effective algorithms with provably low regret in both settings, the situation is far less clear in a mixed setup,
where the reward parameters are drawn from some \emph{unknown} or partially known prior distribution.
A particular challenge in this case is that exploration based on an incorrect prior assumption may lead an algorithm to waste resources by exploring irrelevant actions or, on the other hand, to disregard good ones (for earlier discussions of the influence of prior choice see  \cite{chapelle2011empirical, bubeck2013prior,honda2014optimality,liu2016prior}).

In the frequentist settings, for algorithmic reasons, some algorithms treat the parameters as if they arise from a prior distribution even though it does not reflect nature \citep{agrawal2012thompson, abeille2017linear}.
Although these algorithms can be applied in the mixed setup, it is natural to expect improved performance when an adequate estimate of the prior exists.
In this work we demonstrate in Theorem \ref{Theorem: single instance regret}, for Gaussian prior distributions, that as long as the prior estimate is sufficiently accurate, the performance of an algorithm that uses the approximate prior is close to that of the same algorithm that uses the true prior. 
This analysis is challenging, since it compares two learning algorithms, both evolving throughout their interaction with the environment.
    
One natural approach to acquire a good prior estimation is based on metalearning.
We study $d$-dimensional linear bandits in a metalearning setup where, at the beginning of each one of the $N$ instances, each of duration $T$, a linear bandit task is sampled from an unknown prior distribution.
The meta-learner maintains a continually updated meta-prior estimator across instances, and uses it as a prior for each instance. Then, within an instance, she selects actions in pursuance of maximizing accumulated rewards, based on an updated within-instance posterior. We provide an explicit algorithm and establish regret bounds with respect to (WRT) the algorithm that knows the prior.

The main contributions of this work are the following: 
\begin{itemize}
    \item In the single instance setting, we prove that when the prior deviation is small, an algorithm's regret is close, up to a multiplicative constant, to the regret of the same algorithm that uses the correct prior.
    This result holds even when the prior deviation is not restricted by a function of the instance duration as implied in previous works, e.g., \citep{bastani2019meta}.
    \item We present a class of algorithms that can use \emph{any} single-instance prior-based approach in a metalearning setup to derive regret bounds with $\tO(\sqrt{NT})$ regret, as opposed to previous results with $\tO(\sqrt{N}T^\alpha)$ regret, $\alpha \geq 3/2$. As far as we are aware, our results provide the first regret bounds of order $\tO(\sqrt{NT})$ when \emph{both} the prior mean and covariance are unknown.
    \item Technically, we develop a two-stage approach to compare algorithms using different priors, and hence \emph{different} actions along the run. This significantly reduces the time-dependence of the regret bounds, and allows us to deal with the uncertainty in both the mean and covariance of the prior (See Table~\ref{table:compare}).
    \item We demonstrate empirically the importance of meta-prior learning in general, and covariance estimation in particular.
\end{itemize}
\section{PRELIMINARIES AND SETTING}
We use the following convention:
variables appear with small letters $x$, vectors with capital letters $X$ and matrices with bold capital letters $\B{X}$.
For $X \in \R^d, \B{A}, \B{B} \in \R^{d \times d}$, $\norm{X}_{p}$ is the $l_{p}$ norm, $\norm{X}$ is the $l_{2}$ norm and $\normop{\B{A}}$ is the $l_2$ operator norm.
The smallest and largest eigenvalues of a matrix $\B{A}$ are $\lmin{\B{A}}, \lmax{\B{A}}$ and $\B{A} \succeq \B{B}$ represents that $\B{A}-\B{B}$ is PSD.
The unique square root of a PSD matrix $\B{A}$ is $\B{A}^{1/2}$.
We introduce the notation $\B{A}[\B{B}]$, when we wish to emphasize that $\B{A}$ is a function of a matrix $\B{B}$. This notation is used for vectors and variables as well.
The set $\{1,\ldots,n\}$ is denoted by $[n]$ for $n \in \mathbb{N}$ and the indicator function is denoted by $\I{\cdot}$.
Finally $\tO$ represents the $\mathcal{O}$ notation up to polylogarithmic factors and so does $\tilde \Omega$ and $\Omega$.
\subsection{Setting and Assumptions}
\label{Sec: Setting and assumptions}
We consider a metalearning problem where a learner interacts with $N$ instances sequentially.
At the start of each instance $n\in[N]$, a random vector $\theta_{n}\in \R^d$ is sampled from a multivariate Gaussian distribution $\N \left(\mu_{*}, \ssigma \right)$ with unknown parameters. 
At each time $t \in [T]$, the learner chooses an action $A_{n,t} \in \R^d$ from a subset of available actions $\A_{n,t}$ presented to her and receives a reward $x_{n,t} \left[A_{n,t}\right]=A_{n,t}^{\top} \theta_{n} + \xi_{n,t}$, where $\xi_{n,t}$ is a noise term sampled independently from a known distribution $\N(0, \sigma^2)$.
We also define the vector $X_{n,t} = [x_{n,1},\ldots,x_{n,t}]$
and the matrix $\B{A}_{n,t}$, which is formed by concatenating the vectors $\left\{A_{n,s}^{\top} \right\}_{s=1}^{t}$ in its rows.

The following technical assumptions are required for the proofs.
We first define $B_a(0)$ as a $d$-dimensional ball of radius $a$ centered at $0$ and the density function $f_A\left(A\right) \triangleq \tilde{f}_A\left(A\right)\I{\norm{A} \leq a}/Z_a$ for some function $\tilde{f}_A$ and an appropriate normalization constant $Z_a$.
\begin{assumption}
\label{Assumption: eigenvalues action Covariance matrix} 
The set of actions can be either deterministic, $\A_{n,t}=B_a(0)$, or a set of actions of any size, $\A_{n,t} \subset B_a(0)$, each of which is sampled i.i.d.~from a distribution $f_A$ with a covariance matrix whose minimal eigenvalue is lower bounded by a known constant $\lmin{\Msigma{\A}}  \geq \lbar{\lambda}_{\Msigma{\A}}>0$.
The function $\tilde{f}_A(A)$, can be either a zero mean Gaussian distribution or one which satisfies monotonicity, i.e., for every $\norm{A_1} \leq \norm{A_2}$ in the support of $f_A$, $\tilde{f}_A(A_1) \geq \tilde{f}_A(A_2)$. 
\end{assumption}
\begin{assumption}
\label{Assumption: eigenvalues prior Covariance matrix}
The minimal and maximal eigenvalues of the prior covariance matrix are lower and upper bounded by known constants, $\lmin{\Msigma{*}} 
\geq \lbar{\lambda}_{\Msigma{*}} > 0, \; \lmax{\Msigma{*}}
\leq \bar{\lambda}_{\Msigma{*}}$.
\end{assumption}
\begin{assumption}
\label{Assumption: prior mean is bounded}
The norm of the prior mean is upper bounded by a known constant, $\norm{\mu_*} \leq m$.
\end{assumption}
Regarding Assumption~\ref{Assumption: eigenvalues action Covariance matrix}, only the boundedness of the actions is necessary during all time-steps, while the monotonicity and the eigenvalues bound are used just during the exploration steps of the algorithm.
\subsection{\texorpdfstring{$\mathrm{\textbf{QB}}_{\B{\tau}}$}{QB} Algorithms and Regret Definition}
While optimal Bayesian approaches operate by an exact computation of predictive distributions, we consider algorithms that work with posterior estimates, and which are not committed to Bayesian optimality. 
We refer to such algorithms as \emph{Quasi-Bayesian} (QB), including, for example, Thompson Sampling (TS) \citep{Thomps33, russo2014learning}, and Information Directed Sampling (IDS) \citep{russo2014ids}.
The Bayesian regret of a QB algorithm that uses a prior $\N \left(\mu_{n}, \B{\Sigma}_{n}\right)$ in the $n_{th}$ instance is defined WRT an oracle that chooses at each step the action that yields the highest expected reward, i.e., 
$A^{*}_{n,t}=\argmax_{{A_{n,t} \in \A_{n,t}}} A_{n,t}^{\top} \theta_{n}$,
\begin{equation}
\begin{aligned}
\label{Eq: Bayesian regret}
    &\E{\RQBn{\mu_{n}}{\B{\Sigma}_{n}}{T}}\\
    &\quad \quad  \quad \quad 
    \triangleq \sum_{t=1}^{T} \E{x_{n,t}\left[A^{*}_{n,t}\right]-x_{n,t} \left[A_{n,t}\right]}.
\end{aligned}
\end{equation}
The expectation is taken over the prior used by the learner, which may be random, due to previous observations, the realization of $\theta_n$, the actions that were presented during the instance, the randomness of the algorithm and the received noise terms.

We analyze algorithms that use the first $\tau$ steps of each instance to explore the actions uniformly at random in order to gain information.
We refer to such algorithms as $\QB$ (for example $\TS,$ $\IDS$) and select $\tau$ so as to minimally affect the regret.

For clarification, there are three degrees of knowledge in this problem setup. 
The highest one is direct knowledge of the realization of each instance $\left\{\theta_{n}\right\}_{n=1}^{N}$. 
The oracle which knows these realizations always chooses the best actions $A^{*}_{n,t}$ and does not need to learn anything in the environment.
The second level is knowledge of the prior. 
We use the term $\KQB$ for the special version of each $\QB$ algorithm that \emph{knows} the true prior and denote its actions by $A^{\mathrm{K}}_{n,t}$.
Such algorithms attempt to learn the realization of $\theta_n$ within the instance, but do not need to learn the meta environment between the instances.
Hence their regret scales linearly in the number of instances $N$.
This type of algorithm is the one usually analyzed under the Bayesian setting, e.g., \citep{russo2014learning, russo2014ids}.
The last level of knowledge includes general $\QB$ algorithms that are unaware of the prior and the realizations and may learn both within and between instances. 

The regret of a $\QB$ algorithm incurred by the incorrect prior is defined 
WRT $\KQB$ and essentially measures the cost of `not knowing' the true prior.
We refer to it as the \emph{relative regret},
\begin{align}
\label{Eq: KQB regret}
    \nonumber
    &\E{\RQBKn{\mu_{n}}{\B{\Sigma}_{n}}{T}}\\
    &\quad \quad \triangleq \sum_{t=1}^{T} \E{x_{n,t}\left[A^{\mathrm{K}}_{n,t}\right]-x_{n,t} \left[A_{n,t}\right]}\\
    \nonumber &\quad \quad =\E{\RQBtaun{\mu_{n}}{\B{\Sigma}_{n}}{T}
    -\RQBtaun{\mu_{*}}{\ssigma}{T}}.
\end{align}
Note that a naive approach that uses the same initial prior, without transferring knowledge between instances, yields a relative regret linear in $N$.

By rewriting~\eqref{Eq: KQB regret}, we can view the Bayesian regret of a $\QB$ algorithm as a sum of the `cost of not knowing the realization of $\theta_n$ when the prior is known' and the `cost of not knowing the prior',
\begin{equation}
\label{Eq: different prespective}
\begin{aligned}
    &\E{\RQBtaun{\mu_{n}}{\B{\Sigma}_{n}}{T}}\\
    &\quad \quad =\E{\RQBtaun{\mu_{*}}{\ssigma}{T}
    +\RQBKn{\mu_{n}}{\B{\Sigma}_{n}}{T}}.
\end{aligned}
\end{equation}

For brevity, we now omit the index $n$ until presenting the meta setting in Section~\ref{Sec: MQB algorithm}.
Given an assumed prior $\N \left( \mu, \B{\Sigma}\right)$ at the beginning of an instance, a $\QB$ algorithm updates its posterior at time $t$ based on the actions taken and the rewards received,
\begin{equation}
\label{Eq: posterior update rule}
\begin{aligned}
    \B{\Sigma}_{t}
    &=\left(\B{\Sigma}^{-1} 
    +\frac{1}{\sigma^2} \B{V}_{t-1} \right)^{-1},\\
    \mu_{t}
    &=\B{\Sigma}_{t}
    \left(
    \B{\Sigma}^{-1} \mu
    +\frac{1}{\sigma^2}\B{A}_{t-1}^{\top} X_{t-1}
    \right),
\end{aligned}
\end{equation}
for the Gram matrix $\B{V}_t \triangleq \B{A}^{\top}_t \B{A}_t$.
We remind the reader that $\B{A}_t, X_t$ contain the actions and rewards up to time $t$ respectively.
The full derivation of the posterior calculation can be found in Appendix~\ref{Sec: Posterior calculations}. 

Algorithm~\ref{Alg: general QB} presents a general scheme of a $\QB$ algorithm.
The specific mechanism of each algorithm is reflected in Line~\ref{Alg: line:play}. 
For example, at time $t > \tau$, $\TS$ samples from the posterior $\tilde{\theta} \sim \p{\mu_{t},\B{\Sigma}_{t}}$ and then plays the best action given that sample, $\argmax_{A_t \in \B{\A}_{t}} A_t^{\top} \tilde{\theta}$.
\begin{algorithm}[htbp]
\label{Alg: general QB}
    \SetKwInOut{KwIn}{Inputs}
    \SetKwInOut{KwOut}{Outputs}

    \KwIn{$\mu$, $\B{\Sigma}$, $\tau$, $\sigma$}
    
    \KwOut{$\B{A}_{\tau}, X_{\tau}$
    \tcp*[h]{for meta estimation}}
    
    \textbf{Initialization:} empty matrix $\B{A}_{0}$ and vector $X_{0}$
	
	\For{$ t=1, \ldots, T$}{
	    \eIf( \tcp*[h]{within instance exploration}){$t \leq \tau$}{
    	    Sample $A_{t}$ uniformly from $\A_{t}$, observe a reward $x_{t}$ 
        }{
	        Play $A_{t}$ according to the specific algorithm scheme, observe a reward $x_{t}$ \label{Alg: line:play}
        }
    Concatenate the actions and rewards $\B{A}_{t} \gets \B{A}_{t-1} \con A_{t}$, $X_{t} \gets X_{t-1} \con x_{t}$

    Update the posterior $\N \left(\mu_{t+1},\B{\Sigma}_{t+1} \right)$ by \eqref{Eq: posterior update rule}
	}
\caption{$\QB\left(\mu, \B{\Sigma}, \tau , \sigma\right)$}
\end{algorithm}
\section{SINGLE INSTANCE REGRET} 
\label{Sec: Single instance regret}
Our main result, Theorem~\ref{Theorem: single instance regret}, bounds the relative regret~\eqref{Eq: KQB regret} of \emph{any} $\QB$ algorithm in a Gaussian prior setting.
In order to establish the result, we follow common practice in the bandit literature of dividing random events into the set of `good events' and their complement, e.g., \citep{lattimore2020bandit}. The former refers to situations where the various estimates are `reasonably' close to their true or expected values, and the latter is the complementary event that is shown to occur with low probability. The bulk of the proof consists of bounding the regret for the good event. In our setting, the good event $\mathcal{E}$ is defined as the intersection of four basic events for $\delta > 0$,
\begin{alignat}{2}
\label{Eq: events for main theorem}
    \nonumber &\mathcal{E}_{\theta}
    &&\triangleq \left\{
    \norm{\ssigma^{-1/2} \left(\theta-\mu_{*} \right)}_{\infty}^2
    \leq 2\ln \left(\frac{d^2T}{\delta} \right) \right \},\\
    \nonumber
    &\mathcal{E}_{v}
    &&\triangleq \left\{
    \lmin{\vmtau} \geq 
    \frac{\lbar{\lambda}_{\Msigma{\A}} d}{2} \right \},\\
    &\mathcal{E}_{m} &&\triangleq \left\{\norm{\app{\mu}-\mu_{*}} \leq \sqrt{\fmd \delta}
     \right\}, \\
    \nonumber &\mathcal{E}_{s} &&\triangleq
    \left\{\normop{\asigma-\ssigma}
    \leq
    \sqrt{\fsd \delta},
    \quad \asigma  \succeq \ssigma
    \right\},\\
    \nonumber &\mathcal{E} &&\triangleq \left\{\mathcal{E}_{\theta} \cap \mathcal{E}_{v} \cap \mathcal{E}_{m} \cap \mathcal{E}_{s}\right\}.
\end{alignat}
The event $\mathcal{E}_{\theta}$ is an instance-based event, unrelated to the performed algorithm, and represents the event that the realization of $\theta$ is not too far from its mean.
The event $\mathcal{E}_{v}$ indicates that the $\QB$ algorithm explores sufficiently in all directions during the exploration steps.
The events $\mathcal{E}_{m}, \mathcal{E}_{s}$ represent the distance between the prior of the $\QB$ algorithm $\N (\app{\mu}, \asigma)$ and the true unknown prior $\N \left(\mu_{*}, \ssigma\right)$. 
The arguments $\fmd$ and $\fsd$, introduced in \eqref{Eq: events for main theorem}, quantify these distances.
Moreover, the event $\mathcal{E}_{s}$ specifies that the estimated covariance is wider than the true covariance, reflecting the learner's lower level of certainty compared to an oracle that \textit{knows} the true prior and thus prevents under-exploration. 
This issue can be also realized from a Bayesian point of view, where in the case that both the mean and covariance are unknown, the posterior mean distribution is broader compared to the case that only the mean is unknown (see section 4.6 in \cite{murphy2012machine}).
The arguments $\fmd, \fsd$, as well as $\tau$, may depend on the dimension and the horizon, and may also depend logarithmically on $\nicefrac{1}{\delta}$.

\begin{restatable}{theorem}
{singleInstanceRegret}
\label{Theorem: single instance regret}
Let $\theta \sim \N (\mu_{*},\ssigma )$ and let $\N (\app{\mu},\asigma)$ be the prior of a $\QB$ algorithm.
For $\tau < T$, if for some $0<\delta \leq \nicefrac{1}{M}$ the event $\mathcal{E}$ holds with probability larger than $1-\frac{9\delta}{dT}$, then the relative regret is bounded by,
\begin{align*}
    &\underset{\substack{\textrm{\rm cost of not knowing} \\ \textrm{\rm the prior}}}
    {\underbrace{\E{\RQBK{\app{\mu}}{\asigma}{T}}}}\\
    &\quad \quad
    \leq
    k_1 \cdot
    \underset{\substack{\textrm{\rm cost of not knowing the realization of $\theta$} \\ \textrm{\rm when the prior is known}}}{\underbrace{
    \vphantom{\frac{c_{\text{\rm bad}}\delta}{\sqrt{d}}}
    \E{\RQB
    {\mu_{*,\tau+1}}
    {\B{\Sigma}_{*,\tau+1}}
    {T-\tau}}}}
    +\underset{\substack{\textrm{\rm bad} \\ \textrm{\rm event}}}{\underbrace{\frac{c_{\text{\rm bad}}\delta}{\sqrt{d}}}},
\end{align*}
\end{restatable}
\vspace*{-3mm}
where $M \in \tO\left(\fmd + \tau^2 \fsd\right), 
    \;
    k_1 \in \tO\left(\sqrt{\fmd \delta} +
    \tau \sqrt{\fsd \delta}\right).$
    
The definitions of $M, k_1$ and $c_{\text{\rm bad}}$ are in \eqref{Eq: constants for Theorem: single instance regret} in Appendix~\ref{Appendix: Single instance regret proof} as well as further details.
The relationship between the performance of the algorithm and the initial prior deviation in the events $\mathcal{E}_m$ and $\mathcal{E}_s$ is represented by $k_1$.
The term $M$ ties $\delta$ to the arguments $\fmd$ and $\fsd$, thus forcing the prior deviations to be small,
and $c_{\text{\rm bad}} \in \tO(1)$ stems from the bad event.

An immediate consequence of Theorem~\ref{Theorem: single instance regret} and \eqref{Eq: different prespective} is a bound on the Bayesian regret \eqref{Eq: Bayesian regret} of any $\QB$ algorithm,
\begin{align}
\label{Eq: consequence thrm 1}
    &\underset{\substack{\text{cost of not knowing both the} \\ 
    \text{realization of $\theta$ and the prior}}}
    {\underbrace{\E{\RQB{\app{\mu}}{\asigma}{T}}}}\\
    \nonumber
    &\enspace \; \leq
    \left(1 + k_1\right)
    \E{\RQB
    {\mu_{*,\tau+1}}
    {\B{\Sigma}_{*,\tau+1}}
    {T - \tau}}
    +\tO(\tau).
\end{align}
Note that for $\tau \in \tO\left(d\right)$, $\norm{\app{\mu}-\mu_{*}} \in \tO\left(1\right)$ and ${\big\lVert\asigma -\ssigma\big\rVert}_{\mathrm{op}} \in \tO\left(\nicefrac{1}{d}\right)$, $\QB$ is a $(1$+$\alpha)$-\emph{approximation} of $\KQB$ for some constant $\alpha >0$ which is determined by the constants in $\tau$, $\norm{\app{\mu}-\mu_{*}}$ and ${\big\lVert\asigma -\ssigma\big\rVert}_{\mathrm{op}}$. See Appendix~\ref{Appendix: Theorem 1 implication} for a concrete example.

Having bounded the regret of a $\QB$ algorithm by the standard Bayesian regret of $\KQB$, we can leverage previous results for Bayesian algorithms. 
For example, proposition 6 and Lemma 7 in \cite{lu2019information}, adjusted to the Gaussian prior in \cite{basu2021no}, bound the prior-dependent Bayesian regret for TS and a Bayesian version of UCB (Upper Confidence Bound) in the case of finite action spaces.
Plugging this bound with $\delta=\nicefrac{1}{T^2}$ into \eqref{Eq: consequence thrm 1} we get,
\begin{align}
\label{Eq: prior dependent Theorem 1}
    \nonumber
    &\E{\RQB{\app{\mu}}{\asigma}{T}}\\
    &\quad \leq
    \left(1 + k_1\right)
    \Bigg[4 \sqrt{\frac{ \bar{\lambda}_{\Msigma{*}}a^2}{\ln \left(1 + \frac{\bar{\lambda}_{\Msigma{*}}a^2}{\sigma^2} \right)}\ln \left(4\abs{ \A} T^2 \right)}\\
    \nonumber
    &\quad \quad \times \sqrt{\frac{1}{2}dT \ln \left(1 + \frac{\bar{\lambda}_{\Msigma{*}} T}{\sigma^2} \right)}
    +\sqrt{2 \bar{\lambda}_{\Msigma{*}}a^2} \Bigg]
    + \tO(\tau).
\end{align}
An interesting implication of Theorem~\ref{Theorem: single instance regret} is for the offline learning setup. 
With the increasing amount of data available, the opportunity arises to form more informative priors, which are guaranteed by the theorem to have the same regret (up to constants) as any $\KQB$ algorithm in a single instance. 
Another implication is for sequential settings, where $N$ instances are sampled from the same distribution one by one. We elaborate on the latter in Section~\ref{Sec: MQB algorithm} and show that the suggested meta-algorithm produces the conditions for the good event to hold with high probability.
\paragraph{Proof sketch} The difficulty in bounding the regret based on the comparison 
between an algorithm that knows the prior ($\KQB$) and another that estimates it $(\QB)$, is twofold. 
First, since the posteriors of both algorithms depend on the actions and the rewards throughout the instance, it is hard to track the distance between the posteriors as the instance progresses.
Second, although regret bounds on TS with a known prior are proved to be tighter as the prior is more informative \citep{russo2016information, dong2018information}, an improved bound does not ensure an actual improvement in the regret of the algorithm.
Therefore, establishing low estimation error at the start or during the instance, does not suffice.

To establish a within-instance regret bound between $\QB$ and $\KQB$, we adapt the idea of mean alignment from \cite{bastani2019meta} and adjust it to cover covariance alignment as well.
This analytic tool is used to cause the two algorithms to mathematically posses an identical posterior at a specific time and thus to behave identical (on average) until the end of the instance.
Specifically, with a two stage technique, we use the randomness of the first $\tau$ exploration steps to align both the means and the covariance matrices at time $\tau+1$.   
Since the two compared algorithms start with different covariance matrices, they can only align if the learners would take \emph{different} actions (see \eqref{Eq: posterior update rule}).
Practically, for every set of actions chosen by $\QB$ with a certain probability, there is a nonzero probability for $\KQB$ to choose the set of actions that would give rise to covariance alignment.
This occurs due to the randomness in the actions selection and due to the action space properties in Assumption~\ref{Assumption: eigenvalues action Covariance matrix}.
The cost of analytically switching between these two probabilities in order to align the covariance matrices is termed \emph{covariance alignment cost} and it is reduced as a function of the distance between the covariance matrices at the beginning of the instance.
The covariance alignment is analyzed in Appendix~\ref{Sec: covariance alignment}.

Yet, aligning the covariance matrices does not imply mean alignment as well, 
but rather leads to an alignment with a surrogate algorithm, that has started the interaction with the true covariance but with an incorrect mean.
Since the mean is a function of the actions taken and the reward noise received, after using the randomness in the actions to align the covariance matrices, we can use the noise terms to align the different means (Appendix~\ref{Sec: mean alignment}) and thus to align the surrogate algorithm to $\KQB$ at time $\tau+1$. See the full proof in Appendix~\ref{Appendix: Single instance regret proof}.
\section{\texorpdfstring{$\mathrm{\textbf{MQB}}_{\B{\tau}}$}{MQB} ALGORITHM}
\label{Sec: MQB algorithm}
QB algorithms are designed to minimize the regret within a single instance, by refining their estimation of $\theta$, while exploiting their knowledge, as the interaction with the instance proceeds. Using the same line of thought, $\MQB$ aims to minimize the regret along multiple instances by learning the meta-prior, while using the improved prior to reduce the per-instance regret.
Since the prior distribution is Gaussian, one may think to form MLE estimators for the mean and covariance prior to the $n_{th}$ instance.
However, it is inapplicable in the linear bandits environment, since the learner has no access to the true realizations of $\left\{\theta_j \right\}_{j=1}^{n-1}$.
A simple approach would be to utilize the inner-instance estimation of the QB algorithm.
This straightforward approach has two problems that the $\MQB$ algorithm solves, using two levels of exploration.

The first problem rises from the adaptive nature of bandit algorithms, which leads to biased instance-estimators, as discussed in \cite{shin2019sample}. 
This in turn, would lead to an inconsistent meta-estimation of the prior.
The solution is within-instance exploration.
At the beginning of each instance the learner performs $\tau$ \emph{exploration steps}, in which she chooses actions uniformly at random to ensure sufficient estimation of $\theta$ in all directions. The information from steps $\tau+1,\ldots,T$ is ignored during the meta-estimation to keep it unbiased.
The inner mechanism of the $\QB$ algorithm remains the same, i.e., all the actions taken during the instance participate in the inner estimation of $\theta$.
The number of exploration steps, defined so as to balance the regret incurred and the quality of the estimators, is set to
\begin{equation}
\label{Eq: tau}
    \tau= \max \left\{d, \frac{8a^2}{\lbar{\lambda}_{\Msigma{\A}}}\ln(d^2N^2T) \right\}.
\end{equation}
The second problem is the limited amount of knowledge on the meta-prior during the first instances.
Using the inaccurate meta-prior in these instances may result in poor performance as compared to an algorithm with frequentist guarantees which by its nature explores sufficiently in arbitrary environments.
Hence, the $\MQB$ algorithm uses a second level of exploration and during the first $N_0$ \emph{exploration instances} gathers information on the environment without exploiting it yet.
The number of exploration instances, derived from several requirements along the regret proof, is discussed in Section~\ref{Sec: Meta regret}.

The full scheme of the $\MQB$ algorithm is presented in Algorithm~\ref{Alg: MQB}.
Next, we elaborate on the main idea behind the prior estimation.
We use the Ordinary Least Squares (OLS) estimator to obtain a meta estimation for $\theta_{j}$ in every instance,
\begin{equation}
\label{Eq: app theta_n}
    \app{\theta}_{j}
    =\left(\sum_{t=1}^{\tau} A_{j,t} A_{j,t}^{\top}\right)^{-1}
    \sum_{t=1}^{\tau}  A_{j,t} x_{j,t}
    =\B{V}^{-1}_{j,\tau} \B{A}_{j,\tau}^{\top} X_{j,\tau}.
\end{equation}
The estimator for the mean before the $n_{th}$ instance, $\app{\mu}_{n}$, uses these estimations from all previous instances, 
\begin{align}
\label{Eq: app mu_n}
    \app{\mu}_{n}
    =\frac{1}{n-1}\sum_{j=1}^{n-1} \app{\theta}_{j}.
\end{align}
The bias-corrected MLE for the covariance before the beginning of the $n_{th}$ instance would be $\frac{1}{n-2} \sum_{j=1}^{n-1}
(\app{\theta}_{j}-\app{\mu}_{n})
(\app{\theta}_{j}-\app{\mu}_{n})^{\top}$.
However, as we show in Appendix~\ref{Sec: Covariance estimation error}, the estimation errors of $\left\{\theta_{j}\right\}_{j=1}^{n-1}$ cause it to converge to $\ssigma + \frac{\sigma^2}{n-1}
\sum_{j=1}^{n-1} \E{\B{V}^{-1}_{j,\tau}}$.
In order to cancel out the added variance, we add a further term,
\begin{equation}
\label{Eq: app Sigma_n}
    \app{\B{\Sigma}}_{n}
    =\frac{1}{n-2} \sum_{j=1}^{n-1}
    \left(\app{\theta}_{j} - \app{\mu}_{n}\right)
    \left(\app{\theta}_{j} - \app{\mu}_{n}\right)^{\top}
    -\B{G}_{\Sigma},
\end{equation}
where $\B{G}_{\Sigma} = \frac{\sigma^2}{n-1} \sum_{j=1}^{n-1} \B{V}^{-1}_{j,\tau}$. 
Although this estimator is unbiased, in practice it can be wider or narrower than the true covariance, and, as explained in Section~\ref{Sec: Single instance regret}, we aim for the former. Hence we use a widened version of the covariance as suggested by \cite{bastani2019meta} and proved in Lemma~\ref{Lemma: widend matrix is PSD}. 
Given the initial estimation $\asigma_n$ and a confidence level ${\big\lVert\asigma_n -\ssigma\big\rVert}_{\mathrm{op}} \leq s$, the widened version is given by $\wsigma_n \triangleq \asigma_n + s \cdot \B{I}$, which ensures that $\wsigma_n  \succeq \ssigma$ with high probability.
In Lemma~\ref{Lemma: M-QB} in Section~\ref{Sec: Meta regret} we show that the confidence level prior to the $n_{th}$ instance dictates,
\begin{align}
\label{Eq: covariance widening}
    &\wsigma_n=
    \asigma_n
    +c_{\mathrm{w}} \cdot \sqrt{\frac{5d + 2\ln\left(dnT\right)}
    {n-1}} \; \B{I},
\end{align}
where $c_{\mathrm{w}}=50 \left(\frac{2\sigma^2}{\lbar{\lambda}_{\Msigma{\A}}d}
    +\bar{\lambda}_{\Msigma{*}} \right).$
To align with the $\MQB$ scheme, we adjust the $\QB$ algorithm to output the actions taken and the rewards received during the first $\tau$ steps.
Any algorithm can be used in Line~\ref{MQB: line:uninformative}, as long as it is adapted to perform $\tau$ exploration steps and to return $\B{A}_{\tau}, X_{\tau}$.
\begin{algorithm}[htbp]
\label{Alg: MQB}
    \SetKwInput{KwIn}{Inputs}
    
    \KwIn{$N$, $T$, $a, \lbar{\lambda}_{\Msigma{\A}}$, $\lbar{\lambda}_{\Msigma{*}}, \bar{\lambda}_{\Msigma{*}}$, $m$, $\sigma$} 

    \textbf{Initialization:} set $\tau$ by \eqref{Eq: tau}
    
    \For(\tcp*[h]{meta exp.~instances}){$ n=1,\ldots,N_0$}{
        $\left(\B{A}_{n,\tau}, X_{n,\tau}\right) \gets$ Run any $\QB$ algorithm with frequentist guarantees  \label{MQB: line:uninformative}
        
        Compute $\app{\theta}_n$ by \eqref{Eq: app theta_n}
    }
    
    \For{$n=N_0+1,\ldots$, N}{
        Update $\app{\mu}_{n}$ by \eqref{Eq: app mu_n}
        
        Update $\asigma_{n}$ by \eqref{Eq: app Sigma_n} and $\wsigma_n$ by \eqref{Eq: covariance widening}
        
        $\left(\B{A}_{n,\tau}, X_{n,\tau}\right) \gets \QB\left(\app{\mu}_{n}, \wsigma_n, \tau , \sigma\right)$
        
        Compute $\app{\theta}_n$ by \eqref{Eq: app theta_n}
        }
\caption{$\MQB$}
\end{algorithm}
\section{\texorpdfstring{$\mathrm{\textbf{MQB}}_{\B{\tau}}$}{MQB} REGRET}
\label{Sec: Meta regret}
The meta algorithm consists of two key phases, as depicted in Figure~\ref{fig: meta algorithm regret}: 
\begin{enumerate}
    \item \emph{Within-instance} phase, where actions are taken based on the estimated meta-prior and on the within-instance updated posterior.
    \item \emph{Between-instance} phase, where the estimated meta-prior is updated based on information from previous instances.
\end{enumerate}
\vspace*{-5mm}
\begin{figure}[htpb]
\includegraphics[width=\linewidth]{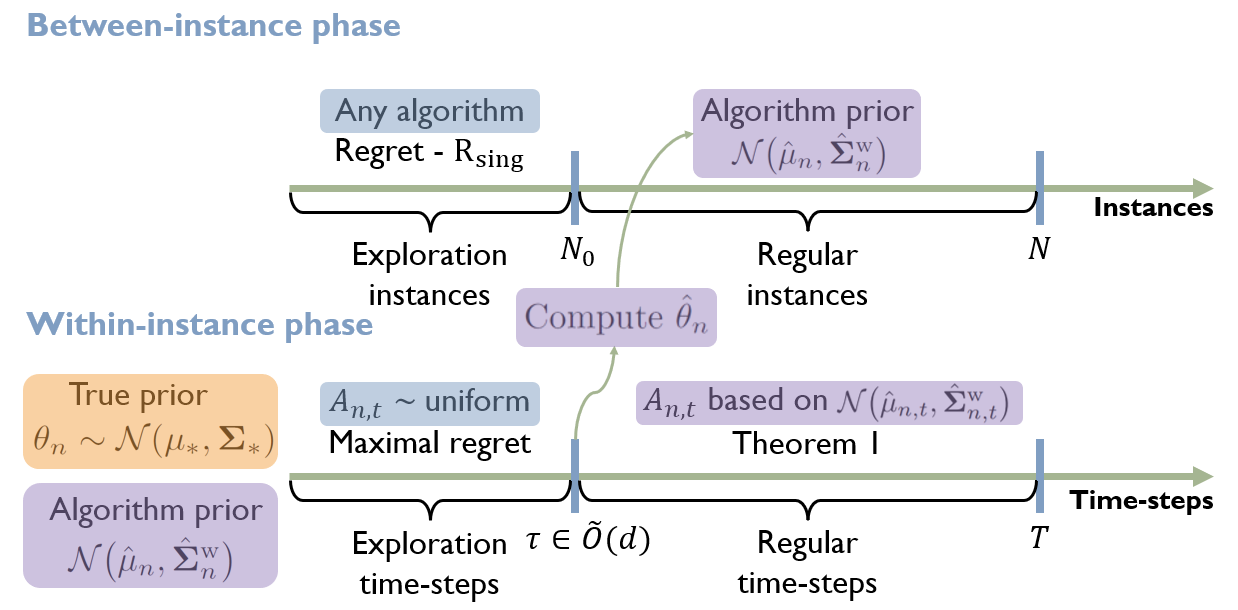}
\caption{\texorpdfstring{$\mathrm{\textbf{MQB}}_{\B{\tau}}$}{MQB} algorithm and regret scheme.
See Section~\ref{Sec: Single instance regret} for the regret analysis in a single instance, and Section~\ref{Sec: MQB algorithm} for the complete algorithm scheme.
$\mathrm{R_{sing}}$ is defined later in the section.}
\label{fig: meta algorithm regret}
\end{figure}
We address the first in~Theorem~\ref{Theorem: single instance regret}, which bounds the per-instance relative
regret given a bound on the deviations between the estimated and the true prior.
The second, addressed in Lemma~\ref{Lemma: M-QB}, explained below, demonstrates that as the number of instances increases, these prior deviations approach zero.
Finally, Theorem~\ref{Thm:FullRegret} combines these two basic components in order to establish a regret bound for $\MQB$ over $N$ instances.
In order to emphasize the instance dependence, we denote in this section several of the arguments with a subscript $n$.

Next, we present Lemma~\ref{Lemma: M-QB}, which provides bounds on the distance between the prior constructed by $\MQB$ and the true prior.
In particular, we show that before the $n_{th}$ instance, $\MQB$ meets the events $\mathcal{E}_{m}, \mathcal{E}_{s}$ defined in~\eqref{Eq: events for main theorem} with $\delta_n \triangleq 1/(n-1)$ and closed-form expressions of $f_{m,n}, f_{s,n} \in \tO(d)$ (see Appendix~\ref{Appendix: Good event definition}).
We denote this adjusted per instance \emph{good event} by $\mathcal{E}_{n\text{($\MQB$)}}$.
The adjusted formalization and the proof of Lemma~\ref{Lemma: M-QB} can be found in Appendix~\ref{Appendix: Good event definition}, based on the mean estimation error (Appendix~\ref{Sec: Mean estimation error}) and the covariance estimation error (Appendix~\ref{Sec: Covariance estimation error}).
\begin{restatable}{lemma}
{lemmaMQB}
\label{Lemma: M-QB}
{($\MQB$ conditions)}
For every instance $n > 10d + 4\ln\left(16dT\right)$, 
$\p{\mathcal{E}_{n\text{($\MQB$)}}} \geq 1-\nicefrac{8}{dnT}$.
\end{restatable}

The expressions of $f_{m,n},f_{s,n},\delta_n$ and the expression of $\tau$ in \eqref{Eq: tau} define $M$ and $k_1$ in Theorem~\ref{Theorem: single instance regret} as a function of $n$, i.e. $M_n, k_{1,n}$.
We define the number of exploration instances for $\MQB$ as $N_0 \triangleq \ceil{M_{N+1}}$, which ensures for every $n > N_0$ that $\delta_n < \nicefrac{1}{M_n}$.
Having established the two components described at the beginning of this section, we can bound the regret incurred by the incorrect prior of $\MQB$ using the following theorem.
The single instance regret of the algorithm used during the first $N_0$ instances is denoted by $\mathrm{R_{sing}}$.
\begin{restatable}{theorem}
{fullMQBregret}
\label{Thm:FullRegret}
For $N_0 \leq N$, the $\MQB$ $N$-instance relative regret is bounded by,
\begin{align*}
    &\sum_{n=1}^{N}
    \E{\RQBKn{\app{\mu}_{n}}{\wsigma_{n}}{T}}\\ 
    & \, \leq
    \underset{\substack{\mathrm{regular}\\\mathrm{instances}}}
    {\underbrace{k_2 \sqrt{N}
    \E{ \RQB
    {\mu_{*,\tau+1}}
    {\B{\Sigma}_{*,\tau+1}}
    {T-\tau}}}}
    + \underset{\substack{\mathrm{exploration}\\\mathrm{instances}}}
    {\underbrace{
    \vphantom{k_2 \sqrt{N}
    \E{ \RQB
    {\mu_{*,\tau+1}}
    {\B{\Sigma}_{*,\tau+1}}
    {T-\tau}}}
    N_0 \mathrm{R_{sing}}},}
\end{align*}
\end{restatable}
\vspace*{-2mm}
where $k_2 \in \tO\left(d^{3/2}\right), \; N_0 \in \tO\left(d^{3}\right).$
The definitions of $k_2$ and $N_0$ are in \eqref{Eq: constants for Theorem 2} and the proof is in \eqref{Eq: Theorem MQB regret - proof} in Appendix~\ref{Appendix: Meta algorithm regret}.

During the first $N_0$ instances, $\MQB$ suffers as much regret as the frequentist algorithm it chooses, for example when using the IDS algorithm of \cite{kirschner2018information}, $\mathrm{R_{sing}} \in \tO(d \sqrt{T})$.

An immediate consequence of Theorem~\ref{Thm:FullRegret} is a  bound on the Bayesian regret \eqref{Eq: Bayesian regret} of the $\MQB$ algorithm,
\begin{align}
\label{Eq: thm2 implication}
    &\sum_{n=1}^{N}
    \E{\RQBn{\app{\mu}_{n}}{\wsigma_{n}}{T}}\\
    \nonumber
    &\quad
    \leq \left(N + k_2 \sqrt{N}\right) 
    \E{ \RQB
    {\mu_{*,\tau+1}}
    {\B{\Sigma}_{*,\tau+1}}
    {T-\tau}}\\
    \nonumber
    &\quad \quad \quad \quad
    +\tO 
    \left(N_0 \mathrm{R_{sing}} + Nd\right).
\end{align}
The multiplicative factor $\left(N + k_2 \sqrt{N}\right)$ is a sum of two parts; $N$, that stems from the inherent regret of the $\KQB$ algorithm and cannot be avoided, and $k_2 \sqrt{N}$ which represents the price of the `prior alignment'.
The second term on the right-hand-side of the inequality is the `cost' of the two exploration levels. 
Both the `prior alignment' and the exploration costs become vanishingly small WRT the inherent regret as $N$ and $T$ increase, implying negligible cost for `not knowing the prior'.

Using the same prior-dependent bound as in \eqref{Eq: prior dependent Theorem 1}, now with $\delta=\nicefrac{1}{(NT)^2}$, \eqref{Eq: thm2 implication} can be further extended to a prior-dependent bound,
\begin{align}
\label{Eq: prior dependent Theorem MQB regret}
    \nonumber
    &\sum_{n=1}^{N}
    \E{\RQBn{\app{\mu}_{n}}{\wsigma_{n}}{T}}\\
    &\quad \leq 
    \left(N+k_2 \sqrt{N}\right)
    \vast[\sqrt{\frac{8 \bar{\lambda}_{\Msigma{*}}a^2
    \ln \left(4\abs{ \A} (NT)^2 \right)}
    {\ln \left(1 + \frac{\bar{\lambda}_{\Msigma{*}}a^2}{\sigma^2} \right)}}\\
    \nonumber
    &\quad \quad \times
    \sqrt{\ln \left(1 + \frac{\bar{\lambda}_{\Msigma{*}} T}{\sigma^2} \right) dT}\vast]
    +\tO 
    \left(N_0 \mathrm{R_{exp}} + Nd\right).
\end{align}
Note that a simple policy that runs all the tasks separately will incur regret of $\tilde{O}(Nd\sqrt{T})$ regardless of the ``informativeness'' of the correct prior.
On the other hand, for the $\MQB$ algorithm with TS as a sub-routine, as the prior is more informative, the regret is lower, and in the extreme case ($\lambda_{\mathrm{max}}(\mathbf{\Sigma}_*) \rightarrow 0$ and $N$ becomes large), only the inner exploration cost remains, i.e., $\tilde{O}(Nd)$.
\begin{table*}[htpb]\renewcommand{\arraystretch}{1.4}
  \caption{Comparison of worst-case relative regret bounds for state-of-the-art meta TS algorithms}
  \label{table:compare}
  {\small
  \begin{minipage}{\linewidth}
    \renewcommand\footnoterule{}
    \renewcommand*{\thempfootnote}{\roman{mpfootnote}}
  \begin{tabular}{cccccc}
    \toprule
    & \makecell{\textbf{SETTING}} & \makecell{\textbf{ACTIONS} \\ \textbf{ASSUMPTIONS}} & \makecell{\textbf{COVARIANCE}\\\textbf{ASSUMPTIONS}} & \makecell{\textbf{MEAN} \\ \textbf{ASSUMPTIONS}} &  \makecell{ \textbf{RELATIVE} \\\textbf{REGRET}} \\ \toprule
    \makecell{Kveton et al.\\ (2021)}
    & \makecell{ K-arms \\ MAB} &  \makecell{---irrelevant---} &  \makecell{Known \\ $c \cdot I$} & \makecell{Bayesian, known \\ hyper-prior}  & \makecell{$\tO(K \sqrt{N} T^2 )$} \\ \midrule
    \makecell{Bastani et al.\\ (2021)}
    & \makecell{Dynamic \\ pricing\footnote{Adapting the analysis to a linear bandits setting results in a reduction of $\tO(\sqrt{d})$ from the  regret.}} &
    \multirow{3}{*}{\makecell{known bounds\\ on actions \\and eigenvalues\\ (Assumption~\ref{Assumption: eigenvalues action Covariance matrix}\footnote{In our work there exists an extra requirement on the distribution `monotonicity'.}) }}&
    \multirow{3}{*}{\makecell{ Unknown, \\ known bounds \\ on eigenvalues \\
    (Assumption~\ref{Assumption: eigenvalues prior Covariance matrix})}} &
    \multirow{3}{*}{\makecell{Unknown, \\ known bound \\\\ (Assumption~\ref{Assumption: prior mean is bounded})}} &
    \makecell{$\tO(d^4\sqrt{N} T^{3/2})$, \\ $N_0 \in \tO (d^4T^2)$} \\ \cmidrule{1-2} \cmidrule{6-6}
    \makecell{Ours, \\ $\MQB$\footnote{With an exploration algorithm for which $\mathrm{R_{sing}}\in \tO(d\sqrt{T})$, for example IDS \citep{kirschner2018information}.}} & \makecell{Linear \\ bandits} &  &   &  & \makecell{$\tO(d^{5/2}\sqrt{N T})$, \\ $N_0 \in \tO (d^3)$} \\
    \bottomrule
  \end{tabular}
  \end{minipage}}
\end{table*}\renewcommand{\arraystretch}{1}

\section{RELATED WORK}\label{sec:RelatedWork}
While a significant amount of empirical and theoretical work has been devoted to metalearning in the domain of supervised learning (see recent review in \cite{hospedales2021meta}), including methods based on prior update \citep{pentina2014pac,amit2018meta}, there has been far less theoretical work on this topic in sequential decision-making problems (for a recent survey of algorithmic issues, see \cite{ortega2019meta}). 

We mention several works that deal with metalearning of stochastic bandits. 
\cite{cella2020meta} consider linear bandits tasks drawn from a more general prior distribution, but assume a \emph{known} variance.
They establish \emph{prior-dependent} regret bound for their proposed regularized optimism-based algorithm, similar to  \eqref{Eq: prior dependent Theorem 1}.
However, our result is a consequence of the tighter bound in Theorem~\ref{Theorem: single instance regret} that holds for every $\QB$ algorithm relatively to its best scenario when the prior is known.
Two recent papers that answer a question similar to Theorem~\ref{Theorem: single instance regret} are \cite{kveton2021meta} and \cite{bastani2019meta}, both suggest TS based meta-algorithms. 
The main difference between the approaches is the analysis technique, leading to the gap in the regret bounds, summarized in Table~\ref{table:compare}.

\cite{kveton2021meta} focus on a fully Bayesian multi-armed bandits (MAB) setting, where tasks are drawn from a Gaussian prior.
The prior is parameterized by a \emph{known} scalar covariance and an unknown mean, that is itself drawn from a \emph{known} hyper-prior.
The authors derive a regret bound which depends on $T$ as $\tO(T^2)$.
Our result preserves their linear dependence in the initial mean deviation, while keeping the same time dependence as the algorithm that knows the prior.
When using TS, this leads to a worst-case regret whose $T$ dependence is $\tO(\sqrt{T})$.
Note that in the setting of known covariance, it is possible to use our proof scheme and still achieve the same regret guarantee of $\tO(\sqrt{T})$, even if we drop the somewhat restrictive action assumption (Assumption~\ref{Assumption: eigenvalues action Covariance matrix}).
\cite{bastani2019meta} consider contextual linear bandits in a dynamic pricing setting.
Their $\tO(d^4 \sqrt{N}T^{3/2})$ regret bound is effective after $N_0 \in \tO(d^4T^2)$ instances, while we obtain $\tO(d^{5/2}\sqrt{NT})$ regret, effective after $N_0 \in \tO(d^3)$ instances in which the learner suffers regret of $\mathrm{R_{sing}}\in \tO(d\sqrt{T})$.

Finally, three very recent papers warrant mention.
\cite{basu2021no} assume a fully Bayesian framework where the covariance is \emph{known} and the mean is sampled from a \emph{known} Gaussian distribution.
These assumptions allow the authors to elegantly expand the information theory analysis previously used in the single instance setup \citep{lu2019information} to the new framework of multiple instances. 
However, relaxing the assumption of a known covariance within their Bayesian setting complicates their analysis significantly and was not pursued in their paper. 
They establish a \emph{prior-dependent} regret bound whose worst-case dependence on $T$ is $\tO(\sqrt{T})$.
\cite{simchowitz2021bayesian} bound the single instance misspecification error for a wide class of priors and settings and achieve an upper-bound of $\tO(\varepsilon T^2)$, where $\varepsilon$ is the initial total-variation prior estimation error, while our bound from Theorem~\ref{Theorem: single instance regret} is $\tO(\sqrt{T})$.
In addition, they derive a lower bound of $\tilde{\Omega}(\varepsilon T^2)$ for MAB with $T \ll |\mathcal{A}|$ 
($|\mathcal{A}|$ is the number of actions).
To the best of our knowledge, this is the only lower bound in the literature, and it is not applicable for most settings, including ours.
For multiple instances, they derive a bound only for the MAB setting.
\cite{wan2021metadata} studies a generalized version of a meta MAB environment, in which they allow the distribution to depend on task-specific features.
Their algorithm uses TS in a Bayesian hierarchical model.

We briefly highlight differences in the proof techniques. 
\cite{kveton2021meta} performs \emph{history alignment}, focusing
on the probability that the two algorithms have the same history.
The alignment process separates each time-step into two events.
\emph{(i)}~Both algorithms perform the same action and receive the same reward, hence have zero regret WRT each other. 
\emph{(ii)}~The algorithms perform different actions that violate the alignment, and therefore suffer a worst case regret of $\tO(T)$ over the rest of the instance.
Summation over the time-steps leads to regret of $\tO(T^2)$.
\cite{bastani2019meta} first performs $\tau$ exploration steps, in which the two algorithms choose the \emph{same} actions but receive different rewards due to noise, thus enabling the mean alignment.
From this point, the proof continues using tools from importance sampling \citep{precup2000eligibility}.
We believe this technique has a shortcoming.
While aligning the means in the first $\tau$ steps facilitates the analysis at time $\tau + 1$, the resulting posterior updates of the means do not render them equal in subsequent steps, even if the two algorithms choose the same actions, since the covariance matrices differ (see \eqref{Eq: posterior update rule}). 
Our work aligns both the means and covariance matrices.
This line of proof establishes at a specific time a full prior alignment at a single cost that scales with the distances between the priors,
while the two other techniques are applied separately for each step, thus their per-step cost is multiplied by the horizon.
These differences lead to a significant gap in the upper bounds.

\section{EXPERIMENTS}\label{sec:Experimtnes}
We demonstrate the effectiveness of $\MQB$ with TS as a subroutine ($\mathrm{MTS}_{\tau}$) in a synthetic environment as in \cite{kveton2021meta, simchowitz2021bayesian}, comparing it to several baselines. 
$(i)$ TS algorithm that does not know the prior and uses a zero vector as $\mu$ and a diagonal covariance matrix $\B{\Sigma}$ with $\lmax{\ssigma}$ in its diagonal (UKTS);
$(ii)$ TS algorithm that \textit{knows} the correct mean and uses the above covariance $\B{\Sigma}$ (KMTS);
$(iii)$ TS algorithm that \textit{knows} the correct prior (KTS).
None of the above perform any forced exploration.
Other metalearning algorithms in the literature, which do not assume known covariance, mostly differ from our work in their settings and analysis. Adapting the algorithms to our setting with empirical adjustments results in an algorithm similar to ours. Therefore we do not use them as baselines.

We compare three versions of the algorithm. 
The first, \mbox{Th-$\mathrm{MTS}_{\tau}$}, uses only the first $\tau$ steps in each instance to form the meta estimator as suggested by theory;
the second, All-$\mathrm{MTS}_{\tau}$, still performs the $\tau$ exploration steps which ensures an invertible Gram matrix $\B{V}_{j,\tau}$, however it uses the information gathered from all time-steps for the meta estimation;
the third, All-MTS, is similar to All-$\mathrm{MTS}_{\tau}$, but does not perform inner-instance exploration (exploration is only used towards the end of an instance in case that the Gram matrix $\B{V}_{j,t}$ remains singular).
Since in realistic environments the learner is often unaware of $T$ and $N$, $\tau$ was adjusted to be the first time in each instance in which $\lmin{\frac{1}{\sigma^2} \B{V}_{j,t}} \geq 0.03$. 
We also set $N_0$ to be $d^3$ instead of the exact definition of $N_0 \in \tO(d^3)$
and the covariance widening constant $c_\mathrm{w}$ to be $10$ for \mbox{Th-$\mathrm{MTS}_{\tau}$} and $1$ for the versions that use all samples. 

We use a linear bandits framework with $d=5$ and $N=10{\small,}000$ instances all drawn from a Gaussian distribution $\N \left(\mu_{*}, \ssigma\right)$, where $\mu_{*}=\left[ 2, 2, 2, 2, 2\right]$ and $\ssigma$ is a non-diagonal covariance matrix, with ones along the diagonal and 0.8 elsewhere.
The horizon is $T=200$ and in each time-step, $20$ actions are available to the learners, all sampled from a uniform distribution over an $a=0.25$ radius ball.
The reward observed by the learners is corrupted by a standard Gaussian noise $\N(0, 1)$.
\begin{figure}[htpb]
\includegraphics[width=\linewidth]{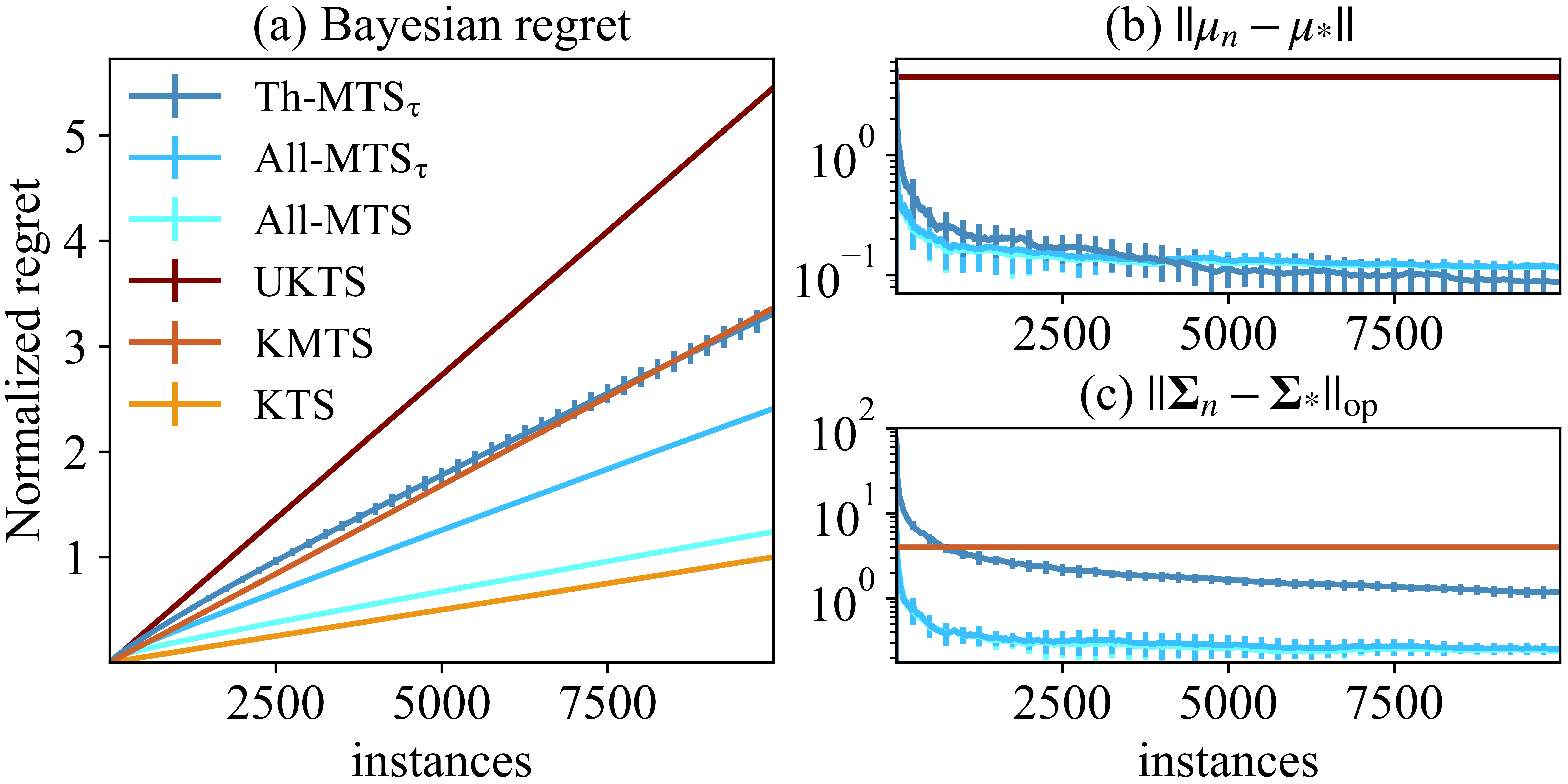}
\caption{Comparing $\MQB$ using TS as a sub-routine to several baselines, averaged over 10 runs with error bars equivalent to one std (hardly noticeable in the scale of the plots). The Bayesian regret is normalized by the KTS regret, such that its maximal regret is equal to 1. The prior convergence graphs ((b) and (c)) are in logarithmic scale.}
\label{fig: basic results}
\end{figure}

As can be seen in Figure~\ref{fig: basic results}(a), all versions of the $\mathrm{MTS}_{\tau}$ algorithm achieve better results than UKTS, indicating the importance of prior learning.
Meta algorithms that assume a known covariance \citep{kveton2021meta, basu2021no} may achieve good results WRT KTS when their assumed covariance is correct. However in realistic environments when the covariance is unknown, KMTS represents their best scenario of estimating the mean alone.
The results of the $\mathrm{MTS}_{\tau}$ versions that reach and even outperform this oracle that knows the correct mean, demonstrate the significance of covariance learning.
As can be seen in Figure~\ref{fig: basic results}(b) the mean of Th-$\mathrm{MTS}_{\tau}$ approaches the true prior mean, in contrast to the other two versions that used all the samples, and incur the known bias of adaptive algorithms \citep{shin2019sample}. 
However, due to scarcity of samples, the covariance convergence is slower (Figure~\ref{fig: basic results}(c)) and results in higher regret compared to All-$\mathrm{MTS}_{\tau}$. 
This, and the additional `cost' of exploration, as demonstrated by the gap between All-$\mathrm{MTS}_{\tau}$ and All-$\mathrm{MTS}$, suggest that empirical adjustments are needed for purely theoretically justified algorithms, perhaps using advances in bias reducing techniques, e.g., \citep{deshpande2018accurate}.

\section{CONCLUSIONS}
\label{Sec: conclusions}
We presented algorithms and expected regret bounds for stochastic linear bandits where the expected rewards originate from a vector $\theta$, sampled from a Gaussian distribution with unknown mean and covariance. 
For $\QB$ algorithms with a good estimation of the prior, we derived single instance regret bounds, which are a multiplicative constant away from the regret of the algorithm that uses the true prior. 
For TS and IDS in the metalearning setup we established a relative regret of  $\tO(d^{5/2}\sqrt{NT})$ when 
using frequentist IDS during the exploration instances, a multiplicative improvement of $\tO(d^{3/2} T)$ from previous results of $\tO(d^4\sqrt{N} T^{3/2})$. 
Two limitations of our approach are the somewhat restrictive Assumption~\ref{Assumption: eigenvalues action Covariance matrix} and the need to compare with algorithms whose first $\tau$ within-instance steps are purely exploratory. We believe that Assumption~\ref{Assumption: eigenvalues action Covariance matrix} can be made more flexible without harming performance. Removing the initial $\tau$ exploratory steps, in a theoretically justifiable way, is left as an open problem.

\subsubsection*{Acknowledgements}
We are grateful to Nadav Merlis for helpful discussions and insights. The work is partially supported by the Ollendorff Center of the Viterbi Faculty of Electrical and Computer Engineering at the Technion, and by the Skillman chair in biomedical sciences.

\newpage
\bibliographystyle{plainnat}
\bibliography{bibliograhy.bib}

\clearpage
\appendix
\thispagestyle{empty}
\onecolumn \makesupplementtitle
\section{SUMMARY OF NOTATION}
Some additional notations we use in the appendix are the trace and the determinant of a matrix $\B{A}$, $\Tr{\B{A}}$ and $\Det{\B{A}}$, respectively. 
We use $\norm{X}_{\B{A}}=\sqrt{X^{\top} \B{A} X}$ for the weighted norm matrix of a PD matrix $\B{A}$.
We define the vector $\Xi_{n,t}=[\xi_{n,1},\ldots,\xi_{n,t}]$
and often use $\int_{E}$ with an abuse of notation to indicate that we are integrating over all the terms that maintain the event $E$.
Next, we summarize the notation used along the paper.
\begin{description}[leftmargin=!,labelwidth=\widthof{\bfseries $\B{A} \succeq \B{B}$}]
    \item[$\A_{n,t}$]  The actions available to the learner at the $n_{th}$ instance at time-step t.
    \item[$a$]  Bound on the actions available to the learner.
    \item[$\Msigma{\A}$]  The action covariance matrix.
    \item[$A^*_{n,t}$]  The optimal action at the $n_{th}$ instance at time-step t.
    \item[$A^{\mathrm{K}}_{n,t}$]  The action taken by $\KQB$ at the $n_{th}$ instance at time-step t.
    \item[$A_{n,t}$]  The action taken by the $\QB$ learner at the $n_{th}$ instance at time-step t.
    \item[$\B{A}_{n,t}$]  A matrix formed by concatenating the vectors $\left\{A_{n,s}^{\top} \right\}_{s=1}^{t}$ in it's rows.
    \item[$\B{V}_{n,t}$]  The Gram matrix $\sum_{t=1}^{T} A_{n,s}A_{n,s}^{\top}
    =\B{A}_{n,t}^{\top}\B{A}_{n,t}$.
    \item[$x_{n,t}$]  The reward at the $n_{th}$ instance at time-step t.
    \item[$X_{n,t}$]  A vector containing all the rewards at the $n_{th}$ instance up to time-step t.
    \item[$\xi_{n,t}$]  The reward noise at the $n_{th}$ instance at time-step t, sampled from $\N(0, \sigma^2)$.
    \item[$\Xi_{n,t}$]  Vector containing all the reward noises at the $n_{th}$ instance up to time-step t.
    \item[$S_{n,t}$]  Summation of the action-noise terms $\sum_{t=1}^{T} A_{n,s} \xi_{n,s} = \B{A}_{n,s}^{\top} \Xi_{n,s}$.
    \item[$\tau$]  The number of exploration time-steps taken in each instance.
    \item[$\theta_n$]  The realization of $\theta$ at the $n_{th}$ instance.
    \item[$\app{\theta}_n$]  Meta approximation of $\theta_n$ using the rewards from the first $\tau$ time-steps.
    \item[$\rho_{n}$]  
    The inner $n_{th}$ instance error, $\app{\theta}_{n} - \theta_{n}=\B{V}_{n,\tau}^{-1} \sum_{s=1}^{\tau} A_{n,s} \xi_{n,s}$.
    \item[$\mu_{*}$]  The true unknown prior mean.
    \item[$m$] Bound on the prior mean, $\norm{\mu_*} \leq m$.
    \item[$\app{\mu}_n$]  The prior mean of the learner for the $n_{th}$ instance.
    \item[$\Delta_n$]  The difference between the realization of the instance and its mean. Can be viewed as it were sampled from $\N \left(0,\ssigma \right)$.
    \item[$\ssigma$]  The true unknown prior covariance.
    \item[$\asigma_{n}$]  The estimated prior covariance for the $n_{th}$ instance
    \item[$\wsigma_n$]  The widend prior covariance for the $n_{th}$ instance.
    \item[$\B{B}$]  Represents the error between the covariance matrices that should be canceled during the covariance alignment phase.
    \item[$\B{A}^{1/2}$]
    The unique square root of a PSD matrix $\B{A}$.
    \item[\texorpdfstring{$\B{A}[\B{B}]$}{}]  Represents the matrix $\B{A}$ as a function of the matrix $\B{B}$, used for vectors and scalars as well.
    \item[$\B{A} \succeq \B{B}$] 
    Represents that $\B{A}-\B{B}$ is PSD.
\end{description}
\begin{description}[leftmargin=!,labelwidth=\widthof{\bfseries $\lmin{\B{A}}, \lmax{\B{A}}$}]
    \item[$\norm{X}_{p}$]  The $l_{p}$ norm.
    \item[$\norm{X}$]  The $l_{2}$ norm.
    \item[$\norm{X}_{\B{A}}$] 
    The weighted norm matrix for a PD matrix $\B{A}$, $\sqrt{X^{\top} \B{A} X}$.
    \item[$\normop{\B{A}}$]  The $l_2$-operator norm.
    \item[$\lambda_{j}\left(\B{A}\right), \sigma_{j}\left(\B{A}\right)$]  The $j_{th}$ eigenvalue and singular value of a matrix $\B{A}$ respectively, arranged in a decreasing manner.
    \item[$\lmin{\B{A}}, \lmax{\B{A}}$]  Smallest and largest eigenvalues of a matrix $\B{A}$ respectively.
    \item[$\lbar{\lambda}_{\B{A}}, \bar{\lambda}_{\B{A}}$]  Bounds on the smallest and largest eigenvalues of a matrix $\B{A}$ respectively.
    \item[$\RQBtaun{\mu}{\B{\Sigma}}{t}$]  The regret of $\QB$ with mean $\mu$, covariance $\B{\Sigma}$ and remaining horizon $t$ WRT the optimal algorithm (an oracle that knows the true realization of $\theta_{n}$).
    \item[$\RQBKn{\mu}{\B{\Sigma}}{t}$]  The relative regret of $\QB$ with mean $\mu$, covariance $\B{\Sigma}$ and remaining horizon $t$ WRT the $\KQB$ algorithm.
\end{description}
\section{WITHIN INSTANCE POSTERIOR CALCULATIONS}
\label{Sec: Posterior calculations}
We recall a basic result from Bayesian statistics.
\begin{lemma}
\label{Lemma: Posterior calculation}
(Bayes rule for linear Gaussian systems - Theorem 4.4.1 in \cite{murphy2012machine})

Suppose we have two variables, $X$ and $Y$. Let $X \in \R^{D_x}$ be a hidden variable and $Y \in \R^{D_y}$ be a noisy observation of $X$. Let us assume we the the following prior and likelihood:
\begin{align*}
    \p{X}=\N \left(\mu_{X},\B{\Sigma_{X}}\right), \quad
    \p{Y \mid X}=\N \left(\B{A}X +B,\B{\Sigma_{Y}}\right).
\end{align*}
The posterior $\p{X \mid Y}$ is given by the following:
\begin{align*}
    &\p{X \mid Y}=\N \left(\mu_{X \mid Y},\B{\Sigma_{X \mid Y}}\right),\\
    &\B{\Sigma_{X \mid Y}^{\mathrm{-1}}}=\B{\Sigma_{X}^{\mathrm{-1}}} +\B{A}^{\top} \B{\Sigma_{Y}^{\mathrm{-1}}} \B{A},\\
    &\mu_{X \mid Y}=\B{\Sigma_{X \mid Y}} \left(\B{\Sigma_{X}^{\mathrm{-1}}}\mu_{X} + \B{A}^{\top} \B{\Sigma_{Y}^{\mathrm{-1}}} (Y-B)\right).
\end{align*}
\end{lemma}
Given a prior $\N\left(\mu_{n}, \B{\Sigma}_{n} \right)$ and using Lemma \ref{Lemma: Posterior calculation}, the prior before choosing an action at time-step $t$ is,
\begin{equation}
\label{Eq: sigma_general}
\begin{aligned}
    \B{\Sigma}_{n,t}
    &=\left(\B{\Sigma}_{n,t-1}^{-1}
    +\frac{1}{\sigma^2} A_{n,t-1}A_{n,t-1}^{\top}\right)^{-1}\\
    &=\left(\B{\Sigma}_{n}^{-1} 
    +\frac{1}{\sigma^2} \sum_{s=1}^{t-1} A_{n,s}  A_{n,s}^{\top}\right)^{-1}\\
    &=\left(\B{\Sigma}_{n}^{-1} 
    +\frac{1}{\sigma^2} \B{A}_{n,t-1}^{\top} 
    \B{A}_{n,t-1}\right)^{-1}\\
    &=\left(\B{\Sigma}_{n}^{-1} 
    +\frac{1}{\sigma^2} \B{V}_{n,t-1} \right)^{-1},
\end{aligned}
\end{equation}
\begin{equation}
\label{Eq: mu_general}
\begin{aligned}
    \mu_{n,t}
    &=\B{\Sigma}_{n,t}
    \left(
    \B{\Sigma}_{n,t-1}^{-1} \mu_{n,t-1}
    +\frac{1}{\sigma^2} A_{n,t-1} x_{n,t-1}
    \right)\\
    &=\B{\Sigma}_{n,t}
    \left(
    \B{\Sigma}_{n}^{-1} \mu_n
    +\frac{1}{\sigma^2}\B{A}_{n,t-1}^{\top} X_{n,t-1}
    \right)\\
    &=\B{\Sigma}_{n,t}
    \left(
    \B{\Sigma}_{n}^{-1} \mu_{n}
    +\frac{1}{\sigma^2} \B{V}_{n,t-1} \theta_{n}
    +\frac{1}{\sigma^2} \sum_{s=1}^{t-1} A_{n,s} \xi_{n,s}
    \right)\\
    &=\B{\Sigma}_{n,t}
    \left(
    \B{\Sigma}_{n}^{-1} \mu_{n}
    +\frac{1}{\sigma^2}\B{V}_{n,t-1} \theta_{n}
    +\frac{1}{\sigma^2}\B{A}_{n,t-1}^{\top} \Xi_{n,t-1}
    \right).
\end{aligned}
\end{equation}
Specifically, for $\MQB$, the meta-prior is $\N\left(\app{\mu}_{n}, \wsigma_{n} \right)$, hence the inner-instance posterior is,
\begin{equation}
\label{Eq: sigma_M}
    \wsigma_{n,t}
    =\left(\left(\wsigma_{n}\right)^{-1} 
    +\frac{1}{\sigma^2} \B{V}_{n,t-1} \right)^{-1},
\end{equation}
\begin{equation}
\label{Eq: mu_M}
    \app{\mu}_{n,t}
    =\wsigma_{n,t}
    \left(
    \left(\wsigma_{n}\right)^{-1} \app{\mu}_{n} 
    +\frac{1}{\sigma^2} \B{V}_{n,t-1} \theta_{n}
    +\frac{1}{\sigma^2} \B{A}_{n,t-1}^{\top} \Xi_{n,t-1}
    \right).
\end{equation}
For $\KQB$, the prior is $\N\left(\mu_{*},\ssigma \right)$
and the inner-instance posterior is,
\begin{equation}
\label{Eq: sigma_K}
    \B{\Sigma}_{*,t}
    =\left(\ssigma^{-1} 
    +\frac{1}{\sigma^2} \B{V}^{\mathrm{K}}_{n,t-1} \right)^{-1},
\end{equation}
\begin{equation}
\label{Eq: mu_K}
    \mu_{*,t}
    =\B{\Sigma}_{*,t}
    \left(
    \ssigma^{-1} \mu_{*} 
    +\frac{1}{\sigma^2} \B{V}^{\mathrm{K}}_{n,t-1} \theta_{n}
    +\frac{1}{\sigma^2} \left( 
    \B{A}^{\mathrm{K}}_{n,t-1}\right)^{\top} \Xi^{\mathrm{K}}_{n,t-1}
    \right).
\end{equation}
\section{SINGLE INSTANCE REGRET PROOF}
\label{Appendix: Single instance regret proof}
In this section we prove Theorem~\ref{Theorem: single instance regret}, which bounds the regret incurred by the incorrect prior within a single instance.
We do so by decomposing the regret (Appendix~\ref{Sec: single instance regret decomposition}) into the good and bad events defined in \eqref{Eq: events for main theorem}. By \eqref{Eq: Gaussian tail bounds upper bound} in Lemma \ref{Lemma: Gaussian tail bounds} and a union bound argument, we have that $\p{\mathcal{E}_{\theta}} > 1 - \frac{\delta}{dT}$.
Thus, if the $\QB$ algorithm maintains $\p{ \mathcal{E}_{v} \cap \mathcal{E}_{m} \cap \mathcal{E}_{s}} \geq 1 - \frac{8\delta}{dT}$ the conditions for Theorem~\ref{Theorem: single instance regret} hold.
\singleInstanceRegret*
\begin{flalign}
    \nonumber
    M &= \max \left\{
    3,
    c_s^2 \tau^2 \fsd, 18 c_{\xi}^2 c_s \left( \fmd + \left( c_1 d + c_{\xi}^2 c_s/36 \right) \fsd \right)\right\},&& \\
    \label{Eq: constants for Theorem: single instance regret}
    k_1 &= 
    12\sqrt{c_{\xi}^2 c_s}
    \sqrt{\fmd \delta}
    +\left(c_s \tau + 12\sqrt{ c_{\xi}^2 c_s c_1 d }
    +2c_{\xi}^2 c_s \right)
    \sqrt{\fsd \delta}, && \\
    \nonumber
    c_{s} &= \frac{2\sigma^2}{\lbar{\lambda}_{\ssigma}^{2}\lbar{\lambda}_{\Msigma{\A}}},
    \quad c_{\xi} =\sigma \sqrt{5\ln\left(\frac{dT}{\delta}\right)}, 
    \quad c_1 = \frac{2}{{\lbar{\lambda}_{\ssigma}}}\ln \left(\frac{d^2T}{\delta} \right),
    \quad c_{\textrm{bad}} = 22a \left(m+\sqrt{4\bar{\lambda}_{\Msigma{*}} \ln \left(\frac{d^2T}{\delta}\right)}\right).&&
\end{flalign}
\textbf{Note} 
The expression for $M$ depends polylogarithmicly on $\nicefrac{1}{\delta}$, which in turn has to satisfy $0 < \delta \leq \nicefrac{1}{M}$, leading to an implicit inequality for $\delta$. 
We show in Appendix~\ref{Appendix: Derivation of of delta} that there exist $\tilde{M} \geq M$, independent of $\delta$, such that $0 < \delta \leq \nicefrac{1}{\tilde{M}} \leq \nicefrac{1}{M}$ is well defined, while maintaining the same asymptotic behavior.

The expectation in the regret analysis includes all sources of randomness in the problem:
the prior of the $\QB$ algorithm at the start of the instance,
the realization of $\theta$, the actions that were presented to the learners during the instance, the randomness of the algorithms and the received noises.
It is worth mentioning that in this work both the actions and the noises can differ between the algorithms compared.
In most sections we abbreviate some or all of the notations to improve readability.
\begin{equation*}
\mathbb{E}
    =
    \underset{\app{\mu}}{\mathbb{E}}
    \underset{\asigma}{\mathbb{E}}
    \underset{\theta}{\mathbb{E}}
    \underset{\B{\A_{\tau}}}{\mathbb{E}}
    \underset{\B{\A^{\mathrm{K}}_{\tau}}}{\mathbb{E}}
    \underset{\amtau}{\mathbb{E}}
    \underset{\aktau}{\mathbb{E}}
    \underset{\Xitau}{\mathbb{E}}
    \underset{\Xiktau}{\mathbb{E}}.
\end{equation*}
\subsubsection*{Intuition on Theorem~\ref{Theorem: single instance regret}}
In order to provide intuition on Theorem~\ref{Theorem: single instance regret} we analyze the demand $0<\delta \leq \nicefrac{1}{M}$.
To do so, we start from the equivalent demand, $0< \delta M \leq 1$ and derive a stricter version of it. 
Plugging the definition of $M$ and rearranging,
\begin{equation*}
    0<\max \left\{
    3\delta,
    c_s^2 \tau^2 \fsd \delta, 18 c_{\xi}^2 c_s \left( \fmd \delta+ \left( c_1 d + c_{\xi}^2 c_s/36 \right) \fsd \delta \right) \right\} \leq 1.
\end{equation*}
Splitting into two demands,
\begin{equation*}
0 < \max \left\{c_s^2 \tau^2 \fsd \delta, 18 c_{\xi}^2 c_s \left( \fmd \delta+ \left( c_1 d + c_{\xi}^2 c_s/36 \right) \fsd \delta \right)\right\}
\leq 1
; \quad 0 < \delta \leq 1/3.
\end{equation*}
Focusing on the first demand, since all the terms are positive, using $\max \left\{a,b\right\} \leq a + b$ yields the stricter demand, 
\begin{equation*}
    18 c_{\xi}^2 c_s \fmd \delta
    +\left(c_s^2 \tau^2 + 18 c_{\xi}^2 c_s \left( c_1 d + c_{\xi}^2 c_s/36 \right) \right) \fsd \delta \leq 1.
\end{equation*}
Taking the square root, using $\sqrt{a+b} \leq \sqrt{a}+\sqrt{b}$, and demanding a stricter condition,
\begin{equation*}
    \sqrt{18 c_{\xi}^2 c_s}  \sqrt{\fmd \delta}
    +\left(c_s \tau
    +\sqrt{18 c_{\xi}^2 c_s c_1 d} +\frac{1}{\sqrt{2}} c_{\xi}^2 c_s\right) 
    \sqrt{\fsd \delta}  \leq 1.
\end{equation*}
From the similarity between the above expression and that for $k_1$ we can conclude the following.
First, every $\delta$ that meets the demand, dictates $k_1$ to be bounded by a constant.
Second, the initial prior deviations in \eqref{Eq: events for main theorem} are bounded by $\sqrt{f_m\delta}$ and $\sqrt{f_s\delta}$, therefore $k_1$ and hence the relative regret have linear dependence in the initial prior deviations.
Third, by Lemma~\ref{Lemma: Minimum eigenvalue of Gram matrix}, $\tau \in \tO(d)$ meets the event $\mathcal{E}_v$, thus the demand holds for $\norm{\app{\mu}-\mu_{*}} \in \tO\left(1\right)$ and ${\big\lVert\asigma -\ssigma\big\rVert}_{\mathrm{op}} \in \tO\left(\nicefrac{1}{d}\right)$.
\subsection{Single Instance Regret Decomposition}
\label{Sec: single instance regret decomposition}
\begin{equation}
\label{Eq: Single instance regret decomposition}
\begin{aligned}
    \E{\RQBK{\app{\mu}}{\asigma}{T}}
    &=
    \E{\RQBK{\app{\mu}_{\tau+1}}{\asigma_{\tau+1}}{T-\tau}}\\
    &=
    \underset{\QB} {\underbrace{
    \E{\RQB{\app{\mu}_{\tau+1}}{\asigma_{\tau+1}}
    {T-\tau}}}}
    -\underset{\KQB} {\underbrace{
    \E{\RQB{\mu_{*,\tau+1}}{\B{\Sigma}_{*,\tau+1}}
    {T-\tau}}}}.
\end{aligned}
\end{equation}
The first equality uses that the $\QB$ algorithm does not incur regret during the exploration time-steps WRT $\KQB$ since both algorithms choose actions randomly with the same distribution and for the same period of time.
Decomposing the regret of $\QB$ in~\eqref{Eq: Single instance regret decomposition} based on the event $\mathcal{E}$,
\begin{equation}
\label{Eq: regular instances regret analysis}
    \E{\RQB{\app{\mu}_{\tau+1}}{\asigma_{\tau+1}}{T-\tau}
    }
    =\underset{\text{``Bad event''}} {\underbrace{
    \E{\RQB{\app{\mu}_{\tau+1}}{\asigma_{\tau+1}}{T-\tau}  \I{\bar{\mathcal{E}}}}}}
    +\underset{\text{``Good event''}}
    {\underbrace{
    \E{\RQB{\app{\mu}_{\tau+1}}{\asigma_{\tau+1}}{T-\tau} 
    \I{\mathcal{E}}}
    }}.
\end{equation}
In \eqref{Eq: regret incurred under the Bad event instances proof} in Section~\ref{Sec: Bad event} we bound the regret incurred under the bad event. We state here the final result,
\begin{equation}
\label{Eq: regret incurred under the Bad event instances}
\begin{aligned}
    \E{\RQB{\app{\mu}_{\tau+1}}{\asigma_{\tau+1}}{T-\tau}  \I{\bar{\mathcal{E}}}}
    &\leq
    \frac{9 c_{\text{bad}}\delta}{11\sqrt{d}},
    \; c_{\text{bad}} \triangleq 
    22a \left(m+\sqrt{4\bar{\lambda}_{\Msigma{*}} \ln \left(\frac{d^2T}{\delta}\right)}\right).
\end{aligned}
\end{equation}
Next, we bound the single instance regret under the good event.
\subsection{Regret Incurred Under the ``Good event''}
Before the beginning of the instance, the learner possesses a prior mean $\mu$ and a prior covariance $\B{\Sigma}$.
Within the instance, this prior is being used by the learner, regardless its origin, whether it is an estimation, heuristics or comes from previous knowledge.
At time-step $t$, given $\theta$ and these priors, the covariance matrix of the learner is only influenced by the actions taken~\eqref{Eq: sigma_general}, while the mean is influenced by both the actions and the noise terms received~\eqref{Eq: mu_general}. 
Thus, we denote,
\begin{align*}
    &\fm{\B{\Sigma}_{t}}{\B{A}_{t-1}}
    =\left(\B{\Sigma}^{-1}
    +\frac{1}{\sigma^2} \B{A}_{t-1}^{\top} 
    \B{A}_{t-1}\right)^{-1}
    =\left(\B{\Sigma}^{-1}
    +\frac{1}{\sigma^2} \B{V}_{t-1}\right)^{-1},\\
    &\fmt{\mu_{t}}{\B{A}_{t-1}}{\Xi_{t-1}}
    =\fm{\B{\Sigma}_{t}}{\B{A}_{t-1}}
    \left(
    \B{\Sigma}^{-1}  \mu 
    +\frac{1}{\sigma^2} \B{A}_{t-1}^{\top} 
    \B{A}_{t-1} \theta
    +\frac{1}{\sigma^2} \B{A}_{t-1}^{\top} \Xi_{t-1}
    \right).
\end{align*}
We adapt the notations for the specific case of $\QB$ and $\KQB$ summarized in \eqref{Eq: sigma_M}, \eqref{Eq: mu_M}, \eqref{Eq: sigma_K}, \eqref{Eq: mu_K}.
\subsubsection{Covariance Alignment}
\label{Sec: covariance alignment}
The first step towards bounding the regret incurred under the ``Good event'' is the alignment of the covariance matrices of $\QB$ and $\KQB$.
For each set of actions $\amtau$, we define $\akofam$ as a specific set of actions that may be taken by $\KQB$ and brings~\eqref{Eq: sigma_M} and~\eqref{Eq: sigma_K} into equality. This requires,
\begin{equation}
    \label{Eq: B definition}
    \vkofam =\vmtau -\B{B} \quad ; \quad \B{B} \triangleq \sigma^2 \left(\ssigma^{-1} -\asigma^{-1} \right).
\end{equation}
Even though this requirement is not unique, we may choose a specific mapping between the two set of actions,
\begin{equation}
\label{Eq: A_k of A_m}
    \akofam
    \triangleq \amtau \left(\amtau^{\top}\amtau \right)^{-1/2}
    \left(\amtau^{\top}\amtau -\B{B} \right)^{1/2}
    =\amtau \vmtau^{-1/2}
    \left(\vmtau -\B{B} \right)^{1/2}.
\end{equation}
We first prove that $\vmtau \succ \B{B}$ under the event $\mathcal{E}_v$ and for $\delta \leq \nicefrac{1}{M}$, thus the square root and the inverse exist and $\akofam$ is well defined,
\begin{equation}
\label{Eq: V - B > 0}
    \lmin{\vmtau - \B{B}}
    \underset{(a)}{\geq} \lmin{\vmtau} + \lmin{-\B{B}}
    \underset{(b)}{=}
    \lmin{\vmtau} - \normop{\B{B}}
    \underset{(c)}{\geq} \frac{\lbar{\lambda}_{\Msigma{\A}} d}{2} - \frac
    {\sigma^2 \sqrt{\fsd \delta}}
    {\lbar{\lambda}_{\ssigma}^{2}}
    \underset{(d)}{>} 0,
\end{equation}
where $(a)$ uses Weyl's inequality,
$(b)$ uses that $\B{B}$ is PSD by Lemma~\ref{Lemma: A-B is PSD} and that for PSD matrices under the $l_2$-operator norm $\lmax{\B{A}}=\smax{\B{A}}=\normop{\B{A}}$,
$(c)$ uses $\mathcal{E}_v$ for the first term and Lemma~\ref{Lemma: B bound}, $\mathcal{E}_s$ for the latter
and $(d)$ uses that $\frac
    {\sigma^2 \sqrt{\fsd \delta}}
    {\lbar{\lambda}_{\ssigma}^{2}}<\frac{\lbar{\lambda}_{\Msigma{\A}} d}{2}$ for $\delta \leq \nicefrac{1}{M}$.

For $\aktau$ in  the image of \eqref{Eq: A_k of A_m}, $\vktau \succ 0$, so we may define the inverse function,
\begin{equation}
\label{Eq: A_m of A_k}
    \amofak
    =\aktau \left(\vktau \right)^{-1/2}
    \left(\vktau +\B{B} \right)^{1/2}.
\end{equation} 
Since every action in the first $\tau$ time-steps is chosen independently from the previous actions, we may view the actions as they are drawn from the following distribution, $\mathlarger{f}_{\B{A}}
    \left(\B{A}\right) 
    =\prod_{t=1}^{\tau}
    \mathlarger{f}_{A}(A_{t})$.
We continue from~\eqref{Eq: regular instances regret analysis} by evaluating the integral of the good event over the action space. 
Since the actions in the first $\tau$ time-steps are independent of $\mathcal{E}_{\theta},\mathcal{E}_{m}, \mathcal{E}_{s}$ we often omit them to improve readability.
\begin{equation}
\label{Eq: Good event instances- first eq}
\begin{aligned}
    &\E{\RQB{\app{\mu}_{\tau+1}}
    {\asigma_{\tau+1}}
    {T-\tau}
    \I{\mathcal{E}}}\\
    &\quad =\mathbb{E}
    \Int_{\mathcal{E}_{v}}
    \mathlarger{f}_{\B{A}}
    \left(\amtau\right)
    \RQB
    {\mumofam}
    {\sigmofam}
    {T-\tau}
    d\amtau\\
    &\quad \leq
    \E{
    \underset{\text{Term A}}
    {\underbrace{
    \vphantom{\Int_{\mathcal{E}_{v}}}
    \max_{\mathcal{E}_{v}}
    \left\{
    \frac{\mathlarger{f}_{\B{A}}
    \left(\amtau\right)}
    {\mathlarger{f}_{\B{A}}
    \left(\akofam \right)}
    \right\}}}
    \cdot
    \underset{\text{Term B}}
    {\underbrace{
    \Int_{\mathcal{E}_{v}}
    \mathlarger{f}_{\B{A}}
    \left(\akofam \right)
    \RQB
    {\mumofam}
    {\sigmofam}
    {T-\tau}
    d\amtau}}},
\end{aligned}
\end{equation}
where the inequality uses that the regret is non-negative.
\subsubsection*{Analyzing ``Covariance alignment'' - Term A}
In order to bound term A, we first prove two auxiliary lemmas.
The first lemma proves that for $A_t$ in the support of $\mathlarger{f}_{A}$, $A_t^{\mathrm{K}}$ is in the support as well, where $A_t^{\top}, \left(A_t^{\mathrm{K}}\right)^{\top}$ are the $t_{th}$ rows of the matrices $\amtau, \akofam$, respectively.
\begin{lemma}
\label{Lemma: Term A is bounded}
For every $t \leq \tau$,
\begin{equation*}
    \frac{\I{\norm{A_t} \leq a}}
    {\I{\norm{A_t^{\mathrm{K}}}\leq a}}
    \leq 1.
\end{equation*}
\begin{proof}
\begin{align*}
    \norm{A_t^{\mathrm{K}}}^2
    &=A_t^{\top} 
    \vmtau^{-1/2}
    \left(\vmtau -\B{B} \right)
    \vmtau^{-1/2}
    A_t\\
    &\underset{(a)}{=}
    \norm{A_t}^2
    -A_t^{\top}
    \vmtau^{-1/2} \B{B} \vmtau^{-1/2} 
    A_t\\
    &\leq \norm{A_t}^2
\end{align*}
Where $(a)$ uses that $\vmtau$ is PD and $\B{B}$ is PSD by Lemma~\ref{Lemma: A-B is PSD}, thus $\vmtau^{-1/2} \B{B} \vmtau^{-1/2}$ is PSD.
\end{proof}
\end{lemma}
The second lemma implies that the probability to sample the set of actions $\amtau$ is lower than the probability to sample $\akofam$,
\begin{lemma}
\label{Lemma: Term A destinities are monotonic}
\begin{align}
     \prod_{t=1}^{\tau} \left( \frac{\mathlarger{\tilde{f}}_{A}\big(A_t\big)}
    {\mathlarger{\tilde{f}}_{A}\big(A_t^{\mathrm{K}}\big)}\right) \leq 1
\end{align}
\begin{proof} 
    For $\tilde{f}_A(A)$, s.t. for every $\norm{A_1} \leq \norm{A_2}$ in the support of $f_A$, $\tilde{f}_A(A_1) \geq \tilde{f}_A(A_2)$ the result can be obtained directly from the proof of Lemma~\ref{Lemma: Term A is bounded}. 
    For the case where $\mathlarger{\tilde{f}}_{A}=\N\left(0,\B{\Sigma}\right)$, for some general $\B{\Sigma}$,
    \begin{equation}
    \begin{aligned}
    \prod_{t=1}^{\tau} \left( \frac{\mathlarger{\tilde{f}}_{A}\big(A_t\big)}
    {\mathlarger{\tilde{f}}_{A}\big(A_t^{\mathrm{K}}\big)}\right) 
    &= \prod_{t=1}^{\tau} \left( \frac{\exp\left(-\frac{1}{2} A_t^\top \B{\Sigma}^{-1} A_t \right)}
    {\exp \left(-\frac{1}{2} \left(A_t^{\mathrm{K}}\right)^\top \B{\Sigma}^{-1} A_t^{\mathrm{K}}\right)}\right)\\
    &= \frac{\exp\left(-\frac{1}{2}
    \sum_{t=1}^{\tau}
    A_t^\top \B{\Sigma}^{-1} A_t \right)}
    {\exp \left(-\frac{1}{2}
    \sum_{t=1}^{\tau}
    \left(A_t^{\mathrm{K}}\right)^\top \B{\Sigma}^{-1} A_t^{\mathrm{K}}\right)}\\
    &\underset{(a)}{=}
    \frac{\exp\left(-\frac{1}{2}
    \sum_{t=1}^{\tau}
    \Tr{A_t^\top \B{\Sigma}^{-1} A_t} \right)}
    {\exp \left(-\frac{1}{2}
    \sum_{t=1}^{\tau}
    \Tr{\left(A_t^{\mathrm{K}}\right)^\top \B{\Sigma}^{-1} A_t^{\mathrm{K}}}\right)}\\
    &\underset{(b)}{=}
    \frac{\exp\left(-\frac{1}{2}
    \Tr{\sum_{t=1}^{\tau}
    \B{\Sigma}^{-1}
    A_t
    A_t^\top} \right)}
    {\exp \left(-\frac{1}{2}
    \Tr{
    \sum_{t=1}^{\tau}
    \B{\Sigma}^{-1}
    A_t^{\mathrm{K}}
    \left( A_t^{\mathrm{K}} \right)^\top  }\right)}\\
    &=
    \frac{\exp\left(-\frac{1}{2}
    \Tr{ \B{\Sigma}^{-1} \vmtau} \right)}
    {\exp \left(-\frac{1}{2}
    \Tr{\B{\Sigma}^{-1} \vktau }\right)}\\
    &\underset{(c)}{=}
    \exp\left(-\frac{1}{2}
    \Tr{ \B{\Sigma}^{-1} \B{B}} \right)\\
    &\underset{(d)}{\leq} 1,\\
    \end{aligned}
    \end{equation}
where $(a)$ applies trace on a scalar, 
$(b)$ uses $\Tr{A^\top B}=\Tr{B A^\top}$ and the linearity of the trace, 
$(c)$ uses the definition of $\vktau$ in \eqref{Eq: B definition} and $(d)$ uses Lemma~\ref{Lemma: Non negative trace for product of PSD matrices}.
\end{proof}
\end{lemma}
Term A is finally bounded by Lemma~\ref{Lemma: Term A is bounded} and Lemma~\ref{Lemma: Term A destinities are monotonic}, 
\begin{equation}
\label{Eq: change measure actions -Term A}
    \max_{\mathcal{E}_{v}} \left\{
    \frac{\mathlarger{f}_{\B{A}}
    \left(\amtau\right)}
    {\mathlarger{f}_{\B{A}}
    \left(\akofam \right)}
    \right\}
    =\max_{\mathcal{E}_{v}}
    \left\{
    \prod_{t=1}^{\tau}
    \left(\frac{\mathlarger{\tilde{f}}_{A}\big(A_t\big)}
    {\mathlarger{\tilde{f}}_{A}\big(A_t^{\mathrm{K}}\big)}
    \frac{\I{\norm{A_t} \leq a}}{\I{\norm{A^{\mathrm{K}}_t} \leq a}}
    \right) \right\} 
    \leq 1,
\end{equation}
\subsubsection*{Analyzing ``Covariance alignment'' - Term B}
Term B in \eqref{Eq: Good event instances- first eq} is an integration over $\amtau$.
Denote $\B{J_{A}} = \frac{\partial \aktau}{\partial \amtau}$ as the Jacobian matrix that transforms the integral to $\aktau$.
We further refer to $\B{J_{A}}$ as the \emph{actions Jacobian}. 
\begin{equation}
\label{Eq: change measure actions integral}
\begin{aligned}
    &\Int_{\mathcal{E}_{v}}
    \mathlarger{f}_{\B{A}}
    \left(\akofam \right)
    \RQB
    {\mumofam}
    {\sigmofam}
    {T-\tau}
    d\amtau\\
    &\quad \leq
    \underset{\text{The Jacobian}}
    {\underbrace{
    \vphantom{\Int_{\mathcal{E}_{v}}}
    \underset{\mathcal{E}_{v}}{\max}
    \left\{\frac{1}{\abs{\Det{\B{J_{A}}}}}\right\}}}
    \cdot \underset{\text{The integral}}
    {\underbrace{\Int_{\mathcal{E}_{v}}
    \mathlarger{f}_{\B{A}}
    \left(\akofam \right)
    \RQB
    {\mumofam}
    {\sigmofam}
    {T-\tau}
    \abs{\Det{\B{J_{A}}}}
    d\amtau}},
\end{aligned}
\end{equation}
where the inequality uses that the regret is non-negative.
\subsubsection*{Analyzing ``Covariance alignment'' - Term B - The Jacobian}
The following lemma bounds the determinant of the actions Jacobian.
\begin{restatable}{lemma}{jacobianBound}
\label{Lemma: main lemma Jacobian}
Let matrices $\B{X} \in \mathbb{R}^{n \times d}, \B{B} \in \mathbb{R}^{d \times d}$, such that $\B{B} \succeq 0, \B{X}^{\top} \B{X} \succ \B{B}$.

Define the matrix $\B{U} \triangleq \B{X} \left(\B{X}^{\top} \B{X}\right)^{-1/2} \left(\B{X}^{\top} \B{X} - \B{B}\right)^{1/2}$ and denote the Jacobian matrix $\B{J}=\frac{\partial \B{U}}{\partial \B{X}}$,
then
\begin{equation*}
    \frac{1}{\abs{\Det{\B{J}}}} \leq \left(\frac
    {\Det {\B{X}^{\top} \B{X}}}
    {\Det {\B{X}^{\top} \B{X}-\B{B}}}
    \right)^{n/2}.
\end{equation*}
\end{restatable}
Lemma~\ref{Lemma: main lemma Jacobian} is highly important in the proof, since it allows later to perform a change of measure over the actions space.
The proof is mostly technical and uses properties of Kronecker product and PSD matrices.
Due to its length, it can be found in Appendix~\ref{Sec: Action Jacobian}.

Define $c_s \triangleq \frac{2\sigma^2}
    {\lbar{\lambda}_{\ssigma}^{2}
    \lbar{\lambda}_{\Msigma{\A}}}$, bounding the Jacobian term in \eqref{Eq: change measure actions integral},
\begingroup
\allowdisplaybreaks
\begin{align}
\label{Eq: jacobian actions alignment}
    \nonumber
    \underset{\mathcal{E}_{v}}{\max} \left\{
    \frac{1}
    {\abs{\Det{\B{J_{A}}}}}
    \right\}
    &\underset{(a)}{\leq} \underset{\mathcal{E}_{v}}{\max}
    \left\{\left(
    \frac{\Det{\vmtau}}
    {\Det{\vmtau - \B{B}} 
    }\right)^{\tau/2} \right\}\\
    \nonumber
    &\underset{(b)}{\leq} \underset{\mathcal{E}_{v}}{\max}
    \left\{ \left(
    \frac{ \prod_{j=1}^{d} \lambda_{j} \left(\vmtau\right)}
    { \prod_{j=1}^{d} \left( 
    \lambda_{j} \left(\vmtau\right) 
    +\lambda_{j} \left(-\B{B}\right) 
    \right)} \right)^{\tau/2}
    \right\}\\
    \nonumber
    &=\underset{\mathcal{E}_{v}}{\max}
    \left\{ \left(
    \prod_{j=1}^{d} \left( 1 
    +\frac{\lambda_{j} 
    \left(-\B{B} \right)}
    {\lambda_{j} \left(\vmtau\right)}
    \right) \right)^{-\tau/2}
    \right\}\\
    \nonumber
    &\underset{(c)}{\leq} \underset{\mathcal{E}_{v}}{\max}
    \left\{ \left(1 
    -\frac{\lmax{\B{B}}}
    {\lmin{\vmtau}}
    \right)^{-d\tau/2}
    \right\}\\
    \nonumber
    &\underset{(d)}{=} \underset{\mathcal{E}_{v}}{\max}
    \left\{ \left(1 
    -\frac{\normop{\B{B}}}
    {\lmin{\vmtau}}
    \right)^{-d\tau/2}
    \right\}\\
    \nonumber
    &\underset{(e)}{\leq}
    \left(1 -\frac{c_s}{d}
    \normop{\asigma
    - \ssigma}
    \right)^{-d\tau/2}\\
    \nonumber
    &\underset{(f)}{\leq} 
    \left( 1 -
    \frac{c_s}{d} 
    \sqrt{\fsd \delta}
    \right)^{-d\tau/2}\\
    \nonumber
    &\underset{(g)}{\leq} 
    \left(1 -\frac{c_s \tau}{2}
    \sqrt{\fsd \delta}
    \right)^{-1} \\
    \nonumber
    &= 1 + \frac{c_s \tau \sqrt{\fsd \delta}}
    {2 - c_s \tau \sqrt{\fsd \delta}} \\
    &\underset{(h)}{\leq} 1 + c_s \tau \sqrt{\fsd \delta},
\end{align}
\endgroup
where $(a)$ uses Lemma \ref{Lemma: main lemma Jacobian} and \eqref{Eq: V - B > 0},
$(b)$ uses Lemma \ref{Lemma: Matrix Analysis (Rajendra Bathia)}, that both matrices are symmetric and~\eqref{Eq: V - B > 0},
$(c)$ uses that both $\B{B}$ and $\vmtau$ are PSD matrices,
$(d)$ uses that for PSD matrices under the $l_2$-operator norm $\lmax{\B{A}}=\smax{\B{A}}=\normop{\B{A}}$,
in $(e)$ the numerator uses Lemma \ref{Lemma: B bound} and the denominator uses the event $\mathcal{E}_{v}$,
$(f)$ uses event $\mathcal{E}_{s}$,
$(g)$ uses Bernoulli inequality and that $c_s \tau \sqrt{\fsd \delta} < 2$ for $\delta \leq \nicefrac{1}{M}$
and $(h)$ uses that $c_s \tau \sqrt{\fsd \delta} \leq 1$ for $\delta \leq \nicefrac{1}{M}$.
\subsubsection*{Analyzing ``Covariance alignment'' - Term B - The integral}
Imagine a $\QB$ algorithm with the correct prior covariance $\ssigma$ and a prior mean $\tilde{\mu}$ which is defined by the following scenario: If this algorithm would have taken the specific set of actions $\akofam$
and would have received the specific set of noises $\Xitau$, it would end up with the same mean as $\QB$ at time-step $\tau+1$,
i.e. the vector that brings the following two equations to equality,
\begin{align*}
    &\app{\mu}_{\tau+1}
    =\sigmofam
    \left(\asigma^{-1} \app{\mu}
    +\frac{1}{\sigma^2} \vmtau \theta
    +\frac{1}{\sigma^2} \amtau^{\top} \Xitau
    \right),\\
    &\tilde{\mu}_{\tau+1}
    =\fm{\B{\Sigma}_{*,\tau+1}}
    {\akofam}
    \left(
    \ssigma^{-1} \tilde{\mu} 
    +\frac{1}{\sigma^2} \vkofam \theta
    +\frac{1}{\sigma^2}
    \left( \akofam \right)^{\top} \Xitau
    \right).
\end{align*}
Describing $\tilde{\mu}$ as a function of the terms determined during the instance,
\begin{alignat}{2}
\label{Eq:mu tilde using A_m}
    &\fmt{\tilde{\mu}}{\amtau}{\Xitau}
    &&=\ssigma
    \left(\asigma^{-1} \app{\mu} 
    +\frac{1}{\sigma^2} \B{B} \theta 
    +\frac{1}{\sigma^2} \left(\amtau -\akofam\right)^{\top}  \Xitau \right),\\
\label{Eq:mu tilde using A_k}
    &\fmt{\tilde{\mu}}{\aktau}{\Xitau}
    &&=\ssigma
    \left(\asigma^{-1} \app{\mu} 
    +\frac{1}{\sigma^2} \B{B} \theta
    +\frac{1}{\sigma^2} \left(\amofak
    -\aktau 
    \right)^{\top}  \Xitau \right).
\end{alignat}
Analyzing the integral in \eqref{Eq: change measure actions integral},
\begin{equation}
\label{Eq: Integral actions alignment}
\begin{aligned}
    &\Int_{\mathcal{E}_{v}}
    \mathlarger{f}_{\B{A}}
    \left(\akofam \right)
    \RQB
    {\mumofam}
    {\sigmofam}
    {T-\tau}
    \abs{\Det{\B{J}_A}}
    d\amtau\\
    &\underset{(a)}{\leq}
    \Int_{\vmtau \succ \B{B}}
    \mathlarger{f}_{\B{A}}
    \left(\akofam \right)
    \RQB
    {\mumofam}
    {\sigmofam}
    {T-\tau}
    \abs{\Det{\B{J}_A}}
    d\amtau\\
    &\underset{(b)}{=}
    \Int_{\vmtau \succ \B{B}}
    \mathlarger{f}_{\B{A}}
    \left(\akofam \right)
    \RQB
    {\fmt{\tilde{\mu}_{\tau+1}}
    {\akofam}{\Xitau}}
    {\fm{\B{\Sigma}_{*,\tau+1}}
    {\akofam}}
    {T-\tau}
    \abs{\Det{\B{J}_A}}
    d\amtau\\
    &\underset{(c)}{=}
    \Int_{\vktau \succ 0}
    \mathlarger{f}_{\B{A}}
    \left(\aktau\right)
    \RQB
    {\fmt{\tilde{\mu}_{\tau+1}}
    {\aktau}{\Xitau}}
    {\fm{\B{\Sigma}_{*,\tau+1}}
    {\aktau}}
    {T-\tau}
    d\aktau\\
    &=
    \underset{\aktau}
    {\mathbb{E}} \left[
    \RQB
    {\fmt{\tilde{\mu}_{\tau+1}}
    {\aktau}{\Xitau}}
    {\B{\Sigma}_{*,\tau+1}}{T-\tau}
    \I{\vktau
    \succ 0 } \right],
\end{aligned}
\end{equation}
where $(a)$ uses~\eqref{Eq: V - B > 0} and that the regret is non-negative, 
$(b)$ uses the definition of $\akofam$
and $(c)$ uses a change of measure.

Plugging \eqref{Eq: change measure actions -Term A},
\eqref{Eq: change measure actions integral},
\eqref{Eq: jacobian actions alignment} 
and \eqref{Eq: Integral actions alignment} back to \eqref{Eq: Good event instances- first eq}, 
the regret incurred under the ``Good event'',
\begin{equation}
\label{Eq: Good event instances- second eq}
    \E{\RQB{\app{\mu}_{\tau+1}}
    {\asigma_{\tau+1}}
    {T-\tau}
    \I{\mathcal{E}}}
    \leq
    \underset{\substack{\text{Covariance} \\ \text{alignment cost}}}
    {\underbrace{
    \left(1 + c_s \tau \sqrt{\fsd \delta} \right)}}
    {\mathbb{E}}
    \underset{\aktau}
    {\mathbb{E}} \left[
    \RQB
    {\fmt{\tilde{\mu}_{\tau+1}}
    {\aktau}{\Xitau}}
    {\B{\Sigma}_{*,\tau+1}}{T-\tau}
    \I{\vktau
    \succ 0} \right].
\end{equation}
\subsubsection{Mean Alignment}
\label{Sec: mean alignment}
After the covariance alignment, in order to bound $\QB$ by $\KQB$, we still need to align the mean. 
At this stage of the proof, the actions at the exploration time-steps have already been determined, yet we still have a degree of freedom in the randomness of the noise terms. 
In the Bayesian update rule for the mean~\eqref{Eq: mu_general} the noise terms appear only in the expression $\B{ A_{\tau}}^{\top} \Xi_{\tau}$.
Since in this stage of the proof, both algorithms use the set of actions $\aktau$, we denote $\smtau \triangleq \left(\aktau\right)^{\top} \Xitau$ and 
$\sktau \triangleq  \left( \aktau\right)^{\top} \Xiktau$. 
To comply with this definition we further denote $\mu_{\tau+1} \left[ S_{\tau}\right]$ instead of $\mu_{\tau+1} \left[\aktau, \Xi_{\tau}\right]$.
During the exploration time-steps, the actions are chosen independently of the rewards achieved, hence, independent of the noise terms and of $\theta$.
Given these actions, $\smtau$ is a Gaussian vector, i.e.
\begin{equation*}
    \smtau \mid \aktau
    \sim \N \left(0, \left( \aktau\right)^{\top} \E{\Xitau \Xitau^{\top}} \aktau\right)
    =\N \left(0,\sigma^2 \vktau\right).
\end{equation*}
We denote by $\skofsm$ the vector $\sktau$ that brings the following equations to equality,
\begin{align*}
    \fms{\tilde{\mu}_{\tau+1}}
    {\smtau}
    &=\B{\Sigma}_{*,\tau+1}
    \left(
    \ssigma^{-1} \tilde{\mu} 
    +\frac{1}{\sigma^2} \vktau \theta
    +\frac{1}{\sigma^2} \smtau \right),\\
    \fms{\mu_{*,\tau+1}}{\sktau}
    &=\B{\Sigma}_{*,\tau+1}
    \left(
    \ssigma^{-1} \mu_{*} 
    +\frac{1}{\sigma^2} \vktau \theta
    +\frac{1}{\sigma^2} \sktau \right).
\end{align*}
Define,
\begin{equation}
\label{Eq: G}
    G \triangleq 
    \sigma^2 \asigma^{-1} 
    \left(\app{\mu} -\mu_{*}\right) +\B{B}
    \left(\theta -\mu_{*}\right).
\end{equation}
We get,
\begin{equation}
\label{Eq: S_k as function of S_m}
\begin{aligned}
    \skofsm
    &=\smtau
    +\sigma^2 \ssigma^{-1} \left(\tilde{\mu} -\mu_{*} \right)\\
    &\underset{(a)}{=}
    \smtau
    +\sigma^2 \left(\asigma^{-1} \app{\mu} 
    +\frac{1}{\sigma^2} \B{B} \theta 
    +\frac{1}{\sigma^2} \left(\amofak
    -\aktau 
    \right)^{\top}  \Xitau
    -\ssigma^{-1}\mu_{*} \right)\\
    &=\smtau
    +\sigma^2 \asigma^{-1} \app{\mu}
    +\B{B} \theta  
    +\amofak^{\top}  \Xitau
    -\smtau 
    -\sigma^2 \ssigma^{-1}\mu_{*}\\
    &=\sigma^2 \asigma^{-1} 
    \left(\app{\mu} -\mu_{*} \right)
    +\B{B} \theta  
    +\amofak^{\top}  \Xitau
    +\sigma^2 \left(\asigma^{-1}-\ssigma^{-1}\right)
    \mu_{*}\\
    &=G
    +\amofak^{\top}  \Xitau\\
    &\underset{(b)}{=}
    G
    +\left(\vktau +\B{B}\right)^{1/2} 
    \left(\vktau \right)^{-1/2}
    \smtau,
\end{aligned}
\end{equation}
where $(a)$ uses the definition of $\tilde{\mu}$ from \eqref{Eq:mu tilde using A_k} and $(b)$ uses the definition of $\amofak$ from \eqref{Eq: A_m of A_k}. 

\subsubsection*{Mean Alignment - Regret Decomposition}
Denote $c_{\xi} \triangleq \sigma \sqrt{5\ln\left(\frac{dT}{\delta}\right)}$ and the event,
\begin{equation}
\label{Eq: E_n}
    \mathcal{E}_{\xi} \triangleq
    \left\{\norm{\smtau}
    _{\left(\vktau\right)^{-1}}
    \leq 
    c_{\xi} \sqrt{d}
    \right\}.
\end{equation}
By Lemma \ref{Lemma: Concentration bound SubGaussian vector} with $X=\Xitau$ and $\B{A}=\left(\vktau\right)^{-1/2} \aktau$ we have that,
\begin{equation}
\label{Eq: probability E_n}
    \p{\mathcal{E}_{\xi}\;\middle|\; \aktau} > 1 - \frac{\delta}{dT}.
\end{equation}
Continue from equation~\eqref{Eq: Good event instances- second eq},
\begin{equation}
\label{Eq: good event regret}
\begin{aligned}
    &\mathbb{E}
    \underset{\Xitau}{\mathbb{E}}
    \underset{\aktau}
    {\mathbb{E}}
    \RQB
    {\fms{\tilde{\mu}_{\tau+1}}
    {\smtau}}
    {\B{\Sigma}_{*,\tau+1}}
    {T-\tau}
    \cdot 
    \I{\vktau
    \succ 0 }\\
    &\quad \quad \quad 
    =\mathbb{E}
    \underset{\aktau}
    {\mathbb{E}}
    \underset{\text{small noise terms}}
    {\underbrace{
    \Int_{\mathcal{E}_{\xi}}\frac{e^{-\frac{1}{2}\norm{\smtau}
    _{\left(\vktau\right)^{-1}}^2}}
    {\sqrt{(2 \pi)^{\tau} \Det{\vktau}}}
    \cdot \RQB
    {\fms{\tilde{\mu}_{\tau+1}}
    {\smtau}}
    {\B{\Sigma}_{*,\tau+1}}
    {T-\tau}
    d\smtau}}
    \cdot 
    \I{\vktau
    \succ 0 }\\
    &\quad \quad \quad \quad \quad
    \quad \quad \quad \quad \quad
    \; \
    +\underset{\text{large noise terms}}
    {\underbrace{
    \mathbb{E}
    \underset{\aktau}
    {\mathbb{E}}
    \underset{\B{\smtau}}
    {\mathbb{E}}
    \RQB
    {\fms{\tilde{\mu}_{\tau+1}}
    {\smtau}}
    {\B{\Sigma}_{*,\tau+1}}
    {T-\tau} 
    \I{\bar{\mathcal{E}}_{\xi} }}}
    \cdot 
    \I{\vktau
    \succ 0 }.
\end{aligned}
\end{equation}
The second term complements the first, but uses an expectation notation.
It is bounded by,
\begin{equation}
\label{Eq: large noise terms}
\begin{aligned}
    \E{\RQB
    {\fms{\tilde{\mu}_{\tau+1}}
    {\smtau}}
    {\B{\Sigma}_{*,\tau+1}}
    {T-\tau} 
    \I{\bar{\mathcal{E}}_{\xi} }}
    &\underset{(a)}{\leq}
    \E{\underset{\Xitau}{\max}
    \left\{\RQB
    {\fms{\tilde{\mu}_{\tau+1}}
    {\smtau}}
    {\B{\Sigma}_{*,\tau+1}}
    {T-\tau} \right\}
    \;\middle|\; \mathcal{E}_{\theta} }
    \cdot \p{\bar{\mathcal{E}}_{\xi} \mid \aktau}\\
    &\underset{(b)}{\leq} 
    \frac{\delta}{dT}
    \E{\underset{\Xitau}{\max}
    \left\{\RQB
    {\fms{\tilde{\mu}_{\tau+1}}
    {\smtau}}
    {\B{\Sigma}_{*,\tau+1}}
    {T-\tau} \right\}
    \;\middle|\; \mathcal{E}_{\theta}}\\
    &\underset{(c)}{\leq}     \frac{c_{\text{bad}}\delta}
    {11\sqrt{d}},
\end{aligned}
\end{equation}
where $(a)$ uses that the regret is non-negative
and $(b)$ uses \eqref{Eq: probability E_n} and $(c)$ uses the same derivations as Lemma \ref{Lemma: Expected maximal regret incurred during a single instance}.
\subsubsection*{Analyzing ``Mean alignment'' - small noise terms}
We bound the small noise term by a change of measure of the integral. Given $\aktau$,
\begin{equation*}
    \B{J}_S
    =\frac{\partial \sktau}{\partial \smtau}
    =\left(\vktau +\B{B}\right)^{1/2} 
    \left(\vktau \right)^{-1/2}.
\end{equation*}
Hence,
\begin{equation}
\label{Eq: Det j_s}
    \abs{\Det{\B{J}_S}}
    =\abs{\frac{
    \left(\Det{\vktau +\B{B}}\right)^{1/2}}
    {\left(\Det{\vktau}\right)^{1/2}}}
    =\left(\frac{\Det{\vktau +\B{B}}}
    {\Det{\vktau}}\right)^{1/2}
    \geq 1,
\end{equation}
where the last equality uses determinant laws and that both matrices are PD and the inequality uses $\B{B} \succeq 0$.
\begin{equation}
\begin{aligned}
\label{Eq: decompose small noise into A and B}
    &\Int_{\mathcal{E}_{\xi}}
    {\frac{e^{-\frac{1}{2}\norm{\smtau}
    _{\left(\vktau\right)^{-1}}^2}}
    {\sqrt{(2 \pi)^{\tau} \Det{\vktau}}}
    \cdot \RQB
    {\fms{\tilde{\mu}_{\tau+1}}
    {\smtau}}
    {\B{\Sigma}_{*,\tau+1}}
    {T-\tau}
    d\smtau}\\
    &\quad \leq
    \underset{\text{Term A}}
    {\underbrace{
    \vphantom{\Int_{\mathcal{E}_{\xi}}}
    \underset{\mathcal{E}_{\xi}}{\max} \left\{\frac{e^{-\frac{1}{2}\norm{\smtau}
    _{\left(\vktau\right)^{-1}}^2}}
    {e^{-\frac{1}{2}\norm{\skofsm}
    _{\left(\vktau\right)^{-1}}^2}}
    \right\}}} \cdot
    \underset{\text{Term B}}
    {\underbrace{
    \Int_{\mathcal{E}_{\xi}}
    \frac{e^{-\frac{1}{2}\norm{\skofsm}
    _{\left(\vktau\right)^{-1}}^2}}
    {\sqrt{(2 \pi)^{\tau} \Det{\vktau}}}
    \RQB
    {\fms{\tilde{\mu}_{\tau+1}}
    {\smtau}}
    {\B{\Sigma}_{*,\tau+1}}
    {T-\tau}
    \abs{\Det{\B{J}_S}}
    d\smtau}},
\end{aligned}
\end{equation}
where the inequality uses that the regret is non-negative and \eqref{Eq: Det j_s}.

\subsubsection*{Analyzing ``Mean alignment'' - small noise terms - Term A}
\begin{equation}
\label{Eq: small noise terms a}
\begin{aligned}
    \underset{\mathcal{E}_{\xi}}{\max} \left\{\frac{e^{-\frac{1}{2}\norm{\smtau}
    _{\left(\vktau\right)^{-1}}^2}}
    {e^{-\frac{1}{2}\norm{\skofsm}
    _{\left(\vktau\right)^{-1}}^2}} \right\}
    &=\underset{\mathcal{E}_{\xi}}{\max}
    \left\{\exp \left(\frac{1}{2}
    \norm{\skofsm}
    _{\left(\vktau\right)^{-1}}^2
    -\frac{1}{2} \norm{\smtau}
    _{\left(\vktau\right)^{-1}}^2 \right) \right\}\\
    &\underset{(a)}{=}
    \underset{\mathcal{E}_{\xi}}{\max}
    \left\{\exp \left(\frac{1}{2}
    \norm{G
    +\left(\vktau +\B{B}\right)^{1/2} 
    \left(\vktau \right)^{-1/2}
    \smtau}
    _{\left(\vktau\right)^{-1}}^2
    -\frac{1}{2}\norm{\smtau}
    _{\left(\vktau\right)^{-1}}^2 \right) \right\}\\
    &\underset{(b)}{\leq}
    \underset{\mathcal{E}_{\xi}}{\max}
    \Bigg\{\exp \Bigg(\frac{1}{2}
    \norm{G}
    _{\left(\vktau\right)^{-1}}^2
    +\norm{G}
    _{\left(\vktau\right)^{-1}}
    \norm{\left(\vktau +\B{B}\right)^{1/2} 
    \left(\vktau \right)^{-1/2}
    \smtau}
    _{\left(\vktau\right)^{-1}}\\
    &\quad \quad \quad \quad \quad
    \quad \quad \quad
    +\frac{1}{2}\norm{\left(\vktau +\B{B}\right)^{1/2} 
    \left(\vktau \right)^{-1/2}
    \smtau}
    _{\left(\vktau\right)^{-1}}^2
    -\frac{1}{2}\norm{\smtau}
    _{\left(\vktau\right)^{-1}}^2
    \Bigg) \Bigg\},
\end{aligned}
\end{equation}
where $(a)$ uses the definition of $\skofsm$ in \eqref{Eq: S_k as function of S_m} and $(b)$ uses Cauchy–Schwarz inequality.
    
Defining $c_1 \triangleq \frac{2}{{\lbar{\lambda}_{\ssigma}}}
    \ln \left(\frac{d^2T}{\delta} \right)$ and analyzing the first term of the exponent in \eqref{Eq: small noise terms a},
\begin{equation}
\label{Eq: small noise terms a aux1}
\begin{aligned}
    \norm{G}
    _{\left(\vktau\right)^{-1}}^2
    &\underset{(a)}{=}
    \norm{\sigma^2 \asigma^{-1}
    \left(
    \left(\app{\mu} -\mu_{*}\right) 
    +\left(\asigma -\ssigma\right)\ssigma^{-1} \left(\theta -\mu_{*} \right) \right)}
    _{\left(\vktau\right)^{-1}}^2\\
    &\underset{(b)}{\leq} 
    \frac{2\sigma^4}
    {\lmin{\vktau}}
    \normop{\asigma^{-1}}^2
    \left(\norm{\app{\mu} -\mu_{*}}^2
    +\norm{\left(\asigma -\ssigma\right)\ssigma^{-1}
    \left(\theta -\mu_{*}\right)}^2 \right)\\
    &\underset{(c)}{\leq} 
    \frac{2c_s}{d}
    \left(\norm{\app{\mu} -\mu_{*}}^2
    +\normop{\asigma -\ssigma}^2
    \normop{\ssigma^{-1/2}}^2
    \norm{\ssigma^{-1/2} \left(\theta -\mu_{*}\right)}^2 \right)\\
    &\underset{(d)}{\leq}
    \frac{2c_s  \delta \left(\fmd +
    c_1 d \fsd \right)}
    {d},
\end{aligned}
\end{equation}
where $(a)$ uses the definition of $G$ in~\eqref{Eq: G} and $\B{B}$ in~\eqref{Eq: B definition}, $(b)$ uses Lemma~\ref{Lemma: maximal singularvalue inequality}, $\norm{\B{A}Y} \leq \normop{\B{A}}\norm{Y}$, triangle inequality and $(a+b)^2 \leq 2(a^2+b^2)$,
$(c)$ uses the event $\mathcal{E}_{v}$, that for PSD matrices $\normop{\B{A}}=\lmax{\B{A}}$, $\asigma \succeq \ssigma$, Assumption \ref{Assumption: eigenvalues prior Covariance matrix}, $\norm{\B{A}Y} \leq \normop{\B{A}}\norm{Y}$ and the sub-multiplicative norm property,
$(d)$ uses the triangle inequality, $(a+b)^2 \leq 2(a^2+b^2)$,
the events $\mathcal{E}_{\theta},\mathcal{E}_{m},\mathcal{E}_{s}$,
that for PSD matrices $\normop{\B{A}}=\lmax{\B{A}}$, $\asigma \succeq \ssigma$ and Assumption \ref{Assumption: eigenvalues prior Covariance matrix}.

Analyzing the third term in the exponent in \eqref{Eq: small noise terms a},
\begin{equation}
\label{Eq: small noise terms a aux2}
\begin{aligned}
    \norm{\left(\vktau +\B{B}\right)^{1/2}
    \left(\vktau \right)^{-1/2}
    \smtau}^2
    _{\left(\vktau \right)^{-1}}
    &=\smtau^{\top}
    \left(\vktau \right)^{-1/2}
    \left(\vktau +\B{B}\right)^{1/2} 
    \left(\vktau \right)^{-1}
    \left(\vktau +\B{B}\right)^{1/2}
    \left(\vktau \right)^{-1/2}
    \smtau\\
    &\underset{(a)}{\leq} \lmax{
    \left(\vktau +\B{B}\right)^{1/2} 
    \left(\vktau \right)^{-1}
    \left(\vktau +\B{B}\right)^{1/2}}
    \norm{\smtau}
    _{\left(\vktau \right)^{-1}}^2\\
    &\underset{(b)}{=}\lmax{
    \left(\vktau \right)^{-1/2}
    \left(\vktau +\B{B}\right)
    \left(\vktau \right)^{-1/2}}
    \norm{\smtau}
    _{\left(\vktau \right)^{-1}}^2\\
    &=\lmax{
    \B{I}
    +\left(\vktau \right)^{-1/2}
    \B{B}
    \left(\vktau \right)^{-1/2}}
    \norm{\smtau}
    _{\left(\vktau \right)^{-1}}^2\\
    &=\left(1+\lmax{
    \left(\vktau \right)^{-1/2}
    \B{B}
    \left(\vktau \right)^{-1/2}}\right)
    \norm{\smtau}
    _{\left(\vktau \right)^{-1}}^2\\
    &\underset{(c)}{\leq}
    \left(1+
    \frac{\lmax{\B{B}}}
    {\lmin{\vktau}}
    \right)
    \norm{\smtau}
    _{\left(\vktau \right)^{-1}}^2\\
    &\underset{(d)}{\leq}
    \left(1+
    \frac{c_s}{d}
    \normop{\asigma - \ssigma}
    \right)
    \norm{\smtau}
    _{\left(\vktau \right)^{-1}}^2\\
    &\underset{(e)}{\leq}
    \left(1+
    \frac{c_s}{d}
    \sqrt{\fsd \delta}
    \right)
    \norm{\smtau}
    _{\left(\vktau \right)^{-1}}^2,
\end{aligned}
\end{equation}
where $(a)$ uses the same derivation as in Lemma \ref{Lemma: maximal singularvalue inequality},
$(b)$ uses Lemma \ref{Lemma: Same eigenvalues for product of PSD matrices},
$(c)$ uses Lemma \ref{Lemma: Same eigenvalues for product of PSD matrices} and sub-multiplicative norm properties,
$(d)$ uses Lemma \ref{Lemma: B bound} and event $\mathcal{E}_{v}$
and $(e)$ uses event $\mathcal{E}_{s}$.

Plugging \eqref{Eq: small noise terms a aux1}, \eqref{Eq: small noise terms a aux2} to \eqref{Eq: small noise terms a}
\begin{align}
\label{Eq: small noise terms a result}
    \nonumber\underset{\mathcal{E}_{\xi}}{\max} \left\{\frac{e^{-\frac{1}{2}\norm{\smtau}
    _{\left(\vktau\right)^{-1}}^2}}
    {e^{-\frac{1}{2}\norm{\skofsm}
    _{\left(\vktau\right)^{-1}}^2}} \right\}
    &\underset{(a)}{\leq}
    \exp \left(\frac{c_s \delta
    \left(\fmd + c_1 d \fsd\right)}{d}
    +2 c_{\xi} \sqrt{c_s \delta
    \left(\fmd + c_1 d \fsd\right)}
    +\frac{c_{\xi}^2 c_s \sqrt{\fsd \delta}}{2}
    \right)\\
    \nonumber
    &\underset{(b)}{\leq}
    \exp \left(
    3 \sqrt{c_{\xi}^2 c_s \delta
    \left(\fmd + c_1 d \fsd\right)}
    +\frac{c_{\xi}^2 c_s \sqrt{\fsd \delta}}{2}
    \right)\\
    &\underset{(c)}{\leq} 1 +
    2\left(
    3 \sqrt{c_{\xi}^2 c_s \delta
    \left(\fmd + c_1 d \fsd\right)}
    +\frac{c_{\xi}^2 c_s \sqrt{\fsd \delta}}{2} \right)\\
    \nonumber
    &\underset{(d)}{\leq}
    1 +
    6 \sqrt{c_{\xi}^2 c_s}
    \sqrt{\fmd \delta}
    +6\left(\sqrt{c_{\xi}^2 c_s c_1 d }
    +c_{\xi}^2 c_s\right)
    \sqrt{\fsd \delta},
\end{align}
where $(a)$ uses that $1+
    \frac{c_s}{d}
    \sqrt{\fsd \delta}
    \leq 2$ for $\delta \leq \nicefrac{1}{M}$
and \eqref{Eq: E_n},
$(b)$ uses that $\sqrt{c_s \delta
    \left(\fmd + c_1 d \fsd\right)} \leq c_{\xi} d$ for $\delta \leq \nicefrac{1}{M}$,
$(c)$ uses Lemma \ref{Lemma: exponent bound} and
$\left(3 \sqrt{c_{\xi}^2 c_s
    \left(\fmd + c_1 d \fsd\right)}
    +\frac{1}{2} c_{\xi}^2 c_s \sqrt{\fsd} \right) \sqrt{\delta} \leq 1$ for $\delta \leq \nicefrac{1}{M}$
and $(d)$ uses $\sqrt{a+b} \leq \sqrt{a}+ \sqrt{b}$.

\subsubsection*{Analyzing ``Mean alignment'' - small noise terms - Term B}
\begin{equation}
\label{Eq: small noise terms b}
\begin{aligned}
    &\Int_{\mathcal{E}_{\xi}}
    \frac{e^{-\frac{1}{2}\norm{\skofsm}
    _{\left(\vktau\right)^{-1}}^2}}
    {\sqrt{(2 \pi)^{\tau} \Det{\vktau}}}
    \RQB
    {\fms{\tilde{\mu}_{\tau+1}}
    {\smtau}}
    {\B{\Sigma}_{*,\tau+1}}
    {T-\tau}
    \abs{\Det{\B{J}_S}}
    d\smtau\\
    &\quad \underset{(a)}{\leq}
    \Int \frac{e^{-\frac{1}{2}\norm{\skofsm}
    _{\left(\vktau\right)^{-1}}^2}}
    {\sqrt{(2 \pi)^{\tau} \Det{\vktau}}}
    \RQB
    {\fms{\tilde{\mu}_{\tau+1}}
    {\smtau}}
    {\B{\Sigma}_{*,\tau+1}}
    {T-\tau}
    \abs{\Det{\B{J}_S}}
    d\smtau\\
    &\quad \underset{(b)}{=}
    \Int \frac{e^{-\frac{1}{2}\norm{\skofsm}
    _{\left(\vktau\right)^{-1}}^2}}
    {\sqrt{(2 \pi)^{\tau} \Det{\vktau}}}
    \RQB
    {\fms{\mu_{*,\tau+1}}{\skofsm}}
    {\B{\Sigma}_{*,\tau+1}}
    {T-\tau}
    \abs{\Det{\B{J}_S}}
    d\smtau\\
    &\quad \underset{(c)}{=}
    \E{\RQB
    {\mu_{*,\tau+1}}
    {\B{\Sigma}_{*,\tau+1}}
    {T-\tau}},
\end{aligned}
\end{equation}
where $(a)$ uses that the regret is non-negative, $(b)$ uses the definition of $\skofsm$ 
and $(c)$ uses change of measure.

Denote $k_1 \triangleq 12 \sqrt{c_{\xi}^2 c_s}
    \sqrt{\fmd \delta} + \left(c_s \tau + 12\sqrt{c_{\xi}^2 c_s c_1 d }
    +2c_{\xi}^2 c_s \right)
    \sqrt{\fsd \delta}$ and
plugging equations \eqref{Eq: small noise terms a result}, \eqref{Eq: small noise terms b}, \eqref{Eq: decompose small noise into A and B} and \eqref{Eq: large noise terms} back to \eqref{Eq: good event regret} and \eqref{Eq: Good event instances- second eq} we get that the regret incurred under the ``Good event'',
\begin{equation}
\label{Eq: good event regret final}
\begin{aligned}
    &\E{\RQB{\app{\mu}_{\tau+1}}
    {\asigma_{\tau+1}}
    {T-\tau}
    \I{\mathcal{E}}}\\
    &\leq 
   \left(1 +c_s \tau \sqrt{\fsd \delta} \right)
    \left[
    \left(1 +
    6 \sqrt{c_{\xi}^2 c_s}
    \sqrt{\fmd \delta}
    +\left(6 \sqrt{c_{\xi}^2 c_s c_1 d }
    +c_{\xi}^2 c_s\right)
    \sqrt{\fsd \delta} \right)
    \E{ \RQB
    {\mu_{*,\tau+1}}
    {\B{\Sigma}_{*,\tau+1}}
    {T-\tau}}
    +\frac{c_{\text{bad}}\delta}{11\sqrt{d}} \right]\\
    &\leq \left(1 
    +k_1  
    \right) \E{ \RQB
    {\mu_{*,\tau+1}}
    {\B{\Sigma}_{*,\tau+1}}
    {T-\tau}}
    +\frac{2c_{\text{bad}}\delta}{11\sqrt{d}},
\end{aligned}
\end{equation}
where the last inequality uses that $c_s \tau \sqrt{\fsd \delta} \leq 1$ for $\delta \leq \nicefrac{1}{M}$.

Plugging back the regret from the bad event \eqref{Eq: regret incurred under the Bad event instances} and the good event \eqref{Eq: good event regret final} to \eqref{Eq: regular instances regret analysis} and then to \eqref{Eq: Single instance regret decomposition} the proof of Theorem \ref{Theorem: single instance regret} follows.
\subsection{Derivation of \texorpdfstring{$M$}{M}}
\label{Appendix: Derivation of of delta}
Along the proof of Theorem~\ref{Theorem: single instance regret} there were several upper bounds on $\delta$ that hold for $\delta \leq \nicefrac{1}{M}$.
Next, we summarize them to derive the exact expression for $M$.
\begin{flalign}
    &
    \nonumber
    \text{From \eqref{Eq: V - B > 0}, \eqref{Eq: small noise terms a result} (transition $(a)$): }
    \frac
    {\sigma^2 \sqrt{\fsd \delta}}
    {\lbar{\lambda}_{\ssigma}^{2}}<\frac{\lbar{\lambda}_{\Msigma{\A}} d}{2}
    \iff \delta < \left(\frac{\lbar{\lambda}_{\ssigma}^{2}\lbar{\lambda}_{\Msigma{\A}}  d}{2\sigma^2 \sqrt{\fsd}}\right)^2
    \iff \delta \underset{(\text{\romannumeral 1})}{<} \frac{d^2}{c_s^2 \fsd},\\
    &\nonumber
    \text{From \eqref{Eq: jacobian actions alignment}, \eqref{Eq: good event regret final} : }
    c_s \tau \sqrt{\fsd \delta} < 1
    \iff \delta \underset{(\text{\romannumeral 2})}{<} \frac{1}{c_s^2 \tau^2 \fsd},
    \\
    &
    \nonumber
    \text{From \eqref{Eq: small noise terms a result} (transition $(b)$): }
    \sqrt{c_s \delta
    \left(\fmd + c_1 d \fsd\right)} \leq c_{\xi} d
    \iff \delta \underset{(\text{\romannumeral 3})}{\leq}  \frac{c_{\xi}^2 d^2}{c_s
    \left(\fmd + c_1 d \fsd\right)}
    \\
    &\nonumber
    \text{From \eqref{Eq: small noise terms a result} (transition $(c)$): }  \left(3\sqrt{c_{\xi}^2 c_s
    \left(\fmd + c_1 d \fsd\right)}
    +\frac{1}{2} c_{\xi}^2 c_s \sqrt{\fsd} \right) \sqrt{\delta} \leq 1\\
    &\label{Eq: first delta demand}
    \iff \delta \underset{(\text{\romannumeral 4})}{\leq} \frac{1}
    {\left(3\sqrt{c_{\xi}^2 c_s
    \left(\fmd + c_1 d \fsd\right)}
    +\frac{1}{2} c_{\xi}^2 c_s \sqrt{\fsd}\right)^2}.
\end{flalign}
Note that inequalities (\romannumeral 1), (\romannumeral 3) in \eqref{Eq: first delta demand} are directly implied by inequalities (\romannumeral 2), (\romannumeral 4) respectively.
To ensure inequality (\romannumeral 4) we tighten its Right-Hand-Side (RHS),
 \begin{equation}
 \label{Eq: second delta demand}
    \begin{aligned}
     \left(3\sqrt{c_{\xi}^2 c_s \left(\fmd + c_1 d \fsd\right)} 
        +\frac{1}{2} c_{\xi}^2 c_s \sqrt{\fsd} \right)^2 
        & \underset{(\text{\romannumeral 5})}{\leq}  2 \left(9 c_{\xi}^2 c_s \left(\fmd + c_1 d \fsd\right) + \frac{1}{4} c_{\xi}^4 c_s^2 \fsd \right) \\
        &= 2 c_{\xi}^2 c_s \left(9  \left(\fmd + c_1 d \fsd\right) + \frac{1}{4} c_{\xi}^2 c_s \fsd \right) \\
        &= 18 c_{\xi}^2 c_s \left( \fmd + c_1 d \fsd + \frac{1}{36} c_{\xi}^2 c_s \fsd \right) \\
        &= 18 c_{\xi}^2 c_s \left( \fmd + \left( c_1 d + \frac{c_{\xi}^2 c_s}{36} \right) \fsd \right) \\
    \end{aligned}
    \end{equation}
Using (\romannumeral 2), (\romannumeral 5) and adding $3$ to ensure that $M \geq e$,
\begin{equation*}
    M =  \max \left\{
    3,
    c_s^2 \tau^2 \fsd, 18 c_{\xi}^2 c_s \left( \fmd + \left( c_1 d + c_{\xi}^2 c_s/36 \right) \fsd \right)\right\}.
\end{equation*}
In the remainder of the section, we emphasize the dependence of $M$ on $\delta$ by $M(\delta)$, even though it is only polylogarithmic.
Specifically, since the constants and $\fmd, \fsd, \tau$ have at most logarithmic dependence, the expression for $M$ is a $p$-degree polynomial of $\ln \left(\nicefrac{1}{\delta} \right)$, for $p \leq 3$.
By definition, $M\left(\delta\right) > e$, thus $\delta < \nicefrac{1}{e}$ and $\ln \left(\nicefrac{1}{\delta} \right) > 1$,
hence we can upper bound the expression by,
\begin{equation*}
    M(\delta) \leq a\ln^p{\left( \frac{1}{\delta} \right)} + b, \text{ for } p \leq 3,
\end{equation*}
for the minimal $a,b$ that upper bound $M\left(\delta\right)$ and satisfy $a \geq \nicefrac{1}{p^p}$, $b>0$.
Next, we define $\tilde{M}$, which is not a function of $\delta$, yet preserving the same asymptotic dependence, such that for $\delta \leq \nicefrac{1}{\tilde{M}}$, it follows that $\delta \leq \frac{1}{a\ln^p{\left( \frac{1}{\delta} \right)} + b} \leq \frac{1}{M\left(\delta\right)}$.
\begin{equation*}
    \delta^{-1} \geq \tilde{M}
    \triangleq \left( 4pa^{1/p}\ln{\left(2pa^{1/p}\right)} + 2b^{1/p}\right)^{p} 
    \iff \delta^{-1/p} \geq 4pa^{1/p}\ln{\left(2pa^{1/p}\right)} + 2b^{1/p}.
\end{equation*}
Thus using Lemma \ref{Lemma: Log bound} with $x=\delta^{-1/p}$, $a_1=pa^{\nicefrac{1}{p}}$, $b_1=b^{1/p}$,
\begin{equation*}
    \delta^{-1/p}
    \geq pa^{1/p}\ln{\left(\delta^{-1/p}\right)} + b^{1/p}
    = a^{1/p}\ln{\left(\frac{1}{\delta}\right)} + b^{1/p}
    \geq \left( a\ln^p{\left(\frac{1}{\delta}\right)} + b \right)^{1/p},
\end{equation*}
where the last inequality uses that $\left(x+y\right)^{\nicefrac{1}{p}} \leq x^{\nicefrac{1}{p}} +y^{\nicefrac{1}{p}}$ for $p \in \mathbb{N}^+$ and $x,y>0$. Since both sides are positive, we can raise both sides to the power of $-p$ and get,
\begin{equation*}
    \delta 
    \leq \frac{1}{a\ln^p{\left(\frac{1}{\delta}\right)} + b}
    \leq \frac{1}{M\left(\delta\right)}.
\end{equation*}
\subsection{Action Jacobian}
\label{Sec: Action Jacobian}
In this section we prove Lemma~\ref{Lemma: main lemma Jacobian}. The proof uses some technical matrix relations appearing in Section \ref{app:BoundJacobianDet}.
\jacobianBound*
\begin{proof}
\leavevmode

\subsubsection{Jacobian Derivation} 
Using differentials,
\begin{equation}
\label{Eq: differentials}
\begin{aligned}
    d\B{U}&=d\B{X} \left(\B{X}^{\top} \B{X}\right)^{-1/2} \left(\B{X}^{\top} \B{X} -\B{B}\right)^{1/2}\\
    &+\B{X} d\left(\B{X}^{\top} \B{X}\right)^{-1/2} \left(\B{X}^{\top} \B{X} -\B{B}\right)^{1/2}\\
    &+\B{X} \left(\B{X}^{\top} \B{X}\right)^{-1/2} d\left(\B{X}^{\top} \B{X} -\B{B}\right)^{1/2}.
\end{aligned}
\end{equation}

Applying vectorization (Lemma \ref{Lemma: Vecotrization}) on \eqref{Eq: differentials}, where $\otimes$ is the \emph{Kronecker product} and $\oplus$ is the \emph{Kronecker sum},
\begin{equation}
\label{Eq: vectorization of du}
\begin{aligned}
    \mathrm{vec}(d\B{U})&= 
    \underset{\triangleq \B{J_{1}}}{\underbrace{
    \left(\left(\B{X}^{\top} \B{X} -\B{B}\right)^{1/2}
    \left(\B{X}^{\top} \B{X}\right)^{-1/2}
    \otimes \B{I}_{n}\right)}} \ \mathrm{vec}(d\B{X})\\
    &+\underset{\triangleq \B{J}_{2a}}{\underbrace{\left(\left(\B{X}^{\top} \B{X} -\B{B}\right)^{1/2} \otimes \B{X}\right)}} \
    \underset{\text{term  vec}_2}{\underbrace{
    \mathrm{vec}\left(d\left(\B{X}^{\top} \B{X}\right)^{-1/2}\right)}}\\
    &+\underset{\triangleq \B{J}_{3a}}{\underbrace{
    \left(\B{I}_{d} \otimes \B{X} \left(\B{X}^{\top} \B{X}\right)^{-1/2} \right)}}\
    \underset{\text{term vec}_3}{\underbrace{
    \mathrm{vec} \left(d\left(\B{X}^{\top} \B{X} -\B{B}\right)^{1/2} \right)}}.
\end{aligned}
\end{equation}
Define 
$\B{N}_d \triangleq \frac{1}{2}\left(\B{I}_{d^2} + \B{K}_d\right)$ as in page 55 in \cite{magnus2019matrix}, for the commutation matrix $\B{K}_d$ and analyzing term vec$_2$,
\begin{equation}
\label{Eq: Term A}
\begin{aligned}
    &\mathrm{vec}\left(d\left(\B{X}^{\top} \B{X}\right)^{-1/2}\right)\\
    &\underset{(a)}{=}
    -\left(\left(\B{X}^{\top} \B{X}\right)^{-1/2} \otimes \left(\B{X}^{\top} \B{X}\right)^{-1/2}\right)
    \mathrm{vec}\left(d\left(\B{X}^{\top} \B{X}\right)^{1/2}\right)\\
    &\underset{(b)}{=}-\left(\left(\B{X}^{\top} \B{X}\right)^{-1/2} \otimes \left(\B{X}^{\top} \B{X}\right)^{-1/2}\right)
    \left(\left(\B{X}^{\top} \B{X}\right)^{1/2}
    \oplus \left(\B{X}^{\top} \B{X}\right)^{1/2}\right)^{-1}
    \mathrm{vec}\left(d\left(\B{X}^{\top} \B{X}\right)\right)\\
    &\underset{(c)}{=}-\left(\left(\B{X}^{\top} \B{X}\right)^{-1/2} \otimes \left(\B{X}^{\top} \B{X}\right)^{-1/2}\right)
    \left(\left(\B{X}^{\top} \B{X}\right)^{1/2}
    \oplus \left(\B{X}^{\top} \B{X}\right)^{1/2}\right)^{-1}
    \underset{\triangleq \B{J}_{2c}}{\underbrace{
    2\B{N}_d\left(\B{I}_d \otimes \B{X}^{\top}\right)}}
    \mathrm{vec}\left(d\B{X}\right)\\
    &\underset{(d)}{=}-\left(\B{I}_d \otimes \left(\B{X}^{\top} \B{X}\right)^{-1/2}\right)
    \left(\left(\B{X}^{\top} \B{X}\right)^{-1/2} \otimes \B{I}_d\right)
    \left(\left(\B{X}^{\top} \B{X}\right)^{1/2}
    \oplus \left(\B{X}^{\top} \B{X}\right)^{1/2}\right)^{-1}
    \B{J}_{2c} \
    \mathrm{vec}\left(d\B{X}\right)\\
    &\underset{(e)}{=}
    \underset{\triangleq \B{J}_{2b}}{\underbrace{
    -\left(\B{I}_d \otimes \left(\B{X}^{\top} \B{X}\right)^{-1/2}\right)
    \left(\left(\B{X}^{\top} \B{X}\right)
    \otimes \B{I}_d
    +\left(\B{X}^{\top} \B{X}\right)^{1/2} 
    \otimes\left(\B{X}^{\top} \B{X}\right)^{1/2} 
    \right)^{-1}}}
    \B{J}_{2c} \
    \mathrm{vec}\left(d\B{X}\right).
\end{aligned}
\end{equation}
Similarly for term vec$_3$,
\begin{equation}
\label{Eq: Term B}
\begin{aligned}
    \mathrm{vec}\left(d\left(\B{X}^{\top} \B{X} -\B{B}\right)^{1/2}\right)
    &\underset{(b)}{=}
    \left(\left(\B{X}^{\top} \B{X} -\B{B}\right)^{1/2}
    \oplus \left(\B{X}^{\top} \B{X} -\B{B}\right)^{1/2}\right)^{-1}
    \mathrm{vec}\left(d\left(\B{X}^{\top} \B{X} -\B{B}\right)\right)\\
    &\underset{(c)}{=}
    \underset{\triangleq \B{J}_{3b}}{\underbrace{
    \left(\left(\B{X}^{\top} \B{X} -\B{B}\right)^{1/2}
    \oplus \left(\B{X}^{\top} \B{X} -\B{B}\right)^{1/2}\right)^{-1}
    2\B{N}_d\left(\B{I}_d \otimes \B{X}^{\top}\right)}} \
    \mathrm{vec}\left(d\B{X}\right),
\end{aligned}
\end{equation}
where $(a)$ uses table 9.7 in \cite{magnus2019matrix}, $(b)$ uses Lemma \ref{Lemma: square root}, $(c)$ uses table 9.6 in \cite{magnus2019matrix}, $(d)$ uses \eqref{Eq: Mixed product} in Lemma~\ref{Lemma: known properties} and $(e)$ uses \eqref{Eq: Inverse}, \eqref{Eq: Mixed product}, \eqref{Eq: Mixed addition} in Lemma~\ref{Lemma: known properties}.

Plugging back \eqref{Eq: Term A}, \eqref{Eq: Term B} to \eqref{Eq: vectorization of du} and using the definition of the Jacobian (bottom of page 196 in \cite{magnus2019matrix}),
\begin{equation*}
    \B{J}=
     \B{J}_1
    +\underset{\triangleq \B{J}_2}{\underbrace{
    \B{J}_{2a}\B{J}_{2b}\B{J}_{2c}}}
    +\underset{\triangleq \B{J}_3}{\underbrace{
    \B{J}_{3a}\B{J}_{3b}}}.
\end{equation*}
\subsubsection{Non Negative Eigenvalues of \texorpdfstring{$\B{C}$}{C}}
Denote,
\begin{equation}
\label{Eq: C}
    \B{C} \triangleq \frac{1}{2} \B{J}_1^{-1}(\B{J}_2 +\B{J}_3).
\end{equation} 
In order to use Corollary 2.2 in \cite{zhan2005determinantal} (Lemma \ref{Lemma: Minkowski}) with $\B{A}_1 = \B{J}_2 +\B{J}_3$ and $\B{B}_1=\B{J}_1$, we need to prove that the matrix $\B{C}$ has non negative eigenvalues. We start by deriving an expression for the matrix $\B{C}$ and then show that its eigenvalues are the same as the eigenvalues of a multiplication of two PSD matrices $\B{G}_1\B{G}_2$, defined later in the proof.

\subsubsection*{Finding a Matrix With the Same Eigenvalues as $\B{C}$}
Starting with the following auxiliary expressions,
\begin{equation}
\label{Eq: J1 inverse J2 aux} 
\begin{aligned}
    \B{D}_1 &\triangleq
    \B{J}_1^{-1} \left(\left(\B{X}^{\top} \B{X} -\B{B}\right)^{1/2} \otimes \B{X}\right)
    \left(\B{I}_d \otimes \left(\B{X}^{\top} \B{X}\right)^{-1/2}\right)\\
    &=\left(\left(\B{X}^{\top} \B{X} -\B{B}\right)^{1/2}
    \left(\B{X}^{\top} \B{X}\right)^{-1/2}
    \otimes \B{I}_{n}\right)^{-1}
    \left(\left(\B{X}^{\top} \B{X} -\B{B}\right)^{1/2} \otimes \B{X}\right) \left(\B{I}_d \otimes \left(\B{X}^{\top} \B{X}\right)^{-1/2}\right)\\
    &\underset{(a+b)}{=}\left(\left(\B{X}^{\top} \B{X}\right)^{1/2}
    \left(\B{X}^{\top} \B{X} -\B{B}\right)^{-1/2}
    \otimes \B{I}_{n}\right)
    \left(\left(\B{X}^{\top} \B{X} -\B{B}\right)^{1/2} \otimes \B{X}
    \left(\B{X}^{\top} \B{X}\right)^{-1/2}\right)\\
    &\underset{(b)}{=}\left(\left(\B{X}^{\top} \B{X}\right)^{1/2}
    \otimes \B{X}
    \left(\B{X}^{\top} \B{X}\right)^{-1/2}\right),
\end{aligned}
\end{equation}
\begin{equation}
\label{Eq: J1 inverse J3 aux} 
\begin{aligned}
    \B{D}_2&\triangleq \B{J_{1}}^{-1}\left(\B{I}_{d} \otimes \B{X} \left(\B{X}^{\top} \B{X}\right)^{-1/2} \right)\\
    &=\left(\left(\B{X}^{\top} \B{X} -\B{B}\right)^{1/2}
    \left(\B{X}^{\top} \B{X}\right)^{-1/2}
    \otimes \B{I}_{n}\right)^{-1}
    \left(\B{I}_{d} \otimes \B{X} \left(\B{X}^{\top} \B{X}\right)^{-1/2} \right)\\
    &\underset{(a)}{=}\left(\left(\B{X}^{\top} \B{X}\right)^{1/2}
    \left(\B{X}^{\top} \B{X} -\B{B}\right)^{-1/2}
    \otimes \B{I}_{n}\right)
    \left(\B{I}_{d} \otimes \B{X} \left(\B{X}^{\top} \B{X}\right)^{-1/2} \right)\\
    &\underset{(b)}{=}\left(\left(\B{X}^{\top} \B{X}\right)^{1/2}
    \left(\B{X}^{\top} \B{X} -\B{B}\right)^{-1/2}
    \otimes \B{X} \left(\B{X}^{\top} \B{X}\right)^{-1/2} \right)\\
    &\underset{(b)}{=}\left(\left(\B{X}^{\top} \B{X}\right)^{1/2}
    \otimes \B{X}\right)
    \left(\left(\B{X}^{\top} \B{X} -\B{B}\right)^{-1/2} \otimes  \left(\B{X}^{\top} \B{X}\right)^{-1/2} \right)\\
    &\underset{(b)}{=}\left(\left(\B{X}^{\top} \B{X}\right)^{1/2}
    \otimes \B{X}\right)
    \left(\B{I}_d \otimes  \left(\B{X}^{\top} \B{X}\right)^{-1/2} \right)
    \left(\left(\B{X}^{\top} \B{X} -\B{B}\right)^{-1/2} \otimes  \B{I}_d \right)\\
    &\underset{(a+b)}{=}\left(\left(\B{X}^{\top} \B{X}\right)^{1/2}
    \otimes \B{X} \left(\B{X}^{\top} \B{X}\right)^{-1/2}\right)
    \left(\left(\B{X}^{\top} \B{X} -\B{B}\right)^{1/2} \otimes  \B{I}_d \right)^{-1},
\end{aligned}
\end{equation}
where $(a)$ uses \eqref{Eq: Inverse} and $(b)$ uses \eqref{Eq: Mixed product}, both from Lemma~\ref{Lemma: known properties}.

Using \eqref{Eq: J1 inverse J2 aux},
\begin{equation}
\label{Eq: J1 inverse J2}
\begin{aligned}
    \frac{1}{2} \B{J}_1^{-1}\B{J}_2
    &= -\B{D}_1 \left(\left(\B{X}^{\top} \B{X}\right)
    \otimes \B{I}_d
    +\left(\B{X}^{\top} \B{X}\right)^{1/2} 
     \otimes \left(\B{X}^{\top} \B{X}\right)^{1/2}
    \right)^{-1}
    \B{N}_d\left(\B{I}_d \otimes \B{X}^{\top}\right) \\
    &=-\left(\left(\B{X}^{\top} \B{X}\right)^{1/2}
    \otimes \B{X}
    \left(\B{X}^{\top} \B{X}\right)^{-1/2}\right) \left(\left(\B{X}^{\top} \B{X}\right)
    \otimes \B{I}_d
    +\left(\B{X}^{\top} \B{X}\right)^{1/2} 
     \otimes \left(\B{X}^{\top} \B{X}\right)^{1/2}
    \right)^{-1}
    \B{N}_d\left(\B{I}_d \otimes \B{X}^{\top}\right).
\end{aligned}
\end{equation}
Using \eqref{Eq: J1 inverse J3 aux},
\begin{equation}
\label{Eq: J1 inverse J3}
\begin{aligned}
    \frac{1}{2} \B{J}_1^{-1}\B{J}_3
    &= \B{D}_2 \left(\left(\B{X}^{\top} \B{X} -\B{B}\right)^{1/2}
    \oplus \left(\B{X}^{\top} \B{X} -\B{B}\right)^{1/2}\right)^{-1}
    \B{N}_d\left(\B{I}_d \otimes \B{X}^{\top}\right)\\
    &=\left(\left(\B{X}^{\top} \B{X}\right)^{1/2}
    \otimes \B{X} \left(\B{X}^{\top} \B{X}\right)^{-1/2}\right)
    \left(\left(\B{X}^{\top} \B{X} -\B{B}\right)^{1/2}
    \otimes  \B{I}_d \right)^{-1} \left(\left(\B{X}^{\top} \B{X} -\B{B}\right)^{1/2}
    \oplus \left(\B{X}^{\top} \B{X} -\B{B}\right)^{1/2}\right)^{-1}\\
    &\quad \quad \B{N}_d\left(\B{I}_d \otimes \B{X}^{\top}\right)\\
    &\underset{(a)}{=}
    \left(\left(\B{X}^{\top} \B{X}\right)^{1/2}
    \otimes \B{X} \left(\B{X}^{\top} \B{X}\right)^{-1/2}\right) \left(\left(\B{X}^{\top} \B{X} -\B{B}\right)
    \otimes \B{I}_d
    +\left(\B{X}^{\top} \B{X} -\B{B}\right)^{1/2}
    \otimes \left(\B{X}^{\top} \B{X} -\B{B}\right)^{1/2}\right)^{-1}\\
    &\quad \quad 
    \B{N}_d\left(\B{I}_d \otimes \B{X}^{\top}\right),
\end{aligned}
\end{equation}
where $(a)$ uses \eqref{Eq: Mixed product}, \eqref{Eq: Mixed addition} from Lemma~\ref{Lemma: known properties}.

Denote,
\begin{equation*}
   \B{G}_1 
   \triangleq \left(\left(\B{X}^{\top} \B{X} -\B{B}\right)
    \otimes \B{I}_d
    +\left(\B{X}^{\top} \B{X} -\B{B}\right)^{1/2}
    \otimes \left(\B{X}^{\top} \B{X} -\B{B}\right)^{1/2}\right)^{-1} -\left(\left(\B{X}^{\top} \B{X}\right)
    \otimes \B{I}_d
    +\left(\B{X}^{\top} \B{X}\right)^{1/2} 
     \otimes \left(\B{X}^{\top} \B{X}\right)^{1/2}
    \right)^{-1}.
\end{equation*}

Plugging \eqref{Eq: J1 inverse J2} and \eqref{Eq: J1 inverse J3} to \eqref{Eq: C},
\begin{equation}
\label{Eq: A inverse B}
    \B{C}
    =\frac{1}{2} \B{J}_1^{-1}\left( \B{J}_2+\B{J}_3 \right)
    =\left(\left(\B{X}^{\top} \B{X}\right)^{1/2}
    \otimes \B{X} \left(\B{X}^{\top} \B{X}\right)^{-1/2}\right)\B{G}_1 \B{N}_d\left(\B{I}_d \otimes \B{X}^{\top}\right).
\end{equation}

Using Lemma \ref{Lemma: Same non-zero eigenvalues}, $\B{C}$ has the same non-zero eigenvalues as,
\begin{equation}
\label{Eq: Eq: A inverse B aux}
    \B{G}_1
    \B{N}_d\left(\B{I}_d \otimes \B{X}^{\top}\right)
    \left(\left(\B{X}^{\top} \B{X}\right)^{1/2}
    \otimes \B{X} \left(\B{X}^{\top} \B{X}\right)^{-1/2}\right)
    =\B{G}_1
   \B{G}_2,
\end{equation}
where the equality uses \eqref{Eq: Mixed product} from Lemma~\ref{Lemma: known properties} and $\B{G}_2 \triangleq \B{N}_d\left(\left(\B{X}^{\top} \B{X}\right)^{1/2}
    \otimes \left(\B{X}^{\top} \B{X}\right)^{1/2}\right)$.

\subsubsection*{The Matrices $\B{G}_1, \B{G}_2$ are PSD}
In order to prove that $\B{G}_2$ is a PSD matrix, we prove that it is symmetric and that all its eigenvalues are non negative.
Using Theorem 3.1 in \cite{magnus1979commutation}, $\B{K}_d$ is a symmetric matrix with eigenvalues equal $\{-1,1\}$, so using observation 1.1.8 in \cite{horn2012matrix},
$\B{N}_d=\frac{1}{2}\left(\B{I}_{d^2} + \B{K}_d\right)$ is a PSD matrix.  
Using \eqref{Eq: transpose} from Lemma~\ref{Lemma: known properties} and Theorem 2.1 in \cite{magnus2019matrix} we get, that $\left(\left(\B{X}^{\top}\B{X}\right)^{1/2}
\otimes \left(\B{X}^{\top} \B{X}\right)^{1/2}\right)$ is PD.
Using Lemma \ref{Lemma: Same eigenvalues for product of PSD matrices} we get that the eigenvalues of $\B{G}_2$ are the same as the eigenvalues of, 
\begin{equation*}
    \left(\left(\B{X}^{\top}\B{X}\right)^{1/2}
    \otimes \left(\B{X}^{\top} \B{X}\right)^{1/2}\right)^{1/2}
    \B{N}_d
    \left(\left(\B{X}^{\top}\B{X}\right)^{1/2}
    \otimes \left(\B{X}^{\top} \B{X}\right)^{1/2}\right)^{1/2},
\end{equation*}
so the eigenvalues of $\B{G}_2$ are non negative. Showing that $\B{G}_2$ is symmetric,
\begin{align*}
   \B{G}_2^{\top}
    &=
    \left(\B{N}_d\left(\left(\B{X}^{\top} \B{X}\right)^{1/2}
    \otimes \left(\B{X}^{\top} \B{X}\right)^{1/2}\right)\right)^{\top}\\
    &=
    \left(\left(\B{X}^{\top} \B{X}\right)^{1/2}
    \otimes \left(\B{X}^{\top} \B{X}\right)^{1/2}\right)^{\top} \B{N}_d^{\top}\\
    &\underset{(a)}{=}
    \left(\left(\B{X}^{\top} \B{X}\right)^{1/2}
    \otimes \left(\B{X}^{\top} \B{X}\right)^{1/2}\right) \B{N}_d\\
    &\underset{(b)}{=} \B{N}_d
    \left(\left(\B{X}^{\top} \B{X}\right)^{1/2}
    \otimes \left(\B{X}^{\top} \B{X}\right)^{1/2}\right),
\end{align*}
where $(a)$ uses that $\B{N}_d$ is symmetric and $(b)$ is from Theorem 3.9 in \cite{magnus2019matrix}.

Next, we prove that $\B{G}_1$ is a PSD matrix.
Using Lemma \ref{Lemma: A-B is PSD} it holds only if the following matrix is PSD,
\begin{align*}
    \B{H}_1& \triangleq \left(\left(\B{X}^{\top} \B{X}\right)
    \otimes \B{I}_d
    +\left(\B{X}^{\top} \B{X}\right)^{1/2} 
     \otimes \left(\B{X}^{\top} \B{X}\right)^{1/2}
    \right) -\left(\left(\B{X}^{\top} \B{X} -\B{B}\right)
    \otimes \B{I}_d
    +\left(\B{X}^{\top} \B{X} -\B{B}\right)^{1/2}
    \otimes \left(\B{X}^{\top} \B{X} -\B{B}\right)^{1/2}\right)\\
    &\underset{(a)}{=}
    \left(\left(\B{X}^{\top} \B{X}\right)- \left(\B{X}^{\top} \B{X} -\B{B}\right)\right)
    \otimes \B{I}_d
    +\left(\B{X}^{\top} \B{X}\right)^{1/2} 
     \otimes \left(\B{X}^{\top} \B{X}\right)^{1/2} -\left(\B{X}^{\top} \B{X} -\B{B}\right)^{1/2}
    \otimes \left(\B{X}^{\top} \B{X} -\B{B}\right)^{1/2}\\
    &\underset{(b)}{=}
    \underset{\triangleq \B{H}_{1a}}{\underbrace{
    \B{B} \otimes \B{I}_d}}
    +\underset{\triangleq \B{H}_{1b}}{\underbrace{
    \left(\left(\B{X}^{\top} \B{X}\right)
     \otimes \left(\B{X}^{\top} \B{X}\right)\right)^{1/2}
     -\left(\left(\B{X}^{\top} \B{X} -\B{B}\right)
    \otimes \left(\B{X}^{\top} \B{X} -\B{B}\right)\right)^{1/2}}},
\end{align*}
where $(a)$ uses \eqref{Eq: Mixed addition} from Lemma~\ref{Lemma: known properties} and $(b)$ uses Lemma \ref{Lemma: Square root of Kronecker product}.

Since $\B{H}_{1a}$ is a PSD matrix, it is left to show that $\B{H}_{1b}$ is a PSD matrix.
Define $\B{M} \triangleq \B{X}^{\top} \B{X} -\B{B}$,
\begin{align*}
    \B{H}_{1b}&=\left(\left(\B{X}^{\top} \B{X}\right)
     \otimes \left(\B{X}^{\top} \B{X}\right)\right)^{1/2}
     -\left(\B{M}
    \otimes \B{M}\right)^{1/2}\\
    &=\left(\left(\B{M}+\B{B}\right)
     \otimes \left(\B{M}+\B{B}\right)\right)^{1/2}
     -\left(\B{M}
    \otimes \B{M}\right)^{1/2}\\
    &=
    \left(
    \underset{\triangleq \B{K}}{\underbrace{
    \left(\B{M} \otimes \B{M}\right)}}
    +\underset{\triangleq \B{L}}{\underbrace{
    \left(\B{M} \otimes \B{B}
    +\B{B} \otimes \left(\B{M}+\B{B}\right)\right)}}
    \right)^{1/2}
    -\underset{\B{K}}{\underbrace{
    \left(\B{M}\otimes \B{M}\right)}}
    ^{1/2},
\end{align*}
where the last equality uses \eqref{Eq: Mixed addition} from Lemma~\ref{Lemma: known properties}.

The matrix $\B{K}$ is PSD, the matrix $\B{L}$ is PSD since it is an addition of two PSD matrices, thus using Lemma \ref{Lemma: square of A minus square of B is PSD} $\B{H}_{1b}$ is PSD and subsequently $\B{G}_1$ is PSD.

\subsubsection{Bounding the Jacobian Determinant}\label{app:BoundJacobianDet}
Using \eqref{Eq: Eq: A inverse B aux} and Lemma \ref{Lemma: Same eigenvalues for product of PSD matrices} we get that the eigenvalues of $\B{C}$ are the same as the eigenvalues of $\B{G}_1^{1/2}\B{G}_2\B{G}_1^{1/2}$, which are non negative, hence the demands of Lemma~\ref{Lemma: Minkowski} are met. Finally,

\begin{align*}
    \frac{1}{\abs{\Det{\B{J}}}}
    &\underset{(a)}{\leq} \frac{1}{\abs{\Det{\B{J}_1}}}\\
    &=\abs{\Det{\left(\B{X}^{\top} \B{X}-\B{B}\right)^{1/2}
    \left(\B{X}^{\top} \B{X}\right)^{-1/2}
    \otimes \B{I}_{n}}^{-1}}\\
    &\underset{(b)}{=}
    \Det{\left(\B{X}^{\top} \B{X}-\B{B}\right)^{1/2}
    \left(\B{X}^{\top} \B{X}\right)^{-1/2}
    \otimes \B{I}_{n}}^{-1}\\
    &\underset{(c)}{=}
    \Det{\left(\B{X}^{\top} \B{X}\right)^{1/2}
    \left(\B{X}^{\top} \B{X}-\B{B}\right)^{-1/2}
    \otimes \B{I}_{n}}\\
    &\underset{(d)}{=}
    \left(\Det{\left(\B{X}^{\top} \B{X}\right)^{1/2}
    \left(\B{X}^{\top} \B{X}-\B{B}\right)^{-1/2}
    }\right)^{n}
    \left(\Det{\B{I}_{n}}\right)^{d}\\
    &\underset{(e)}{=}
    \left(\frac
    {\Det {\B{X}^{\top} \B{X}}}
    {\Det{\B{X}^{\top} \B{X}-\B{B}}}
    \right)^{n/2},
\end{align*}
where $(a)$ uses Lemma \ref{Lemma: Minkowski}, $(b)$ uses that all the eigenvalues are positive, $(c)$ uses \eqref{Eq: Inverse}, $(d)$ uses Corollary 2.2 in \cite{magnus2019matrix} and $(e)$ uses $\Det{\B{I}}=1$.
\end{proof}
\begin{lemma}
\label{Lemma: Vecotrization}
(Vectorization) Theorem 2.2 in \cite{magnus2019matrix}

For matrices $\B{A} \in \mathbb{R}^{n \times d}$; $\B{B},\B{C} \in \mathbb{R}^{d \times d}$,
\begin{align*}
    &\mathrm{vec}(\B{A}\B{B}\B{C})=(\B{C}^{\top}\B{B}^{\top} \otimes \B{I}_n) \mathrm{vec}(\B{A})\\
    &\mathrm{vec}(\B{A}\B{B}\B{C})=(\B{C}^{\top} \otimes \B{A}) \mathrm{vec}(\B{B})\\
    &\mathrm{vec}(\B{A}\B{B}\B{C})=(\B{I}_d \otimes \B{A}\B{B}) \mathrm{vec}(\B{C}).
\end{align*}
\end{lemma}
\begin{lemma}
\label{Lemma: known properties}
(Kronecker properties) From page 32 in \cite{magnus2019matrix}

If $\B{A}+\B{B}$ and $\B{C}+\B{D}$ exist,
\begin{equation}
\label{Eq: Mixed addition}
    \left(\B{A} + \B{B} \right) \otimes \left(\B{C} + \B{D} \right)
    =\B{A} \otimes \B{C} + \B{A} \otimes \B{D} + \B{B} \otimes \B{C} + \B{B} \otimes \B{D}.
\end{equation}

If $\B{A}\B{C}$ and $\B{B}\B{D}$ exist,
\begin{equation}
\label{Eq: Mixed product}
    \left(\B{A} \otimes \B{B} \right) \left(\B{C} \otimes \B{D} \right)
    =\left(\B{A} \B{C} \right) \otimes \left(\B{B} \B{D}\right).
\end{equation}

\begin{equation}
\label{Eq: transpose}
    \left(\B{A} \otimes \B{B} \right)^{\top}
    =\B{A}^{\top} \otimes \B{B}^{\top}.
\end{equation}

If $\B{A}$ and $\B{B}$ are nonsingular,
\begin{equation}
\label{Eq: Inverse}
    \left(\B{A} \otimes \B{B} \right)^{-1}
    =\B{A}^{-1} \otimes \B{B}^{-1}.
\end{equation}
\end{lemma}
\begin{lemma}
\label{Lemma: Square root of Kronecker product}
(Square root of Kronecker product)

Let $\B{A},\B{B}$ be PSD matrices. 
Then,
\begin{equation*}
    \left(\B{A} \otimes \B{B}\right)^{1/2}
    =\B{A}^{1/2} \otimes \B{B}^{1/2}.
\end{equation*}
\begin{proof}
The square root of a matrix $\B{X}$ is defined such that $\B{X}=\B{X}^{1/2}\B{X}^{1/2}=\left(\B{X}^{1/2}\right)^{\top}\B{X}^{1/2}$ and equivalently, defined in terms of the eigenvalue decomposition $\B{X}=\B{U} \B{\Sigma}_{\B{X}} \B{U}^{\top}$ as $\B{X}^{1/2}=\B{U} \B{\Sigma}_{\B{X}}^{1/2} \B{U}^{\top}$.
Using \eqref{Eq: Mixed product}, 
$\left(\B{A}^{1/2} \otimes \B{B}^{1/2}\right)\left(\B{A}^{1/2} \otimes \B{B}^{1/2}\right)=\B{A} \otimes \B{B}$.
Hence we can define $\left( \B{A} \otimes \B{B} \right)^{1/2} = \left(\B{A}^{1/2} \otimes \B{B}^{1/2}\right)$.
\end{proof}
\end{lemma}
\begin{lemma}
\label{Lemma: square root}
(Square root vectorization)

Let $\B{A} \in \mathbb{R}^{d \times d}$ be PD matrix. Then,
\begin{equation*}
    \mathrm{vec} \left(d\B{A}^{1/2} \right)
    =\left(\B{A}^{1/2}
    \oplus \B{A}^{1/2}\right)^{-1}
    \mathrm{vec} \left(d\B{A}\right).
\end{equation*}
\begin{proof}
\begin{equation*}
    \B{A}=\B{A}^{1/2}\B{A}^{1/2}.
\end{equation*}
Taking the differential from both sides,
\begin{equation*}
    d\B{A}=(d\B{A}^{1/2})\B{A}^{1/2} + \B{A}^{1/2}(d\B{A}^{1/2}),
\end{equation*}
\begin{align*}
    \mathrm{vec}(d\B{A})
    &\underset{(a)}{=}\left(\left(\left(\B{A}^{1/2}\right)^{\top} \otimes \B{I}_d\right)
    +\left(\B{I}_d \otimes \B{A}^{1/2} \right)\right) \mathrm{vec}(d\B{A}^{1/2})\\
    &\underset{(b)}{=} 
    \left(\left(\B{A}^{1/2} \otimes \B{I}_d\right)
    +\left(\B{I}_d \otimes \B{A}^{1/2} \right)\right) \mathrm{vec}(d\B{A}^{1/2})\\
    &=\left(\B{A}^{1/2}
    \oplus \B{A}^{1/2} \right) \mathrm{vec}(d\B{A}^{1/2}),
\end{align*}
where $(a)$ is from Lemma~\ref{Lemma: Vecotrization} and $(b)$ uses that $\B{A}$ is symmetric.
Rearranging the equation, the proof follows.
\end{proof}
\end{lemma}
\begin{lemma}
\label{Lemma: Minkowski}
(Corollary 2.2 in \cite{zhan2005determinantal})

Let $\B{A}_1,\B{B}_1 \in \mathbb{\B{C}}^{n \times n}$ $(n \geq 2)$. If $\B{B}_1$ is invertible and $\mathrm{Re}\{\lambda_k\} \geq 0$ $(k=1,2,....,n)$, where $\lambda(\B{B}_1^{-1}\B{A}_1)= \{\lambda_1,\lambda_2,\ldots,\lambda_n \}$, then,
\begin{equation*}
    \abs{\Det{\B{A}_1+\B{B}_1}} \geq \abs{\Det{\B{A}_1}} + \abs{\Det{\B{B}_1}}.
\end{equation*}
\end{lemma}
\subsection{Regret Incurred Under the ``Bad event''}
\label{Sec: Bad event}
\begin{lemma}
\label{Lemma: Expected maximal regret incurred during a single instance}
(Expected maximal regret incurred during a single instance)
\begin{equation*}
    \underset{\theta}{\mathbb{E}} 
    \left[
    \underset{\mu,\B{\Sigma},A,\Xi}{\max}
    \left\{\RQB{\mu}{\Sigma}{T}\right\}
    \;\middle|\; \bar{\mathcal{E}}_{\theta} \right]
    \leq \frac{c_{\text{bad}} \sqrt{d} T}{11},
    \quad  c_{\text{bad}} \triangleq 
    22a \left(m+\sqrt{4\bar{\lambda}_{\Msigma{*}} \ln \left(\frac{d^2T}{\delta}\right)}\right).
\end{equation*}
\end{lemma}
\begin{proof}
\begin{align}
    \nonumber
    \underset{\theta}{\mathbb{E}}\left[
    \underset{\mu,\Sigma,A,\Xi}{\max}
    \left\{\RQB{\mu}{\B{\Sigma}}{T}\right\}
    \;\middle|\; \bar{\mathcal{E}}_{\theta} \right]
    \nonumber
    &\underset{(a)}{\leq}
    \underset{\theta}{\mathbb{E}}\left[
    \sum_{t=1}^{T}
    \underset{A\in \A_{t}}{\max}
    \left\{ A^{\top} \theta \right\} 
    \;\middle|\; \bar{\mathcal{E}}_{\theta} \right]
    -\underset{\theta}{\mathbb{E}}\left[
    \sum_{t=1}^{T}
    \underset{A\in \A_{t}}{\min}
    \left\{ A^{\top} \theta \right\} 
    \;\middle|\; \bar{\mathcal{E}}_{\theta} \right]\\
    \nonumber
    &\leq 2\underset{\theta}{\mathbb{E}}\left[
    \sum_{t=1}^{T}
    \underset{A\in \A_{t}}{\max}
    \left\{ \abs{A^{\top} \theta} \right\} \;\middle|\; \bar{\mathcal{E}}_{\theta} \right]\\
    \label{Eq: bad event}
    &\underset{(b)}{\leq} 2\underset{\theta}{\mathbb{E}}\left[\sum_{t=1}^{T} 
    \underset{A\in \A_{t}}{\max}
    \left\{
    \norm{A} \norm{\theta}  \right\} \;\middle|\; \bar{\mathcal{E}}_{\theta} \right]\\
    \nonumber
    &\underset{(c)}{\leq} 
    2aT
    \underset{\theta}{\mathbb{E}} 
    \left[\norm{\theta} 
    \;\middle|\; \bar{\mathcal{E}}_{\theta} \right],
\end{align}    
where $(a)$ is the maximal regret of any algorithm, $(b)$ uses Cauchy-schwarz inequality
and $(c)$ uses Assumption~\ref{Assumption: eigenvalues action Covariance matrix}.
Denote $Z \triangleq \ssigma^{-1/2} \left(\theta - \mu_{*}\right)$
and analyzing the expectation,
\begin{equation}
\label{Eq: bound theta}
\begin{aligned}
    \E{\norm{\theta} \mid \bar{\mathcal{E}}_{\theta}}
    &\underset{(a)}{\leq}
    \norm{\mu_{*}}+\E{\norm{\theta - \mu_{*}}\mid \bar{\mathcal{E}}_{\theta}}\\
    &\underset{(b)}{\leq}
    m+\sqrt{\bar{\lambda}_{\Msigma{*}}}
    \E{\norm{Z}\mid \bar{\mathcal{E}}_{\theta}}\\
    &=
    m+\sqrt{\bar{\lambda}_{\Msigma{*}}}
    \E{\sqrt{\sum_{i=1}^{d} Z_{i}^{2}} \;\middle|\; \bar{\mathcal{E}}_{\theta}}\\
    &\underset{(c)}{\leq} 
    m+\sqrt{\bar{\lambda}_{\Msigma{*}}}
    \sqrt{\sum_{i=1}^{d}\E{Z_{i}^{2}\mid \bar{\mathcal{E}}_{\theta}}},
\end{aligned}    
\end{equation}
where $(a)$ uses the triangle inequality, $(b)$ uses Lemma \ref{Lemma: maximal singularvalue inequality} and Assumptions \ref{Assumption: eigenvalues prior Covariance matrix} and \ref{Assumption: prior mean is bounded}, 
and $(c)$ uses Jensen inequality.

The expectation in the last expression can be written as $\E{Z_{i}^{2}\mid \bar{\mathcal{E}}_{\theta}} = \frac{\E{Z_i^2 \cdot \I{Z_i^2>z}}}{\p{Z_i^2>z}}$ for $z \triangleq \sqrt{2\ln \left(\frac{d^2T}{\delta} \right)}$.
Taking note that $Z_i$ is a standard normal variable and calculating the numerator,
\begin{equation}
\label{Eq: indictaor bad event}
\begin{aligned}
    \E{Z_i^2 \cdot \I{Z_i^2>z}}
    &\underset{(a)}{=}
    2\E{Z_i^2 \cdot \I{Z_i<-\sqrt{z}}}\\
    &=\frac{2}{\sqrt{2 \pi}}\Int_{-\infty}^{-\sqrt{z}} Z_i \cdot Z_i e^{-\frac{1}{2}Z_i^2} dZ_{i}\\
    &\underset{(b)}{=} 
    -\frac{2}{\sqrt{2 \pi}} Z_i \cdot e^{-\frac{1}{2}Z_i^2}
    \bigg\rvert_{-\infty}^{-\sqrt{z}}
    +\frac{2}{\sqrt{2 \pi}}\Int_{-\infty}^{-\sqrt{z}} e^{-\frac{1}{2}Z_i^2} dZ_{i}\\
    &\underset{(c)}{=} 
    \sqrt{\frac{2z}{\pi}} \cdot e^{-\frac{1}{2}z}
    +2\Phi \left(-\sqrt{z}\right),
\end{aligned}
\end{equation}
where $(a)$ uses the symmetry of a standard Gaussian distribution,
$(b)$ uses integration by parts
and in $(c)$ $\Phi(\cdot)$ stands for the standard Gaussian CDF.

Calculating the denominator,
\begin{equation}
\label{Eq: bad event probability}
    \p{Z_i^2>z}
    =\p{\abs{Z_i}>\sqrt{z}}
    =2\Phi \left(-\sqrt{z}\right).
\end{equation}
Using the symmetry of a standard Gaussian distribution and \eqref{Eq: Gaussian tail bounds lower bound} in Lemma \ref{Lemma: Gaussian tail bounds},
\begin{equation}
\label{Eq: bad event lower bound gaussian}
    \Phi \left(-\sqrt{z}\right)
    \geq \frac{1}{\sqrt{2\pi}} \frac{z-1}{z^{3/2}} e^{-\frac{1}{2}z}.
\end{equation}
Bounding the second moment given $\bar{\mathcal{E}}_{\theta}$,
\begin{align*}
    \E{Z_i^2 \mid Z_i^2>z}
    &=\frac{\E{Z_i^2 \cdot \I{Z_i^2>z}}}{\p{Z_i^2>z}}\\
    &\underset{(a)}{=}
    \frac{\sqrt{\frac{2z}{\pi}} \cdot e^{-\frac{1}{2}z}
    +2\Phi \left(-\sqrt{z}\right)}
    {2\Phi \left(-\sqrt{z}\right)}\\
    &=\frac{\sqrt{\frac{z}{2\pi}}
    \cdot e^{-\frac{1}{2}z}}
    {\Phi \left(-\sqrt{z}\right)} +1\\
    &\underset{(b)}{\leq}
    \frac{z^2}{z-1}+1\\
    &\underset{(c)}{\leq} z+3 \\
    &\underset{(c)}{\leq} 4z,
\end{align*}
where $(a)$ uses \eqref{Eq: indictaor bad event}, \eqref{Eq: bad event probability},
$(b)$ uses \eqref{Eq: bad event lower bound gaussian}
and $(c)$ uses $\ln \left(\frac{d^2T}{\delta} \right) \geq 2$

Plugging back to \eqref{Eq: bound theta},
\begin{equation}
    \E{\norm{\theta} \mid \bar{\mathcal{E}}_{\theta}}
    \leq 
    m+\sqrt{4\bar{\lambda}_{\Msigma{*}} d\sqrt{2\ln \left(\frac{d^2T}{\delta} \right)}}
    \leq \left( m+\sqrt{4\bar{\lambda}_{\Msigma{*}} \ln \left(\frac{d^2T}{\delta} \right)}\right) \sqrt{d}.
\end{equation}
Plugging into \eqref{Eq: bad event} the proof follows.
\end{proof}
Finally, the regret incurred under the bad event,
\begin{equation}
\label{Eq: regret incurred under the Bad event instances proof}
\begin{aligned}
    \E{\RQB{\app{\mu}_{\tau+1}}{\asigma_{\tau+1}}{T-\tau} \I{\bar{\mathcal{E}}}}
    &\underset{(a)}{\leq} 
    \underset{\theta}{\mathbb{E}} 
    \left[
    \underset{\mu,\Sigma,\amtau}
    {\max}
    \left\{\RQB{\mu}{\B{\Sigma}}{T-\tau}\right\} \;\middle|\; \bar{\mathcal{E}}_{\theta} \right]
    \E{\I{\bar{\mathcal{E}}}} \\
    &\leq 
    \frac{9\delta}{dT}
    \underset{\theta}{\mathbb{E}} 
    \left[
    \underset{\mu,\Sigma,\amtau}
    {\max}
    \left\{\RQB{\mu}{\B{\Sigma}}{T-\tau}\right\}
    \;\middle|\; \bar{\mathcal{E}}_{\theta} \right]\\
    &\underset{(b)}{\leq}
    \frac{9 c_{\text{bad}} \delta}{11\sqrt{d}},
\end{aligned}
\end{equation}
where $(a)$ uses that the events influence only the prior or the actions taken during the instance
and $(b)$ uses Lemma \ref{Lemma: Expected maximal regret incurred during a single instance}.
\subsection{A Demonstration of Theorem~\ref{Theorem: single instance regret}}
\label{Appendix: Theorem 1 implication}
We demonstrate that using Theorem~\ref{Theorem: single instance regret}, a $\QB$ algorithm with an adequate prior is a $(1$+$\alpha)$-\emph{approximation} of $\KQB$, by presenting a case where $k_1$ is constant, using the following values of $\fmd, \fsd, \tau$,
\begin{equation*}
    \tau=
    \max \left\{d, \frac{8a^2}{\lbar{\lambda}_{\Msigma{\A}}}\ln\left(\frac{d^2T}{\delta}\right)\right\}
    \; ; \;
    \norm{\app{\mu}-\mu_{*}} \leq \sqrt{\fmd \delta} = \sqrt{d^2 \delta}
    \; ; \;
    \normop{\asigma-\ssigma} \leq \sqrt{\fsd \delta} = \sqrt{\delta}.
\end{equation*}
The value for $\tau$ ensures that $\mathcal{E}_v$ occurs with probability larger that $1- \frac{\delta}{dT}$ by Lemma \ref{Lemma: Minimum eigenvalue of Gram matrix}. 
In order to find a valid value of $\delta$, we first bound $M$ (defined in Theorem \ref{Theorem: single instance regret}),
\begin{equation*}
    M \leq  \max \left\{
    \underset{\triangleq M_1}{\underbrace{
    3+c_s^2 \tau^2 \fsd}} \ ,\
    \underset{\triangleq M_2}{\underbrace{
    3+ 18 c_{\xi}^2 c_s \left( \fmd + \left( c_1 d + c_{\xi}^2 c_s/36 \right) \fsd \right)}}
    \right\}.
\end{equation*}
We begin by bounding $M_1,M_2$ separately,
\begin{align*}
    M_1
    &\underset{(a)}{\leq} 
    3+c_s^2 \left(d^2 +\left(\frac{8a^2}{\lbar{\lambda}_{\Msigma{\A}}}\right)^2 \ln^{2}\left(\frac{d^2T}{\delta}\right) \right)\\
    &\underset{(b)}{\leq}
    \underset{\triangleq b_{M_1}}{\underbrace{3+c_s^2 d^2 +2c_s^2\left(\frac{8a^2}{\lbar{\lambda}_{\Msigma{\A}}}\right)^2 \ln^{2}\left(d^2T\right)}}
    +\underset{\triangleq a_{M_1}}{\underbrace{
    2c_s^2 \left(\frac{8a^2}{\lbar{\lambda}_{\Msigma{\A}}}\right)^2}} \ln^{2}\left(\frac{1}{\delta}\right),\\
    M_2 
    &\underset{(c)}{\leq} 
    3 + 90
    \sigma^2
    \ln^2\left(\frac{dT}{\delta}\right) c_s \left( d^2 + \frac{4}{{\lbar{\lambda}_{\ssigma}}}
    d + \frac{5\sigma^2c_s}{36} \right)\\
    &\underset{(b)}{\leq}
    \underset{\triangleq b_{M_2}}{\underbrace{
    3 + 180
    \sigma^2 c_s \left( d^2 + \frac{4 d}{{\lbar{\lambda}_{\ssigma}}}
    + \frac{5\sigma^2c_s}{36} \right) \ln^2\left(dT\right)}}
    + \underset{\triangleq a_{M_2}}{\underbrace{180
    \sigma^2 c_s \left( d^2 + \frac{4 d}{{\lbar{\lambda}_{\ssigma}}}
    + \frac{5\sigma^2c_s}{36} \right) }} \ln^2\left(\frac{1}{\delta}\right)
\end{align*}
where $(a)$ uses that $\max \{a^2,b^2\} \leq a^2+b^2$,
$(b)$ uses the $(a+b)^2 \leq 2a^2 + 2b^2$
and $(c)$ uses that $\ln\left(\frac{dT}{\delta}\right) \geq 1$ and
$\ln\left(\frac{d^2T}{\delta}\right) \leq 2 \ln\left(\frac{dT}{\delta}\right)$.
Finally,
\begin{equation*}
    M \leq \max \left\{ M_1, M_2 \right\} \leq \left\{ b_{M_1} + a_{M_1} \ln^2\left(\frac{1}{\delta}\right), b_{M_2} + a_{M_2} \ln^2\left(\frac{1}{\delta}\right) \right\} \triangleq b_{M} + a_{M} \ln^2\left(\frac{1}{\delta}\right).
\end{equation*}
Noticing that $a_M,b_M \in \tO\left(d^2\right)$ and using the same derivations as in Appendix~\ref{Appendix: Derivation of of delta} we choose,
\begin{align*}
    \delta \triangleq \frac{1}{\left( 8a_{M}^{1/2}\ln{\left(4a_{M}^{1/2}\right)} + 2b_{M}^{1/2}\right)^{2}} \Rightarrow \begin{array}{l}
         \norm{\app{\mu}-\mu_{*}} \leq \sqrt{d^2 \delta} = \frac{d}{8a_{M}^{1/2}\ln{\left(4a_{M}^{1/2}\right)} + 2b_{M}^{1/2}} \in \tO(1) \\
         \normop{\asigma-\ssigma} \leq \sqrt{\delta} = \frac{1}{8a_{M}^{1/2}\ln{\left(4a_{M}^{1/2}\right)} + 2b_{M}^{1/2}}  \in \tO(\nicefrac{1}{d})
    \end{array}
\end{align*}
Furthermore, for $\delta$ defined above,
\begin{equation*}
    k_1 = 12\sqrt{c_{\xi}^2 c_s}
    \sqrt{\fmd \delta}
    +\left(c_s \tau + 12\sqrt{ c_{\xi}^2 c_s c_1 d }
    +2c_{\xi}^2 c_s \right)
    \sqrt{\fsd \delta}
    \in \tO(1),
\end{equation*}
The $\tO$ notation indicates that $k_1$ is at most polylogarithmic in $d,T$. We can cancel this dependence by choosing $\fmd=\nicefrac{d^2}{k_1^2}, \fsd=\nicefrac{1}{k_1^2}$ instead, while preserving the same value of $\delta$.
\section{PRIOR ESTIMATION ERROR}
\label{Appendix: Prior estimation error}
In the following section we prove that the prior formed by the $\MQB$ algorithm meets the events of Theorem~\ref{Theorem: single instance regret} with probability greater than $1 - \frac{8}{dnT}$.
\subsection{Good Event Definition and Proof}
\label{Appendix: Good event definition}
The proof of Lemma~\ref{Lemma: M-QB} requires the events $\mathcal{E}_{v}$ to hold for every instance $j \in [n]$ where each event is denoted by $\mathcal{E}_{v_j}$.
The $\MQB$ version of the events, based on the events in~\eqref{Eq: events for main theorem}, is defined as follows,
\begin{alignat}{2}
\label{eq:MBQt_events}
    \nonumber
    &\mathcal{E}_{v\text{($\MQB$)}} &&\triangleq
    \{\mathcal{E}_{v_j} \forall j \in [n]\},\\
    \nonumber
    &\mathcal{E}_{m\text{($\MQB$)}} &&\triangleq  \: \mathcal{E}_{m} \; \text{with } 
    \delta = \frac{1}{n-1}, \quad f_{m,n}=3\left(
    \frac{2\sigma^2}{\lbar{\lambda}_{\Msigma{\A}}d}
    +\bar{\lambda}_{\Msigma{*}}\right)
    \left(d + \ln\left(dnT\right)\right),\\
    &\mathcal{E}_{s\text{($\MQB$)}} &&\triangleq \: \mathcal{E}_{s} \; 
    \text{ with } 
    \delta = \frac{1}{n-1}, \quad f_{s,n}=
    100^2
    \left(\frac{2\sigma^2}{\lbar{\lambda}_{\Msigma{\A}}d}
    +\bar{\lambda}_{\Msigma{*}}
    \right)^2
    \left(5d + 2\ln\left(dnT\right)\right),\\
    \nonumber
    &\mathcal{E}_{n\text{($\MQB$)}} &&\triangleq \{\mathcal{E}_{v\text{($\MQB$)}} \ \cap \ \mathcal{E}_{m\text{($\MQB$)}} \ \cap \ \mathcal{E}_{s\text{($\MQB$)}}\}.
\end{alignat}
\lemmaMQB*
\begin{proof}
Using the union bound on Lemma \ref{Lemma: Minimum eigenvalue of Gram matrix} with $\delta=\nicefrac{1}{N^2}$ for all instances up to the $n_{th}$ instance,
\begin{equation*}
    \p{\mathcal{E}_{v(\MQB)}} \geq 1 - \frac{1}{dNT}.
\end{equation*}
Lemma \ref{Lemma: Mean estimation error} and Lemma \ref{Lemma: Covariance estimation error} for $n>5d+2\ln(dnT)$, with $\eta_{n}=\frac{1}{dnT }$ yield,
\begin{equation*}
    \p{\mathcal{E}_{m(\MQB)}} \geq 1 - \frac{1}{dnT}; \quad \p{\mathcal{E}_{s(\MQB)}} \geq 1 - \frac{6}{dnT}.
\end{equation*}
Using the union bound argument, for $n>5d+2\ln(dnT)$, $\p{\mathcal{E}_{n\text{($\MQB$)}}} \geq 1-\nicefrac{8}{dnT}$.

Next, using Lemma~\ref{Lemma: Log bound} with $a_1=2$, $b_1=5d + 2\ln(dT)$, we get that $n>10d+4\ln(16dT)$ implies $n>5d+2\ln(dnT)$.
\end{proof}
\subsection{Mean Estimation Error}
\label{Sec: Mean estimation error}
The mean estimation error originates from two different sources. 
First, after $n-1$ instances, the learner has interacted only with $n-1$ samples of the prior distribution $\N(\mu_*, \ssigma)$.
Second, at the end of each instance, she only has an estimator $\app{\theta}_j$ for the true value of each sample $\theta_j$. 
More formally,
\begin{align*}
    \norm{\app{\mu}_n-\mu_{*}}
    =\norm{\frac{1}{n-1} \sum_{j=1}^{n-1} \left(  \app{\theta}_j - \mu_{*} \right)} =\norm{\frac{1}{n-1} \sum_{j=1}^{n-1} \left(  
    \underset{\triangleq \rho_j}
    {\underbrace{
    \app{\theta}_j - \theta_j}} +\underset{\triangleq  \Delta_j}
    {\underbrace{
    \theta_j - \mu_{*} }}
    \right)}.
\end{align*}
In order to bound the mean estimation error we first prove that each \emph{inner instance error} $\rho_j$ is unbiased and sub-Gaussian.
Then, we show that a \emph{single instance error} $\app{\theta}_j - \mu_*$ is sub-Gaussian as well. 
\begin{lemma}
\label{Lemma: Unbiased inner instance error}
(Unbiasedness of the inner instance error)

Under the event $\mathcal{E}_{v\text{($\MQB$)}}$, for every instance $j \in [n]$,
\begin{equation*}
    \E{\rho_{j}} = \E{\app{\theta}_j - \theta_j} = 0.
\end{equation*}
\end{lemma}
\begin{proof}
\begin{align*}
    \E{\rho_{j}}
    &=\E{\app{\theta}_j - \theta_j}\\
    &=\E{\B{V}^{-1}_{j,\tau} \sum_{s=1}^{\tau} A_{j,s}x_{j,s} - \theta_{j}}\\
    &=\E{\B{V}^{-1}_{j,\tau} \sum_{s=1}^{\tau} A_{j,s} \left( A_{j,s}^{\top} \theta_{j} +\xi_{j,s}\right)  - \theta_{j}}\\
    &=\E{\B{V}^{-1}_{j,\tau} \sum_{s=1}^{\tau} A_{j,s}  \xi_{j,s}}\\
    &=\sum_{s=1}^{\tau} \E{\B{V}^{-1}_{j,\tau} A_{j,s}  \xi_{j,s}}\\
    &\underset{(a)}{=}
    \sum_{s=1}^{\tau} \E{\B{V}^{-1}_{j,\tau} A_{j,s}} \E{\xi_{j,s}}\\
    &\underset{(b)}{=}
    0,
\end{align*}
where $(a)$ uses that the actions taken in the first $\tau$ time-steps are independent from the noise terms and $(b)$ uses that $\xi_{j,s}$ is a zero-mean noise.
\end{proof}
\begin{lemma}
\label{Lemma: Sub-Gaussian inner instance error}
(Sub-Gaussianity of the inner instance error)

Under the event $\mathcal{E}_{v\text{($\MQB$)}}$, for every instance $j \in [n]$, $\rho_j$ is a $\sqrt{\frac{2\sigma^2}
{\lbar{\lambda}_{\Msigma{\A}}d}}$ sub-Gaussian vector.
\begin{proof}
For any $s \in \R$ and $U \in \R^{d}$ s.t. $\norm{U}=1$
\begin{align*}
    \E{\exp \left(sU^{\top} \rho_j \right)}
    &=\E{\exp \left(sU^{\top}\B{V}^{-1}_{j,\tau} \sum_{s=1}^{\tau} A_{j,s} \xi_{j,s} \right)}\\
    &\underset{(a)}{=}
    \underset{\B{A}_{j,\tau}}{\mathbb{E}}
    \left[
    \E{\exp \left(sU^{\top}\B{V}^{-1}_{j,\tau} \sum_{s=1}^{\tau} A_{j,s} \xi_{j,s} \right)}
    \;\middle|\; \B{A}_{j,\tau}\right]\\
    &\underset{(b)}{\leq}
    \E{\exp \left(\frac{s^2 \sigma^2}{2}
    \sum_{s=1}^{\tau}
    \left(U^{\top}\B{V}^{-1}_{j,\tau} 
    A_{j,s}\right)^2 \right)}\\
    &=
    \E{\exp \left(\frac{s^2 \sigma^2}{2}
    \sum_{s=1}^{\tau}
    A_{j,s}^{\top}
    \B{V}^{-1}_{j,\tau} U
    U^{\top}\B{V}^{-1}_{j,\tau} A_{j,s}
    \right)}\\
    &\underset{(c)}{=}
    \E{\exp \left(\frac{s^2 \sigma^2}{2}
    \Tr{\sum_{s=1}^{\tau} 
    A_{j,s} A_{j,s}^{\top}
    \B{V}^{-1}_{j,\tau} U
    U^{\top}\B{V}^{-1}_{j,\tau} }\right)}\\
    &\underset{(c)}{=}
    \E{\exp \left(\frac{s^2 \sigma^2}{2}
    U^{\top}\B{V}^{-1}_{j,\tau} U \right)}\\
    &\underset{(d)}{\leq}
    \exp \left(
    \frac{s^2}{2}
    \cdot \frac{2\sigma^2}
    {\lbar{\lambda}_{\Msigma{\A}}d} \right),
\end{align*}
where $(a)$ uses the law of total expectation,
$(b)$ uses that the actions taken in the first $\tau$ time-steps are independent from the noise terms, the MGF of a Gaussian variable and the law of total expectation,
$(c)$ uses the linearity and the product properties of the trace
and $(d)$ uses Lemma \ref{Lemma: maximal singularvalue inequality} and the event $\mathcal{E}_{v\text{($\MQB$)}}$.
\end{proof}
\end{lemma}
\begin{lemma}
\label{Lemma: Sub-Gaussian single instance error}
(Sub-Gaussianity of the single instance error)

Under the event $\mathcal{E}_{v\text{($\MQB$)}}$, for every instance $j \in [n]$, $\app{\theta}_j - \mu_{*}$ is a $\sqrt{\frac{2\sigma^2}
{\lbar{\lambda}_{\Msigma{\A}}d}
+\bar{\lambda}_{\Msigma{*}}}
$ sub-Gaussian vector.
\begin{proof}
For any $s \in \R$ and $U \in \R^{d}$ s.t. $\norm{U}=1$
\begin{equation}
\begin{aligned}
\label{Eq: SubGaussian - A}
    \E{\exp \left(sU^{\top} \left(\app{\theta}_j - \mu_{*} \right)\right)}
    &\underset{(a)}{=}
    \E{\exp \left(sU^{\top} \rho_j \right)
    \exp \left(sU^{\top} \Delta_j\right)}\\
    &\underset{(b)}{=}
    \E{\exp \left(sU^{\top} \rho_j \right)}
    \E{\exp \left(sU^{\top} \Delta_j\right)}\\
    &\underset{(c)}{\leq}
    \exp \left( \frac{s^2}{2} \cdot \frac{2\sigma^2}{\lbar{\lambda}_{\Msigma{\A}}d} \right)
    \exp \left(\frac{s^{2}U^{\top}\Msigma{*}U}
    {2}\right)\\
    &\underset{(d)}{\leq}\exp \left( \frac{s^2}{2}
    \left(\frac{2\sigma^2}
    {\lbar{\lambda}_{\Msigma{\A}}d}
    +\bar{\lambda}_{\Msigma{*}}\right)
    \right),
\end{aligned}
\end{equation}
where $(a)$ uses the definitions of $\rho_j,\Delta_j$, 
$(b)$ uses that the actions taken during the first $\tau$ time-steps are independent of the noise terms and $\theta_j$, $(c)$ uses the MGF of the Gaussian variable $U^{\top} \Delta_j$ and Lemma~\ref{Lemma: Sub-Gaussian inner instance error} and $(d)$ uses Lemma \ref{Lemma: maximal singularvalue inequality}, together with Assumption~\ref{Assumption: eigenvalues prior Covariance matrix}.
\end{proof}
\end{lemma}
The following lemma bounds the mean estimation error of $\MQB$ algorithm with high probability.
\begin{lemma}
\label{Lemma: Mean estimation error} (Mean estimation error)

For every $\eta_{n} > 0$ and for every instance $n>1$,
\begin{equation*}
    \p{\mathcal{E}_{m\text{($\MQB$)}} \mid \mathcal{E}_{v\text{($\MQB$)}}} \geq 1 -\eta_{n}.
\end{equation*}
\begin{proof}
We start by proving that $\app{\mu}_n - \mu_{*}$ is a $\sqrt{\frac{2\sigma^2+\lbar{\lambda}_{\Msigma{\A}}\bar{\lambda}_{\Msigma{*}}d}
    {\lbar{\lambda}_{\Msigma{\A}}d(n-1)}}$ sub-Gaussian vector.
For any $s \in \R$ and $U \in \R^{d}$ s.t. $\norm{U}=1$,
\begin{align*}
    \E{\exp \left(s
    U^{\top}
    \left(\app{\mu}_n - \mu_{*} \right)\right)}
    &=\E{\exp \left(\frac{s}{n-1}
    U^{\top}
    \sum_{j=1}^{n-1}
    \left(\app{\theta}_j - \mu_{*} \right)\right)}\\
    &\underset{(a)}{=} 
    \prod_{j=1}^{n-1}{
    \E{\exp \left(\frac{s}{n-1}
    U^{\top}
    \left(\rho_j + \Delta_j \right)\right)}}\\
    &\underset{(b)}{\leq}\exp \left( \frac{s^2}{2} \cdot
    \frac{2\sigma^2+\lbar{\lambda}_{\Msigma{\A}}\bar{\lambda}_{\Msigma{*}}d}
    {\lbar{\lambda}_{\Msigma{\A}}d(n-1)}
    \right),
\end{align*}
where $(a)$ uses that the actions taken during the first $\tau$ time-steps are independent of the noise terms, $\theta_j$ and the inner instance errors of other instances
and $(b)$ uses the same steps as in \eqref{Eq: SubGaussian - A}.

From Lemma \ref{Lemma: Concentration bound SubGaussian vector},
\begin{equation*}
    \p{\norm{\app{\mu}_n - \mu_{*}}^2 > 
    \frac{2\sigma^2+\lbar{\lambda}_{\Msigma{\A}}\bar{\lambda}_{\Msigma{*}}d}
    {\lbar{\lambda}_{\Msigma{\A}}d(n-1)}
    \cdot \left( d+2\sqrt{d\ln \left(\nicefrac{1}{\eta_{n}} \right)}+2\ln \left(\nicefrac{1}{\eta_{n}} \right)\right)}
    \leq \eta_{n}
\end{equation*}
and using the inequality of arithmetic and geometric means, $\sqrt{d \ln \left(\nicefrac{1}{\eta_n} \right)} \leq (d +\ln\left(\nicefrac{1}{\eta_n} \right)) / 2$, the proof follows.
\end{proof}
\end{lemma}
\subsection{Covariance Estimation Error}
\label{Sec: Covariance estimation error}
In this section, we bound the estimation error of the estimated covariance matrix $\asigma_n$ (Eq.~\eqref{Eq: app Sigma_n}), WRT the true covariance matrix $\ssigma$, under the operator norm.
\subsubsection{Covariance Estimation Error Decomposition}
Define,
\begin{equation}
\label{Eq: Sigma}
    \B{\Sigma} \triangleq \frac{\sigma^2}{n-1}
    \sum_{j=1}^{n-1} \E{\B{V}^{-1}_{j,\tau}} +\ssigma, 
\end{equation}
\begingroup
\begin{align}
    \nonumber
    & \normop{\asigma_n - \ssigma} \\
    \nonumber
    & = \normop{
    \frac{1}{n-2} \sum_{j=1}^{n-1}
    \left(\app{\theta}_{j} - \app{\mu}_{n}\right)
    \left(\app{\theta}_{j} - \app{\mu}_{n}\right)^{\top}
    -\frac{\sigma^2}{n-1} \sum_{j=1}^{n-1} \B{V}^{-1}_{j,\tau}
    -\ssigma}\\
    \nonumber
    & \leq  \normop{ 
    \frac{1}{n-2} \sum_{j=1}^{n-1}
    \left(\app{\theta}_{j} - \app{\mu}_{n}\right)
    \left(\app{\theta}_{j} - \app{\mu}_{n}\right)^{\top}
    -\B{\Sigma}}
    +\normop{
     \frac{\sigma^2}{n-1} \left(\sum_{j=1}^{n-1}  \B{V}^{-1}_{j,\tau} - \sum_{j=1}^{n-1} \E{\B{V}^{-1}_{j,\tau}}\right) 
    } \\
    \nonumber
    & \underset{(a)}{=}         
    \normop{
    \frac{1}{n-2} \sum_{j=1}^{n-1}
    \left(\app{\theta}_{j} - \mu_{*}\right)
    \left(\app{\theta}_{j} - \mu_{*}\right)^{\top} 
    -\frac{n-1}{n-2}
    \left(\app{\mu}_{n} - \mu_{*}\right)
    \left(\app{\mu}_{n} - \mu_{*}\right)^{\top} 
    -\B{\Sigma}} \\
    \nonumber
    & \quad \quad \quad +\normop{
     \frac{\sigma^2}{n-1} \left(\sum_{j=1}^{n-1}  \B{V}^{-1}_{j,\tau} - \sum_{j=1}^{n-1} \E{\B{V}^{-1}_{j,\tau}}\right) 
    } \\
    \nonumber
    & \leq 
    \frac{n-1}{n-2}
    \underset{\text{Term A}}{\underbrace{\normop{
    \frac{1}{n-1}\sum_{j=1}^{n-1}
    \left(\app{\theta}_{j} - \mu_{*}\right)
    \left(\app{\theta}_{j} - \mu_{*}\right)^{\top}
    -\B{\Sigma}}}}
    + \frac{n-1}{n-2} \underset{\text{Term B}}{\underbrace{\normop{ \frac{1}{n-1} \sum_{j=1}^{n-1}
    \left(\app{\mu}_{n} - \mu_{*}\right)
    \left(\app{\mu}_{n} - \mu_{*}\right)^{\top} 
    -\frac{1}{n-1} \B{\Sigma}
    }}}\\
    \label{Eq: decompose covariance bound}
    & \quad\quad\quad 
    +\underset{\text{Term C}}{\underbrace{\normop{
     \frac{\sigma^2}{n-1} \left(\sum_{j=1}^{n-1}  \B{V}^{-1}_{j,\tau} - \sum_{j=1}^{n-1} \E{\B{V}^{-1}_{j,\tau}}\right) 
    }}},
\end{align}
\endgroup
where $(a)$ uses $\app{\mu}_n=\frac{1}{n-1}\sum_{j=1}^{n-1}\app{\theta}_j$ and both inequalities come from the triangle inequality.    
   
We bound each of the three terms separately, using Theorem 6.5 in \cite{wainwright2019high} for terms A and B. For convenience, we state it here.
\begin{lemma}
\label{Lemma: Empirical covariance bounds}
(Empirical covariance bounds, Theorem 6.5 in \cite{wainwright2019high}, constants were taken from \cite{bastani2019meta}).

For any row-wise $\sigma$ sub-Gaussian random matrix $\B{X} \in \mathbb{R}^{n \times d}$, the sample covariance matrix $\asigma = \frac{1}{n} 
\sum_{i=1}^{n} X_i X_i^{\top}$ satisfies the bound
\begin{equation*}
    \p{\normop{\asigma-\B{\Sigma}}
    \geq 32\sigma^2 \cdot 
    \max \left\{
    \sqrt{\frac{5d + 2\ln\left(\frac{2}{\delta}\right)}
    {n}},
    \frac{5d + 2\ln\left(\frac{2}{\delta}\right)}{n}
    \right\}}
    \leq \delta
    \quad \forall \; 0 < \delta < 1.
\end{equation*}
\end{lemma}

The first part of Terms A and B in \eqref{Eq: decompose covariance bound} are of the form $\frac{1}{n-1} \sum_{j=1}^{n-1}XX^{\top}$ for vector $X$ equal to $\app{\theta}_j - \mu_{*}$ and $\app{\mu}_n - \mu_{*}$, respectively. 
Lemmas \ref{Lemma: Sub-Gaussian single instance error} and \ref{Lemma: Mean estimation error}, proved during the mean estimation analysis (Appendix~\ref{Sec: Mean estimation error}), show that these vectors are sub-Gaussian.
The proof of Lemma~\ref{Lemma: Empirical covariance bounds}, implicitly requires that $\E{\asigma}=\B{\Sigma}$, i.e. the estimator $\asigma$ is unbiased, which we prove next. 
\subsubsection{Lack of Bias of the Covariance Estimator}
\begin{lemma}
\label{Lemma: Unbiased covariance estimation  - Term C}
(Lack of bias of the covariance estimator  - auxiliary)
\begin{equation*}
    \E{\rho_{j} \rho_{j}^{\top}}
    =\sigma^2 \E{\B{V}^{-1}_{j,\tau}}.
\end{equation*}
\begin{proof}
\begin{equation}
\label{Eq: proof-Approximation error covariance}
\begin{aligned}
    \E{\rho_{j} \rho_{j}^{\top}}
    &=\E{\left( \B{V}^{-1}_{j,\tau} \sum_{s=1}^{\tau} A_{j,s} \xi_{j,s}\right) 
    \left( \B{V}^{-1}_{j,\tau} \sum_{k=1}^{\tau} A_{j,k} \xi_{j,k} \right)^{\top}}\\
    &\underset{(a)}{=}
    \E{\left( \B{V}^{-1}_{j,\tau} \sum_{s=1}^{\tau} A_{j,s} \xi_{j,s}\right) 
    \left( \sum_{k=1}^{\tau} 
    A_{j,k}^{\top} \B{V}^{-1}_{j,\tau} \xi_{j,k} \right)}\\
    &=\underset{\text{same time-step}}
    {\underbrace{
    \vphantom{\E{ \B{V}^{-1}_{j,\tau} \sum_{s=1}^{\tau}
    \sum_{k \neq s}^{\tau}
    A_{j,s} 
    A_{j,k}^{\top} \B{V}^{-1}_{j,\tau}
    \xi_{j,s}
    \xi_{j,k}}}
    \E{\B{V}^{-1}_{j,\tau} \sum_{s=1}^{\tau} A_{j,s}
    A_{j,s}^{\top} \B{V}^{-1}_{j,\tau} \xi_{j,s}^{2}}}}
    +\underset{\text{different time-steps}}
    {\underbrace{
    \E{ \B{V}^{-1}_{j,\tau} \sum_{s=1}^{\tau}
    \sum_{k \neq s}^{\tau}
    A_{j,s} 
    A_{j,k}^{\top} \B{V}^{-1}_{j,\tau}
    \xi_{j,s}
    \xi_{j,k}}}},
\end{aligned}
\end{equation}
where $(a)$ uses that $\B{V}_{j,\tau}$ is symmetric.
    
Analyzing the \emph{same time-step}
\begin{align*}
    \E{\B{V}^{-1}_{j,\tau} \sum_{s=1}^{\tau} A_{j,s}
    A_{j,s}^{\top} \B{V}^{-1}_{j,\tau} \xi_{j,s}^{2}}
    &\underset{(a)}{=}
    \E{ \B{V}^{-1}_{j,\tau} \sum_{s=1}^{\tau} 
    A_{j,s} A_{j,s}^{\top} \B{V}^{-1}_{j,\tau}}
    \E{\xi_{j,s}^{2}}\\
    &=
    \sigma^2 \E{ \B{V}^{-1}_{j,\tau}  
    \left( \sum_{s=1}^{\tau} A_{j,s} A_{j,s}^{\top}  \right)
    \B{V}^{-1}_{j,\tau}}\\
    &=
    \sigma^2 \E{\B{V}^{-1}_{j,\tau}}.
\end{align*}
Analyzing the \emph{different time-steps}
\begin{equation*}
    \E{ \B{V}^{-1}_{j,\tau} \sum_{s=1}^{\tau}
    \sum_{k \neq s}^{\tau}
    A_{j,s} 
    A_{j,k}^{\top} \B{V}^{-1}_{j,\tau}
    \xi_{j,s}
    \xi_{j,k}}
    \underset{(a)}{=}
    \sum_{s=1}^{\tau} \sum_{k \neq s}^{\tau}
    \E{ \B{V}^{-1}_{j,\tau}  A_{j,s}  A_{j,k}^{\top} \B{V}^{-1}_{j,\tau}}
    \E{ \xi_{j,s}}
    \E{ \xi_{j,k}}
    =0,
\end{equation*}
where $(a)$ uses that the actions during the first $\tau$ time-steps are independent of the noise terms.
    
Plugging back to \eqref{Eq: proof-Approximation error covariance} the proof follows.
\end{proof}
\end{lemma}
\begin{lemma}
\label{Lemma: Unbiased covariance estimation  - Term A}
(Lack of bias of the covariance estimator  - Term A)
\begin{equation*}
    \E{\frac{1}{n-1}  \sum_{j=1}^{n-1}
    \left( \app{\theta}_j - \mu_{*} \right) \left( \app{\theta}_j - \mu_{*} \right)^{\top}}
    =\B{\Sigma}.
\end{equation*}
\end{lemma}    
\begin{proof}
\begin{align*}
    \E{\left( \app{\theta}_j - \mu_{*} \right) \left( \app{\theta}_j - \mu_{*} \right)^{\top}}
    &=\E{\app{\theta}_j \app{\theta}_j^{\top}}
    -\E{\app{\theta}_j} \mu_{*}^{\top}
    -\mu_{*} \E{\app{\theta}_j^{\top}}
    +\mu_{*} \mu_{*}^{\top}\\
    &\underset{(a)}{=}
    \E{\left(\theta_j +\rho_j \right)
    \left(\theta_j +\rho_j \right)^{\top}}
    -\mu_{*} \mu_{*}^{\top}\\
    &\underset{(b)}{=}
    \E{\theta_j \theta^{\top}_j} - \mu_{*} \mu_{*}^{\top}
    +\E{\rho_j \rho_j^{\top}}\\
    &\underset{(c)}{=}
    \ssigma 
    +\sigma^2 \E{\B{V}^{-1}_{j,\tau}},
\end{align*}
where $(a)$ uses that $\app{\theta}_j$ is unbiased i.e. $\E{\app{\theta}_j} = \mu_{*}$,
$(b)$ uses that the actions taken during the first $\tau$ time-steps are independent of the noise terms and of $\theta_j$ and Lemma \ref{Lemma: Unbiased inner instance error}
and $(c)$ uses Lemma \ref{Lemma: Unbiased covariance estimation  - Term C}.

Using the definition of $\B{\Sigma}$ in \eqref{Eq: Sigma}, summing over the instances and dividing by $n-1$ the proof follows.
\end{proof}
\begin{lemma}
\label{Lemma: Aux - Unbiased covariance estimation  - Term B}
(Lack of bias of the covariance estimator  - Term B - auxiliary)
\begin{equation*}
    \E{\app{\mu}_n \app{\mu}_n^{\top}}
    =\frac{\B{\Sigma}}{n-1} 
    +\mu_{*} \mu_{*}^{\top}.
\end{equation*}
\end{lemma}
\begin{proof}
\begin{align*}
    \E{\app{\mu}_n \app{\mu}_n^{\top}}
    &=\E{\left(\frac{1}{n-1} \sum_{i=1}^{n-1} \app{\theta}_i \right)  
    \left(\frac{1}{n-1}\sum_{i=1}^{n-1} \app{\theta}_i \right) ^{\top}}\\
    &=\left(\frac{1}{n-1} \right)^2 
    \E{\left(\sum_{i=1}^{n-1} 
    \theta_i +\rho_i\right)  
    \left(\sum_{i=1}^{n-1} 
    \theta_i +\rho_i \right) ^{\top}}\\
    &\underset{(a)}{=}\left(\frac{1}{n-1} \right)^2 
    \E{\sum_{i=1}^{n-1} \theta_i \theta^{\top}_{i}
    +\sum_{i=1}^{n-1} \sum_{j\neq i} 
    \theta_i \theta^{\top}_j
    +\sum_{i=1}^{n-1} \rho_i \rho_i^{\top}  }\\
    &\underset{(b)}{=}
    \frac{\ssigma +\mu_{*} \mu_{*}^{\top} }{n-1} 
    +\frac{n-2}{n-1} \: \mu_{*} \mu_{*}^{\top}
    +\frac{\sigma^2 \sum_{i=1}^{n-1} \E{\B{V}^{-1}_{i,\tau} }}
    {\left(n-1\right)^2 }\\
    &=\frac{\ssigma}{n-1} 
    +\mu_{*} \mu_{*}^{\top}
    +\frac{\sigma^2 \sum_{i=1}^{n-1} \E{\B{V}^{-1}_{i,\tau}}}
    {\left(n-1\right)^2 },
\end{align*}
where $(a)$ uses Lemma \ref{Lemma: Unbiased inner instance error} and that the actions taken during the first $\tau$ time-steps are independent of the noise terms, of $\theta_j$ and of other inner instance errors
and $(b)$ uses Lemma \ref{Lemma: Unbiased covariance estimation  - Term C}.
    
Using the definition of $\B{\Sigma}$ in \eqref{Eq: Sigma} the proof follows. 
\end{proof} 
\begin{lemma}
\label{Lemma: Unbiased covariance estimation  - Term B}
(Lack of bias of the covariance estimator - Term B)
\begin{equation*}
    \E{\left( \app{\mu}_n - \mu_{*} \right) \left( \app{\mu}_n - \mu_{*} \right)^{\top}}
    =\frac{\B{\Sigma}}{n-1}.
\end{equation*}
\end{lemma}   
\begin{proof}
\begin{align*}
    \E{\left( \app{\mu}_n - \mu_{*} \right) \left( \app{\mu}_n - \mu_{*} \right)^{\top}}
    &=\E{\app{\mu}_n  \app{\mu}_n^{\top}}
    -\E{\app{\mu}_{n}}\mu_{*}^{\top}
    -\mu_{*}\E{\app{\mu}_{n}^{\top}}
    +\mu_{*} \mu_{*}^{\top}\\
    &=\E{\app{\mu}_n  \app{\mu}_n^{\top}}
    -\mu_{*} \mu_{*}^{\top}\\
    &\underset{(a)}{=}
    \frac{\B{\Sigma}}{n-1} 
    +\mu_{*} \mu_{*}^{\top}
    -\mu_{*} \mu_{*}^{\top},
\end{align*}
where $(a)$ uses Lemma \ref{Lemma: Aux - Unbiased covariance estimation  - Term B}.
\end{proof}
\begin{lemma}
\label{Lemma: Covariance estimation error - Term C}
(Covariance estimation error - Term C)

For every instance $n > 1$, $\delta > 0$, with probability greater than $1-\delta$,
\begin{equation*}
    \normop{
     \frac{\sigma^2}{n-1} \left(\sum_{j=1}^{n-1}  \B{V}^{-1}_{j,\tau} - \sum_{j=1}^{n-1} \E{\B{V}^{-1}_{j,\tau}}\right) 
    }
    \leq 4 \cdot  \frac{\sigma^2}{\lbar{\lambda}_{\Msigma{\A}} d}\sqrt{\frac{2\ln{d} + 2\ln \left(\nicefrac{2}{\delta}\right)}
    {n-1} 
    }.
\end{equation*}
\end{lemma}   
\begin{proof}
The maximal eigenvalue of $\B{V}_{j,\tau}^{-1} - \E{\B{V}_{j,\tau}^{-1}}$ is upper bounded by,
\begin{align*}
    \lmax{\B{V}_{j,\tau}^{-1} - \E{\B{V}_{j,\tau}^{-1}}}
    &\underset{(a)}{\leq}
    \lmax{\B{V}_{j,\tau}^{-1}}
    +\lmax{-\E{\B{V}_{j,\tau}^{-1}}}\\
    &\underset{(b)}{\leq}
    \lmax{\B{V}_{j,\tau}^{-1}}\\
    &\underset{(c)}{\leq}
    \frac{2}{\lbar{\lambda}_{\Msigma{\A}} d},
\end{align*}
where $(a)$ uses Weyl's inequality,
$(b)$ uses that $\B{V}_{j,\tau}$ is PD
and $(c)$ uses the event $\mathcal{E}_{v\text{($\MQB$)}}$.

Using Lemma \ref{Lemma: Matrix Hoffeding} with $\B{X}_j \triangleq \B{V}_{j,\tau}^{-1} - \E{\B{V}_{j,\tau}^{-1}}, \B{A}_j^{2} \triangleq \frac{4}{\lbar{\lambda}_{\Msigma{\A}}^2 d^2} \B{I}$ \text{and}
$s^2 = \frac{4(n-1)}{\lbar{\lambda}_{\Msigma{\A}}^2 d^2}$, we get,
\begin{equation*}
    \p{\lmax{\sum_{j=1}^{n-1} \B{V}_{j, \tau}^{-1} - \E{\B{V}_{j, \tau}^{-1}}}
    \geq \delta}
    \leq d \exp \left(-\frac{\left(\lbar{\lambda}_{\Msigma{\A}} d \right)^2 \delta^2}{32(n-1)} \right),
\end{equation*}
multiplying the RHS by 2 to achieve the spectral norm bound, rearranging and multiplying by $\nicefrac{\sigma^2}{n-1}$ the proof follows. 
\end{proof}
\subsubsection{Covariance Estimation Error Bound}
\begin{lemma}
\label{Lemma: Covariance estimation error}
(Covariance estimation error)

For every $\eta_{n} > 0$ and for every
instance $n> 5d + 2\ln\left(\nicefrac{1}{\eta_{n}}\right)$,
\begin{equation*}
    \p{\mathcal{E}_{s\text{($\MQB$)}}\mid \mathcal{E}_{v\text{($\MQB$)}}} \geq 1 -6 \eta_{n}.
\end{equation*}
\end{lemma}
\begin{proof}
Using Lemma \ref{Lemma: Empirical covariance bounds} on terms A and B in \eqref{Eq: decompose covariance bound}, together with Lemma \ref{Lemma: Covariance estimation error - Term C} on Term C and a union bound argument, we get that for $n> 5d + 2\ln\left(\nicefrac{1}{\eta_{n}}\right)$, the three following expressions exist in probability greater than $1-6\eta_{n}$, 
\begin{align*}
    &\normop{
    \frac{1}{n-1}\sum_{j=1}^{n-1}
    \left(\app{\theta}_{j} - \mu_{*}\right)
    \left(\app{\theta}_{j} - \mu_{*}\right)^{\top}
    -\B{\Sigma}}
    \leq 32 \cdot \frac{2\sigma^2+\lbar{\lambda}_{\Msigma{\A}}\bar{\lambda}_{\Msigma{*}}d}
    {\lbar{\lambda}_{\Msigma{\A}}d}
    \cdot \sqrt{\frac{5d + 2\ln\left(\nicefrac{1}{\eta_{n}}\right)}
    {n-1}},\\
    &\normop{
    \frac{1}{n-1} \sum_{j=1}^{n-1} \left(\app{\mu}_{n} - \mu_{*}\right)
    \left(\app{\mu}_{n} - \mu_{*}\right)^{\top} 
    -\frac{\B{\Sigma}}{n-1}}
    \leq 32 \cdot \frac{2\sigma^2+\lbar{\lambda}_{\Msigma{\A}}\bar{\lambda}_{\Msigma{*}}d}
    {\lbar{\lambda}_{\Msigma{\A}}d(n-1)}
    \cdot \sqrt{\frac{5d + 2\ln\left(\nicefrac{1}{\eta_{n}}\right)}
    {n-1}},\\
    &\normop{
     \frac{\sigma^2}{n-1} \left(\sum_{j=1}^{n-1}  \B{V}^{-1}_{j,\tau} - \sum_{j=1}^{n-1} \E{\B{V}^{-1}_{j,\tau}}\right) 
    }
    \leq 4 \cdot  \frac{\sigma^2}{\lbar{\lambda}_{\Msigma{\A}} d}\sqrt{\frac{2\ln{d} + 2\ln \left(\nicefrac{1}{\eta_{n}}\right)}
    {n-1} 
    }.
\end{align*}
Plugging back to \eqref{Eq: decompose covariance bound},
\begin{align*}
    \normop{\asigma_n - \ssigma}
    &\leq
    32 \sqrt{\frac{5d + 2\ln\left(\nicefrac{1}{\eta_{n}}\right)}
    {n-1}}
    \left(\frac{n}{n-2}
    \left(\frac{2\sigma^2+\lbar{\lambda}_{\Msigma{\A}}\bar{\lambda}_{\Msigma{*}}d}
    {\lbar{\lambda}_{\Msigma{\A}}d}
    \right)
    +\frac{\sigma^2}
    {8\lbar{\lambda}_{\Msigma{\A}}d}
    \right)\\
    &\leq 50 \sqrt{\frac{5d + 2\ln\left(\nicefrac{1}{\eta_{n}}\right)}
    {n-1}}
    \left(
    \frac{2\sigma^2}
    {\lbar{\lambda}_{\Msigma{\A}}d}
    +\bar{\lambda}_{\Msigma{*}}
    \right),
\end{align*}
where the last inequality uses that $\frac{n}{n-2}\leq 3/2$ for $n\geq 6$.

Using the covariance widening scheme suggested by \cite{bastani2019meta} and proved in Lemma \ref{Lemma: widend matrix is PSD} we get,
\begin{equation*}
    \normop{\wsigma_n - \ssigma}
    \leq 100 \sqrt{\frac{5d + 2\ln\left(\nicefrac{1}{\eta_{n}}\right)}
    {n-1}}
    \left(
    \frac{2\sigma^2}
    {\lbar{\lambda}_{\Msigma{\A}}d}
    +\bar{\lambda}_{\Msigma{*}}
    \right),
    \quad \wsigma_n \succeq \ssigma.
\end{equation*}
\end{proof}
\section{META ALGORITHM REGRET}
\label{Appendix: Meta algorithm regret}
In the following section we prove Theorem~\ref{Thm:FullRegret}.
\fullMQBregret*
\begin{alignat}{2}
\label{Eq: constants for Theorem 2}
    \nonumber
    &k_2 &&\triangleq
    24\sqrt{c_{\xi}^2 c_s}
    \sqrt{3\left(
    2\sigma^2/(\lbar{\lambda}_{\Msigma{\A}}d)
    +\bar{\lambda}_{\Msigma{*}}\right)}
    \sqrt{d + \ln\left(dNT\right)}
    +4c_{\mathrm{w}}\left[c_s \tau + 12\sqrt{c_{\xi}^2 c_s c_1 d }
    +2c_{\xi}^2 c_s \right]
    \sqrt{5d + 2\ln\left(dNT\right)},\\
    &N_0 &&= \max \Big\{
    3, 4c_{\mathrm{w}}^2 c_s^2 \tau^2  \left(5d + 2\ln\left(d(N+1)T\right)\right),\\ 
    \nonumber
    &\quad &&\quad \quad \quad
    18 c_{\xi}^2 c_s \left( 3\left(
    2\sigma^2/(\lbar{\lambda}_{\Msigma{\A}}d)
    +\bar{\lambda}_{\Msigma{*}}\right)
    \left(d + \ln\left(d(N+1)T\right)\right) + 4c_{\mathrm{w}}^2 \left(c_1 d + c_{\xi}^2 c_s/36 \right) \left(5d + 2\ln\left(d(N+1)T\right)\right) \right)\Big\},
\end{alignat}
where $\tau$ defined in~\eqref{Eq: tau}, $c_{\mathrm{w}}$ in~\eqref{Eq: covariance widening} and the other constants are as in~\eqref{Eq: constants for Theorem: single instance regret} with $\delta=\nicefrac{1}{N}$.
\begin{proof}
\begin{equation}
\label{Eq: Theorem MQB regret - proof}
\begin{aligned}
    \sum_{n=1}^{N}
    \E{\RQBKn{\app{\mu}_{n}}{\wsigma_{n}}{T}}
    &=\underset{\text{exploration instances}}
    {\underbrace{
    \vphantom{\sum_{n=N_0+1}^{N}
    \E{\RQBKn{\app{\mu}_{n}}{\wsigma_{n}}{T}}}
    \sum_{n=1}^{N_0}
    \E{\RQBKn{\app{\mu}_{n}}{\wsigma_{n}}{T}}}}
    +\underset{\text{regular instances}}
    {\underbrace{
    \sum_{n=N_0+1}^{N}
    \E{\RQBKn{\app{\mu}_{n}}{\wsigma_{n}}{T}}}}\\
    &\underset{(a)}{\leq} N_0 \mathrm{R_{exp}}
    +\sum_{n=N_0+1}^{N}
    \left(\frac{\tilde{k}_2(n)}{2\sqrt{n-1}}
    \E{\RQBtaun
    {\mu_{*,\tau+1}}
    {\B{\Sigma}_{*,\tau+1}}
    {T-\tau}}
    +\frac{c_{\textrm{bad}}}{\sqrt{d}(n-1)}\right)\\
    &\underset{(a)}{\leq}
    N_0 \mathrm{R_{exp}}
    +k_2\sqrt{N}
     \E{ \RQB
    {\mu_{*,\tau+1}}
    {\B{\Sigma}_{*,\tau+1}}
    {T-\tau}},
\end{aligned}
\end{equation}
where $(a)$ uses Lemma \ref{Lemma: M-QB} and Theorem~\ref{Theorem: single instance regret} with $\tilde{k}_2(n) \triangleq 2k_1(n)/\sqrt{\delta(n)}$ for the regular instances 
and $(b)$ uses $k_2=\tilde{k}_2(N)$ and $\sum_{n=N_0+1}^{N}
    \frac{1}{\sqrt{n-1}} 
    \leq \int_{N_0}^{N}
    \frac{1}{\sqrt{x-1}} dx \leq 2\sqrt{N}$.
\end{proof}
\section{AUXILIARY LEMMAS}
\begin{lemma}
\label{Lemma: minimal eigenvalue sub-Gaussian}
(Matrix Chernoff, Theorem 5.1.1 in \cite{tropp2015introduction})

Consider a finite sequence $\{\B{X}_k\}$ of independent, random, Hermitian matrices with common dimension $d$. Assume that,
\begin{equation*}
    0 \leq \lmin{\B{X}_k} \text{ and }  \lmax{\B{X}_k} \leq L \text{ for each index k.} 
\end{equation*}
Then for every $0 \leq t <1$,
\begin{equation*}
    \p{\lmin{\sum_{k}\B{X}_k}
    \leq t \lmin{\sum_{k}\E{\B{X}_k}}}
    \leq d e^{-\frac{(1-t)^2\lmin{\sum_{k}\E{\B{X}_k}}}{2L}}.
\end{equation*}
\end{lemma}
\begin{lemma}
\label{Lemma: Minimum eigenvalue of Gram matrix}
(Minimum eigenvalue of Gram matrix)

For $\tau = \max \left\{d, \frac{8a^2}{\lbar{\lambda}_{\Msigma{\A}}}\ln\left(\frac{d^2T}{\delta}\right) \right\}$, the probability of the event $\mathcal{E}_{v}$ is bounded by,
\begin{equation*}
    \p{\mathcal{E}_{v}}
    \geq 1- \frac{\delta}{dT}.
\end{equation*}
\end{lemma}
\begin{proof}
\begin{equation*}
    \lmin{\vmtau}
    =\lmin{\sum_{s=1}^{\tau}
    {A_{s}}A_{s}^{\top}}
\end{equation*}
From Theorem 5.1.1 in \cite{tropp2015introduction}, Lemma \ref{Lemma: minimal eigenvalue sub-Gaussian} with $t= \frac{1}{2}$, and Assumption \ref{Assumption: eigenvalues action Covariance matrix} we get,
\begin{align*}
    \p{\lmin{\vmtau} 
    \leq 
    \frac{\lbar{\lambda}_{\Msigma{\A}}d}{2}}
    \leq
    \p{\lmin{\vmtau} 
    \leq 
    \frac{\lbar{\lambda}_{\Msigma{\A}}\tau}{2}}
    \leq
    de^{-\frac{\lbar{\lambda}_{\Msigma{\A}}\tau}
    {8a^2}}
    \leq \frac{\delta}{dT}.
\end{align*}
\end{proof}

\begin{lemma}
\label{Lemma: widend matrix is PSD}
($\wsigma \succeq \ssigma $)

Let $\asigma$ be a symmetric matrix and $\ssigma$ be a PD matrix s.t. $\normop{\asigma - \ssigma} \leq s$. 
Define $\wsigma= \asigma+s \cdot I$, then
\begin{equation*}
    \wsigma \succeq \ssigma 
    \quad \text{and} \quad
    \ssigma^{-1}
    \succeq
    \left(\wsigma\right)^{-1}.
\end{equation*}
\end{lemma}
\begin{proof}
\begin{align*}
    \lmin{\wsigma-\ssigma}
    \underset{(a)}{\geq}
    \lmin{s \cdot \B{I}}
    +\lmin{\asigma-\ssigma}
    \underset{(b)}{\geq}
    s
    -\normop{\ssigma - \asigma}
    \geq 0,
\end{align*}
where $(a)$ uses Weyl's inequality
and $(b)$ uses $\lmin{\asigma-\ssigma} = -\lmax{\ssigma-\asigma} \geq -\normop{\ssigma - \asigma}$.
Since $\left( \wsigma-\ssigma \right)$ is PSD, $\wsigma$ is PD, thus using Lemma \ref{Lemma: A-B is PSD} the proof follows.
\end{proof}

\begin{lemma}
\label{Lemma: maximal singularvalue inequality}
(maximal eigenvalue inequality)

Let $A$ be a vector and $\B{B}$ a PD matrix, then,
\begin{equation*}
    \norm{A} \leq \sqrt{\lmax{\B{B}}}\norm{A}_{\B{B}^{-1}}.
\end{equation*}
\begin{proof}
Since $\B{B}$ is symmetric we can use an eigenvalue decomposition $\B{B}=\B{Q} \B{\Lambda} \B{Q}^{\top}$, where $\B{\Lambda}$ is a diagonal matrix containing all the eigenvalues of $\B{B}$ and $\B{Q}, \B{Q}^{\top}$ are orthonormal matrices. Then, 
\begin{align*}
    \norm{A}&=\sqrt{A^{\top} A}\\
    &=\sqrt{A^{\top} \B{B}^{-1/2} \B{B} \B{B}^{-1/2} A}\\
    &=\sqrt{A^{\top} \B{B}^{-1/2} \B{Q} \B{\Lambda} \B{Q}^{\top} \B{B}^{-1/2} A}\\
    &\underset{(a)}{=} \sqrt{U^{\top} \B{\Lambda} U}\\
    &\underset{(b)}{\leq} \sqrt{\lmax{\B{B}}}\sqrt{U^{\top} U}\\
    &=\sqrt{\lmax{\B{B}}} \sqrt{A^{\top} \B{B}^{-1/2} \B{Q} \B{Q}^{\top} \B{B}^{-1/2} A}\\
    &\underset{(c)}{=} \sqrt{\lmax{\B{B}}} \sqrt{A^{\top} \B{B}^{-1} A}\\
    &=\sqrt{\lmax{\B{B}}} \norm{A}_{\B{B}^{-1}},
\end{align*}
where $(a)$ uses $U \triangleq \B{Q}^{\top} \B{B}^{-1/2} A$, $(b)$ uses that $U^{\top} U$ is a sum of non negative elements and $(c)$ uses that $\B{Q}$ is orthonormal.
\end{proof}
\end{lemma}

\begin{lemma}
\label{Lemma: Matrix Hoffeding}
(Matrix Hoffeding, Theorem 1.3 in \cite{tropp2012user})

Consider a finite sequence $\{\B{X}_j\}$ of independent, random, self-adjoint matrices with dimension $d$, and let $\{\B{A}_j\}$ be a sequence of fixed self-adjoint matrices. Assume that each random matrix satisfies,
\begin{equation*}
    \mathbb{E}\{\B{X}_j\} = 0
    \quad \text{and} \quad 
    \B{X}_j^{2} \preceq \B{A}_j^{2} \quad
    \text{almost surely.} 
\end{equation*}
Then for every $\delta \geq 0$,
\begin{equation*}
    \p{\lmax{\sum_{j}\B{X}_j}
    \geq \delta}
    \leq d e^{-\frac{\delta^2}{8 s^2}}
    \quad \text{where} \quad
    s^2 \triangleq \normop{\sum_{j}\B{A}_j^2}.
\end{equation*}
\end{lemma}

\begin{lemma}
\label{Lemma: Concentration bound SubGaussian vector}
(Concentration bound sub-Gaussian vector, Theorem 1 in \cite{hsu2012tail})

Let $\B{A} \in \R^{m \times n}$ be a matrix, and let $\B{\Sigma}=\B{A}^{\top} \B{A}$. Suppose that $X=(X_1,\ldots, X_n)$ is a
random vector such that for $\mu=0$ and some $\sigma \geq 0$,
\begin{equation*}
    \E{\exp \left( U^{\top} X\right)}
    \leq \exp \left(\frac{\norm{U}^2\sigma^2}{2} \right),
\end{equation*}
for all $U\in \R^{n}$.
For all $\delta>0$,
\begin{equation*}
    \p{\norm{\B{A}X}^2 > \sigma^2 \left( \Tr{\B{\Sigma}}+2\sqrt{\Tr{\B{\Sigma}^2}\delta}+2\normop{\B{\Sigma}}\delta\right)}
    \leq e^{-\delta}.
\end{equation*}
\end{lemma}

\begin{lemma}
\label{Lemma: Gaussian tail bounds}
(Gaussian tail bounds, section 7.1 in \cite{feller1968introduction})

Let $z$ be a standard normal variable.
Then for any $t>0$,
\begin{equation}
\label{Eq: Gaussian tail bounds lower bound}
    \p{z > t}
    > \frac{1}{\sqrt{2\pi}} 
    \left(\frac{1}{t} -\frac{1}{t^3} \right)
    e^{\frac{-t^2}{2}},
\end{equation}
\begin{equation}
\label{Eq: Gaussian tail bounds upper bound}
    \p{z > t}
    \leq \frac{1}{\sqrt{2\pi}} \frac{1}{t} e^{\frac{-t^2}{2}}.
\end{equation}
\end{lemma}

\begin{lemma}
\label{Lemma: Same non-zero eigenvalues}
(Same non-zero eigenvalues for $\B{A}\B{B}$ and $\B{B}\B{A}$, Theorem 1.3.22 in \cite{horn2012matrix})

Let $\B{A} \in \mathbb{R}^{m \times n}$ and $\B{B} \in \mathbb{R}^{n \times m}$ with $m \leq n$. Then the $n$ eigenvalues of $\B{B}\B{A}$ are the $m$ eigenvalues of $\B{A}\B{B}$ together with $n-m$ zeros. If $m=n$ and at least one of $\B{A}$ or $\B{B}$ is nonsingular, then $\B{A}\B{B}$ and $\B{B}\B{A}$ are similar.
\end{lemma}
\begin{lemma}
\label{Lemma: Same eigenvalues for product of PSD matrices}
(Same eigenvalues for product of PSD matrices $\lambda_{i} \left( \B{A}\B{B} \right)
    =\lambda_{i} \left( \B{A}^{1/2}\B{B}\B{A}^{1/2} \right)$)

Let $\B{A},\B{B}$ PSD matrices. 
Then,
\begin{equation*}
    \lambda_{i} \left( \B{A}\B{B} \right)
    =\lambda_{i} \left( \B{A}^{1/2}\B{B}\B{A}^{1/2} \right)
    =\lambda_{i} \left( \B{B}\B{A} \right)
    =\lambda_{i} \left( \B{B}^{1/2}\B{A}\B{B}^{1/2} \right).
\end{equation*}
\begin{proof}
$\B{A}\B{B}=\B{A}^{1/2} \B{A}^{1/2} \B{B}$.
Using Lemma \ref{Lemma: Same non-zero eigenvalues} this matrix has the same eigenvalues as $\B{A}^{1/2} \B{B} \B{A}^{1/2}$.
In the same line of proof we get that the eigenvalues of $\B{B}\B{A}$ and $\B{B}^{1/2} \B{A} \B{B}^{1/2}$ are the same.
Finally, by Lemma \ref{Lemma: Same non-zero eigenvalues} $\B{A}\B{B}$ and $\B{B}\B{A}$ has the same eigenvalues.
\end{proof}
\end{lemma}
\begin{lemma}
\label{Lemma: Non negative trace for product of PSD matrices}
(Non negative trace for product of PSD matrices)

Let $\B{A}$ and $\B{B}$ be PSD matrices. 
Then, $\Tr{\B{A}\B{B}} \geq 0$.
\begin{proof}
From Lemma~\ref{Lemma: Same eigenvalues for product of PSD matrices},
\begin{align*}
    \Tr{\B{A}\B{B}}
    =\Tr{\B{A}^{1/2}\B{B}\B{A}^{1/2}}
    =\sum_{j=1}^{d} \lambda_j \left( \B{A}^{1/2}\B{B}\B{A}^{1/2}  \right)
    \geq 0,
\end{align*}
where the inequality uses that the matrix $\B{B}$ is PSD, hence $\B{A}^{1/2}\B{B}\B{A}^{1/2}$ is PSD.
\end{proof}
\end{lemma}
\begin{lemma}
\label{Lemma: A-B is PSD}
(Corollary 7.74.(a) in \cite{horn2012matrix})

Let $\B{A},\B{B} \in \mathbb{R}^{n \times n}$ PD matrices. $\B{A} \succeq \B{B}$ \ iff \ $\B{B}^{-1} \succeq \B{A}^{-1}$.
\end{lemma}
\begin{lemma}
\label{Lemma: square of A minus square of B is PSD}
(Corollary 7.74.(b) in \cite{horn2012matrix})

Let $\B{A},\B{B} \in \mathbb{R}^{n \times n}$ be symmetric matrices. If $\B{A} \succ 0$, $\B{B} \succeq 0$ and $\B{A} \succeq \B{B}$, then $\B{A}^{1/2} \succeq \B{B}^{1/2}$.
\end{lemma}
\begin{lemma}
\label{Lemma: Matrix Analysis (Rajendra Bathia)}
(Exercise VI.7.2 in \cite{bhatia1997matrix} (page 182))

Let $\B{A}$ and $\B{B}$ be Hermitian matrices. 
If $\lmin{\B{A}} +\lmin{\B{B}} > 0$, then
\begin{equation*}
    \prod_{j=1}^{d} \left( 
    \lambda_{j} \left(\B{A}\right) + \lambda_{j} \left( \B{B} \right)
    \right)
    \leq \Det{\B{A} + \B{B}}.
\end{equation*}
\end{lemma}
\begin{lemma}
\label{Lemma: B bound}
($\B{B}$ bound)
\begin{equation*}
    \normop{\B{B}}
    \leq 
    \frac{\sigma^2}
    {\lbar{\lambda}_{\ssigma}^{2}}
    \normop{\asigma
    - \ssigma}.
\end{equation*}
\begin{proof}
\begin{align*}
    \normop{\B{B}}
    &=\normop{\sigma^2
    \left(\ssigma^{-1} - \asigma^{-1}\right)}\\
    &=\normop{\sigma^2
    \left(\asigma^{-1}
    \left( \asigma - \ssigma\right)
    \ssigma^{-1}
    \right)}\\
    &\underset{(a)}{\leq}
    \sigma^2
    \normop{\asigma^{-1}}
    \normop{\asigma
    - \ssigma}
    \normop{\ssigma^{-1}}\\
    &\underset{(b)}{\leq} 
    \frac
    {\sigma^2 }
    {\left(\lmin{\ssigma}\right)^{2}}
    \normop{\asigma
    - \ssigma}\\
    &\underset{(c)}{\leq}
    \frac{\sigma^2}{\lbar{\lambda}_{\ssigma}^{2}}
    \normop{\asigma
    - \ssigma},
\end{align*}
where $(a)$ uses sub-multiplicative norm properties,
$(b)$ uses that $\asigma \succeq \ssigma$ and Lemma \ref{Lemma: maximal singularvalue inequality},
and $(c)$ uses Assumption~\ref{Assumption: eigenvalues prior Covariance matrix}.
\end{proof}
\end{lemma}
\begin{lemma}
\label{Lemma: exponent bound}
(Exponent bound, Lemma 20 in
\cite{bastani2019meta})

For any $x \in [0,1]$,
    $\quad e^{x} \leq 1 +2x$.
\end{lemma}
\begin{lemma}
\label{Lemma: Log bound}
(Log bound, Lemma A.2 in \cite{shalev-shwartz_ben-david_2014})

Let $a_1 \geq 1$ and $b_1>0$. Then, $\quad x\geq 4a_1\ln \left(2a_1\right) +2b_1 \Rightarrow x \geq a_1 \ln \left(x \right) +b_1$.
\end{lemma}

\end{document}